\documentclass[11pt]{article}
\usepackage[margin=1in]{geometry}
\usepackage{amsmath,amssymb,amsfonts,cancel}
\usepackage{algorithm}
\usepackage[noend]{algpseudocode}
\usepackage{mathtools}
\usepackage{tikz}
\usepackage{dsfont}
\usepackage{subfigure}
\usepackage{xcolor}
\definecolor{edits}{rgb}{1,0,0}
\usepackage{caption}
\usepackage[colorlinks=true,breaklinks=true,bookmarks=true,urlcolor=magenta,citecolor=magenta,linkcolor=magenta,bookmarksopen=false,draft=false]{hyperref}

\usepackage{booktabs}
\usepackage{xspace}

\newcommand{\NN}{\mathbb{N}}
\newcommand{\RR}{\mathbb{R}}
\renewcommand{\S}{\mathcal{S}}
\renewcommand{\P}{\mathcal{P}}
\newcommand{\E}{\mathbb{E}}
\newcommand{\D}{\mathcal{D}}
\newcommand{\A}{\mathcal{A}}
\newcommand{\B}{\mathcal{B}}
\newcommand{\F}{\mathcal{F}}

\newcommand{\regret}{\textsc{Regret}}
\newcommand{\diam}[1]{\operatorname{diam}\left( #1 \right)}

\newcommand{\Ind}[1]{\mathds{1}_{\left[ #1 \right]}}

\newcommand{\Exp}[1]{\mathbb E \left[ #1 \right]} % Variance

\renewcommand{\Pr}{\mathbb{P}}
\newcommand{\pfail}{\delta}
\newcommand{\relevant}{\textsc{Relevant}}
\newcommand{\Qhat}[2]{\ensuremath{\overline{\mathbf{Q}}_{#1}^{#2}}}

\newcommand{\Vhat}[2]{\overline{\mathbf{V}}_{#1}^{#2}}
\newcommand{\Vtilde}[2]{\widetilde{\mathbf{V}}_{#1}^{#2}}
\newcommand{\gap}{\textsc{gap}}

\newcommand{\X}{\mathcal{X}}
\newcommand{\I}{\mathcal{I}}

\newcommand{\rbar}[2]{\overline{\mathbf{r}}_{#1}^{#2}}
\newcommand{\rhat}[2]{\hat{\mathbf{r}}_{#1}^{#2}}

\newcommand{\Tbar}[2]{\overline{\mathbf{T}}_{#1}^{#2}}
\newcommand{\rbonus}[2]{\textsc{Rucb}_{#1}^{#2}}
\newcommand{\tbonus}[2]{\textsc{Tucb}_{#1}^{#2}}
\newcommand{\bbonus}[2]{\textsc{Bias}}
\newcommand{\bias}{\textsc{Bias}}

\newcommand{\clip}[1]{\textsc{clip}\left[#1 \right]}
\newcommand{\lev}[1]{\ensuremath{\ell(#1)}}
\newcommand{\Pkh}[1][k]{\ensuremath{\mathcal{P}^{#1}_h}}
\newcommand{\nplus}[1]{\ensuremath{n_{+}(#1)}}
\newcommand{\frall}{\ensuremath{\,\forall\,}}
\newcommand{\AdaMB}{\textsc{AdaMB}\xspace}
\newcommand{\AdaQL}{\textsc{AdaQL}\xspace}
\newcommand{\EpsQL}{\textsc{EpsQL}\xspace}
\newcommand{\EpsMB}{\textsc{EpsMB}\xspace}
\newcommand{\dyad}[1]{\S(\P_{#1})}
\newcommand{\gam}{\ensuremath{\gamma}}
\newcommand{\conf}{\textsc{Conf}}
\newcommand{\confcons}{\widetilde{C}}
\newcommand{\conslip}{C_{L}}
\newcommand{\bcenter}{\textsc{Center}}

\ifdefined \argmax
\else \DeclareMathOperator*{\argmax}{arg\,max}
\fi
\ifdefined \argmin
\else \DeclareMathOperator*{\argmin}{arg\,min}
\fi

\usepackage[suppress]{color-edits}
\addauthor{sb}{magenta}
\addauthor{srs}{blue}
\addauthor{cy}{purple}
\addauthor{srstodo}{orange}

\mathchardef\mhyphen="2D % Define a "math hyphen"
\DeclarePairedDelimiter{\norm}{\lVert}{\rVert}

\ifdefined\informs \else \usepackage{amsthm} \fi
\usepackage[capitalize]{cleveref}

\newenvironment{rproof}[1]{ \ifdefined\informs \proof{Proof of #1.}
\else \begin{proof} \fi }{ \ifdefined\informs 
\endproof \else \end{proof} \fi  }

\crefname{assumption}{Assumption}{Assumptions}

\newtheorem{informaltheorem}{Informal Theorem}
\newcommand{\Halmos}{{}}
\usepackage[numbers]{natbib}

\allowdisplaybreaks

\newtheorem{theorem}{Theorem}
\numberwithin{theorem}{section}
\newtheorem{definition}[theorem]{Definition}

\newtheorem{lemma}[theorem]{Lemma}

\newtheorem{corollary}[theorem]{Corollary}

\newtheorem{assumption}{Assumption}

\begin{document}
	\title{Adaptive Discretization in Online Reinforcement Learning}
	\author{Sean R. Sinclair 
			\,\, Siddhartha Banerjee  
			\,\, Christina Lee Yu \\
			School of Operations Research and Information Engineering \\
			Cornell University}
	\date{}
	\maketitle

	\begin{abstract}
    Discretization based approaches to solving online reinforcement learning problems have been studied extensively on applications like resource allocation and cache management. The two major questions in designing discretization based algorithms are how to create the discretization and when to refine it.  There have been several experimental results investigating heuristic approaches to these questions but little theoretical treatment.  In this paper we provide a unified theoretical analysis of \emph{model-free} and \emph{model-based} tree-based adaptive hierarchical partitioning methods for online reinforcement learning.  We show how our algorithms take advantage of inherent problem structure by providing guarantees that scale with respect to the ``zooming dimension'' instead of the ambient dimension, an instance-dependent quantity measuring the benignness of the optimal $Q_h^\star$ function.

Many applications in computing systems and operations research require algorithms that compete on three facets: low sample complexity, mild storage requirements, and low computational burden for policy evaluation and training.  Our algorithms are easily adapted to operating constraints, and our theory provides explicit bounds across each of the three facets.
	\end{abstract}
	\newpage
	\setcounter{tocdepth}{2}
	\tableofcontents
	\newpage
	
\section{Introduction}
\label{sec:introduction}

Reinforcement learning (RL) is a popular approach for sequential decision-making and has been successfully applied to games~\citep{silver2016mastering} and systems applications~\citep{alizadeh2010dctcp}.  In these models, a principal interacts with a system that has stochastic transitions and rewards. The principal aims to control the system either online through exploring available actions using real-time feedback, or offline by exploiting known properties of the system and an existing dataset.

These sequential decision making problems have been considered across multiple communities.  As data has become more readily available and computing power improves, the new zeitgeist for these fields is developing \textit{data-driven decision algorithms}: algorithms which adapt to the structure of information, constraints, and objectives in any given domain.  This paradigm highlights the importance of taking advantage of data collected and inherent structure and geometry of the problem to help algorithms scale to complex domains.

With the successes of neural networks as a universal function approximator, RL has received a lot of interest in the design of algorithms for large-scale systems using parametric models~\citep{jiang2017contextual,mozur_2017}.
While these results highlight the power of RL in learning complex control policies, they are \emph{infeasible} for many applications arising in operations research and computing systems~\citep{hubbs2020or}.  As an example, the AlphaGo Zero algorithm that mastered Chess and Go from scratch was trained over seventy-two hours using four TPUs and sixty-four GPUs~\citep{silver2016mastering}.  The limiting factor in using these large scale parametric algorithms is implementing regression oracles or gradient steps on computing hardware, and the large storage burden in maintaining the models.  These issues don't typically arise in game-based or robotics applications.  However, they are key algorithmic ingredients which are ignored in theoretical treatment analyzing storage and time complexity.  Moreover, these models require strict parametric assumptions, suffer under model misspecification, and do not adapt to the underlying geometry of the problem.  

In contrast, RL has received interest in designing small-scale and efficient controllers for problems arising in operations research (OR) and computing systems, including memory systems~\citep{alizadeh2010dctcp} and resource allocation in cloud-based computing~\citep{10.1145/1394608.1382172}.  Their engineering approaches use discretizations at various levels of coarseness to learn estimates in a data-efficient manner.  Common to these examples are computation and storage limitations on the devices used for the controller, requiring algorithms to compete on three major facets: efficient learning, low computation, and low storage requirements, \textbf{a trifecta for RL in OR.}

Motivated by this paradigm we consider nonparametric discretization (or quantization) techniques which map complex problems to discrete ones.  These algorithms are based on simple primitives which are easy to implement on hardware, can leverage existing hardware quantization techniques, and have been tested heuristically in practice~\citep{pyeatt2001decision,7929968,uther1998tree,araujo2020single, araujo2020control}.  A key challenge in this approach is picking a discretization to manage the trade-off between the discretization error and the errors accumulated from solving the discrete problem.  Moreover, if the discretization is fixed a priori then the algorithm cannot \emph{adapt} to the underlying structure in the problem, and so adaptive discreitzations are necessary for instance specific gains.  We develop theoretical foundations for an adaptive discretization of the space, where the discretization is only refined on an \textit{as needed} basis using \textit{collected data}.  We answer the two important aspects of designing adaptive discretization algorithms: how to create the discretization and when to refine it, by exploiting the metric structure induced by specific problem instances.

\subsection{Our Contributions}

We provide a unified analysis of \AdaMB and \AdaQL, \textit{model-based} and \textit{model-free} algorithms that discretize the state action space in a data-driven way so as to minimize regret.  This extends the well-studied adaptive discretization techniques seen in contextual multi-armed bandits to dynamic environments \citep{Kleinberg:2019:BEM:3338848.3299873,slivkins_2014}.  Moreover, it illustrates that adaptive discretization can be viewed as an all-purpose tool for improving fixed discretization algorithms to better take advantage of problem structure.

These algorithms require that the state and action spaces are embedded in compact metric spaces, and the problem primitives are Lipschitz continuous with respect to this metric.  This encompasses discrete and continuous state action spaces with mild assumptions on the transition kernel and rewards. Our algorithms only requires access to the metric, unlike prior nonparametric algorithms which require access to simulation oracles or impose additional assumptions on the action space to be computationally efficient.  In fact, the assumption that $|\A| < \infty$ is quite common in theoretical treatment for simplicity.  However, this ignores the technical and computational hurdles required.  Our algorithms avoid this issue through the efficient discretization of the action space.

We show that \AdaMB and \AdaQL achieve near optimal dependence of the regret on the \emph{zooming dimension} of the metric space, an instance dependent quantity which measures the intrinsic complexity and geometry of the problem by scaling with the dimension of the level sets of the optimal $Q_h^\star$ function instead of the ambient dimension.  Our main result is summarized in the following informal theorem.
\begin{informaltheorem}
For an $H$-step MDP played over $K$ episodes, our algorithms achieve regret:
\begin{align*}
        \regret(K) & \lesssim \begin{cases}
        	\AdaQL : & H^{5/2} K^{\frac{z_{max}+1}{z_{max}+2}} \\
            \AdaMB: & H^{3/2}K^{\frac{z_{max}+d_\S - 1}{z_{max}+d_\S}} \quad d_S > 2 \\
            \AdaMB: & H^{3/2}K^{\frac{z_{max}+ 1}{z_{max}+2}} \quad d_S \leq 2
                \end{cases}
\end{align*}
where $d_\S$ is the covering of the state space and $z_{max} = \max_{h} z_h$ is the worst-case over the step $h$ zooming dimensions $z_h$ (see Definition~\ref{def:zooming}).  
\end{informaltheorem}

Our bounds are uniformly better in terms of dependence on $K$ and $H$ than the best existing bounds for nonparametric RL (see \cref{tab:comparison_of_bounds}).  Our bounds exhibit explicit dependence on the zooming dimension instead of the ambient dimension, leading to exponential improvements in sample complexity since the zooming dimension is trivially upper bounded by the ambient dimension.  In addition, \AdaQL matches the lower bound up to polylogarithmic factors, while \AdaMB suffers from additional $d_\S$ terms when $d_\S > 2$.  In general, the lower bound shows that this exponential scaling is necessary in nonparametric settings (a fundamental trade-off of using nonparametric algorithms), but note that the exponential scaling is with respect to the zooming dimension instead of ambient dimensions.  We also show that \AdaMB matches the bounds of \AdaQL under additional assumptions of the transition distribution.

\begin{table}[!tb]

\setlength\tabcolsep{0pt} % let LaTeX compute intercolumn whitespace
\centering
\begin{tabular*}{\columnwidth}{@{\extracolsep{\fill}}lcccc}
\toprule
  Algorithm  & Type & Regret & Time & Space \\
\midrule
  \AdaMB $(d_\S > 2)$ & MB & $H^{3/2}K^{\frac{z_{max} + d_\S - 1}{z_{max}+d_\S}}\;\;\;$ & $HK^{\frac{d+2d_\S}{d+d_\S}}$ & 
  $HK$ \\
    \AdaMB $(d_\S \leq 2)$ & MB & $H^{3/2}K^{\frac{z_{max} + 1}{z_{max}+2}}$ & $HK^{\frac{d+d_\S + 2}{d+d_\S}}$ & 
  $HK^{\frac{d+d_\S}{d+2}}$ \\
  \AdaQL & MF & $H^{5/2}K^{\frac{z_{max} + 1}{z_{max}+2}}$ & $HK\log_d(K)$  & $HK^{\frac{d}{d+2}}$\\
  % \textsc{UCCRL-Kernel Density}~\citep{lakshmanan2015improved} & $HK^{(2d+1)/(2d+2)}$ & ?? & ?? \\
  \textsc{Kernel UCBVI}~\citep{domingues2020regret} & MB & $H^3\;\;\;\,K^{\frac{2d}{2d+1}}$ & $HAK^2$ & $HK$ \\
  \textsc{Net-Based $Q$-Learning}~\citep{song2019efficient} & MF &  $H^{5/2}K^{\frac{d+1}{d+2}}$ & $HK^2$ & $HK$ \\
  \textsc{Net-Based UCBVI} & MB & $H^{3/2}K^{\frac{2d+1}{2d+2}}$ & $H^2K^2$ & $HK$ \\
\midrule
  \textsc{Lower Bound} & N/A & $H\;\;\;\;K^{\frac{z_{max} + 1}{z_{max}+2}}$ & - & - \\
\bottomrule
\end{tabular*}
\caption{\em Comparison of our bounds with several state-of-the-art bounds for nonparametric RL.  $d$ is the covering dimension of the state-action space, $d_\S$ is the covering dimension of the state space, $H$ is the horizon of the MDP, and $K$ is the total number of episodes.  Under {Type} we denote whether the algorithm is model-based (MB) or model-free (MF).  As implementing \textsc{Kernel UCBVI}~\citep{domingues2020regret} is unclear under general action spaces, we specialize the time complexity under a finite set of actions of size $A$, but more details are included in their paper.  See \cref{app:full_algo} for a discussion on \textsc{Net-Based UCBVI}.}
\label{tab:comparison_of_bounds}
\end{table}

In addition to having lower regret, \AdaMB and \AdaQL are also simple and practical to implement, with low query complexity and storage compared to other techniques (see~\cref{tab:comparison_of_bounds}).  To the best of our knowledge, our algorithms are the first to have provably sublinear regret with improved time and storage complexity in the setting of continuous state and action MDPs.  We complement our theory with synthetic experiments comparing model-free and model-based algorithms using both fixed and adaptive discretizations.  We picked experiments of varying complexity in low-dimensional spaces, including those with provably smaller zooming dimensions, to help highlight the adaptive discretizations matching the level sets of the underlying $Q_h^\star$ value.  Through our experiments we measure the three aspects of the \textbf{trifecta for RL in OR}, comparing the performance of these algorithms in terms of regret, time complexity, and space complexity. Our experiments show that with a fixed discretization, model-based algorithms outperform model-free ones while suffering from worse storage and computational complexity.  However, when using an adaptive partition of the space, model-based and model-free algorithms perform similarly.  
\subsection{Motivating Examples}
\label{sec:examples}

Reinforcement learning has enjoyed remarkable success in recent years in large scale game playing and robotics.  These results, however, mask the high underlying costs in terms of computational resources, energy costs, training time, and hyperparameter tuning that their demonstrations require \citep{silver2017mastering,mnih2016asynchronous,mnih2013playing,silver2016mastering}.  On the other hand, RL has been applied heuristically in the following problems:

\medskip \noindent \textbf{Memory Management}: Many computing systems have two sources of memory: on-chip memory which is fast but limited, and off-chip memory which has low bandwidth and suffers from high latency.  Designing memory controllers for these system requires a scheduling policy to adapt to changes in workload and memory reference streams, ensuring consistency in the memory, and controlling for long-term consequences of scheduling decisions \citep{alizadeh2010dctcp,alizadeh2013pfabric,chinchali2018cellular}.
	
\medskip \noindent \textbf{Online Resource Allocation}: Cloud-based clusters for high performance computing must decide how to allocate computing resources to different users or tasks with highly variable demand.  Controllers for these algorithms make decisions online to manage the trade-offs between computation cost, server costs, and delay in job-completions~\citep{10.1145/1394608.1382172,lykouris2018competitive,nishtala2013scaling,tessler2021reinforcement}. 

\medskip Common to these examples are computation and storage limitations on the devices used for the controller.
\begin{enumerate}
	\item \textit{Limited Memory}: As any RL algorithm requires memory to store estimates of relevant quantities, algorithms for computing systems must manage their storage requirements so frequently needed estimates are stored in on-chip memory.
	\item \textit{Power Consumption}: Many applications require low-power consumption for executing RL policies on general computing platforms. 
	\item \textit{Latency Requirements}: Problems for computing systems (e.g. memory management) have strict latency quality of service requirements that limits reinforcement learning algorithms to execute their policy quickly.
\end{enumerate}

A common technique in these domains is cerebellar model articulation controllers (CMACs) or other quantization and hashing based methods, which have been used in optimizing controllers for dynamic RAM access~\citep{10.1145/1394608.1382172,lykouris2018competitive,nishtala2013scaling}.  The CMAC technique uses a random discretization of the space at various levels of coarseness combined with hashing.  The other approaches are hierarchical decision trees, where researchers have investigated splitting heuristics for refining the adaptive partition by testing their empirical performance~\citep{pyeatt2001decision,7929968,uther1998tree}.  These quantization based algorithms address the computation and storage limitations by allowing algorithm implementations to exploit existing hashing and caching techniques for memory management, since the algorithms are based on simple look-up tables.  Our algorithms are motivated by these approaches, taking a first step towards designing theoretically efficient reinforcement learning algorithms for continuous spaces.

\subsection{Related Work}
\label{sec:related_work}

There is an extensive and growing literature on reinforcement learning, below we highlight the work which is closest to ours, but for more extensive references see~\cite{sutton2018reinforcement,agarwal2019reinforcement,puterman_1994,powell2019reinforcement} for RL, and~\cite{bubeck2012regret,aleks2019introduction} for bandits.

\medskip

\noindent \textbf{Tabular RL}:  There is a long line of research on the regret for RL in tabular settings.  In particular, the first asymptotically tight regret bound for tabular model-based algorithms with non-stationary dynamics was established to be $O(H^{3/2} \sqrt{SAK})$ where $S$ and $A$ are the size of the state and action spaces respectively~\citep{azar2013minimax}.
These bounds were matched (in terms of $K$) using an ``asynchronous value-iteration'' (or one-step planning) approach~\citep{azar2017minimax,efroni2019tight}, which is simpler to implement~\citep{10.1145/1394608.1382172,lykouris2018competitive,nishtala2013scaling,tessler2021reinforcement}.  This regret bound was also matched (in terms of $K$) for model-free algorithms~\citep{jin_2018}. More recently, the analysis was extended to develop instance dependent bounds as a function of the variance or shape of the underlying $Q_h^\star$ function~\citep{zanette2019tighter,simchowitz2019}.  Our work extends this latter approach to continuous spaces using adaptive discretization to obtain instance specific guarantees scaling with the zooming dimension instead of the ambient dimension of the space.

\medskip

\noindent\textbf{Parametric Algorithms}:  For RL in continuous spaces, several recent works have focused on the use of linear function approximation~\citep{jin2019provably, du2019provably,zanette2019limiting, wang2019optimism,wang2020provably,osband2014model,chen2021estimating}.  These works assume that the principal has a feature-extractor under which the process is well-approximated by a linear model. In practice, these algorithms require an initial ``feature engineering'' process to learn features under which the problem is linear, and the guarantees then hinge upon a perfect construction of features.  If the requirements are violated, it has been shown that the theoretical guarantees degrade poorly~\citep{du2019good}.  Other work has extended this approach to problems with bounded eluder dimension and other notions of dimension of parametric problems~\citep{wang2020provably,russo2013eluder}.

\medskip

\noindent\textbf{Nonparametric Algorithms}: In contrast, nonparametric algorithms only require mild local assumptions on the underlying process, most commonly that the $Q$-function is Lipschitz continuous with respect to a given metric.  For example,~\cite{yang2019learning} and~\cite{shah2018q} consider nearest-neighbour methods for deterministic, infinite horizon discounted settings.  
Others assume access to a generative model~\citep{kakade2003exploration,henaff2019explicit,shah2020sample}. 

The works closest to ours concerns online algorithms for finite horizon problems with continuous state action spaces (see also \cref{tab:comparison_of_bounds}).
In model-free settings, tabular algorithms have been adapted to continuous state-action spaces via fixed discretization (i.e. $\epsilon$-nets)~\citep{song2019efficient}. In model-based settings, researchers have tackled continuous spaces using kernel methods based on either a fixed discretization of the space \citep{lakshmanan2015improved}, or with smooth kernel functions~\citep{domingues2020regret}. While the latter learns a data-driven representation of the space, it requires solving a complex optimization problem over actions at each step, and hence is efficient mainly for finite action sets (more discussion on this is in \cref{sec:main_results}).
Finally, adaptive discretization has been successfully implemented and analyzed in model-free and model-based settings~\citep{Sinclair_2019,cao2020provably,sinclair2020adaptive}.  This work serves as a follow up providing a unified analysis between the two approaches, improved performance guarantees for \AdaMB, instance dependent guarantees scaling with the zooming dimension instead of the ambient dimension, and additional numerical simulations.

\medskip

\noindent\textbf{Discretization Based Approaches}: Discretization-based approaches to reinforcement learning have been explored heuristically in different settings.  One line of work investigates adaptive basis functions, where the parameters of the functional model (e.g. neural network) are learned online while simultaneously adapting the basis functions~\citep{keller2006automatic,menache2005basis,whiteson2006evolutionary}.  Similar techniques are done with soft state aggregation~\citep{singh1995reinforcement}.  Most similar to our algorithm, though, are tree based partitioning rules, which store a hierarchical partition of the state and action space which is refined over time~\citep{pyeatt2001decision,7929968,uther1998tree}.  These were tested heuristically with various splitting rules (e.g. Gini index, etc), where instead our algorithm splits based off the metric and statistical uncertainty in the estimates.  Researchers have also extended our adaptive discretization techniques to using a single partition in infinite horizon time-discounted settings, and the algorithm was benchmarked on various control tasks in \textsc{Open AI} showing comparative performance between discretization and deep learning techniques~\citep{araujo2020control,araujo2020single}.

\subsection{Outline of Paper}

Section~\ref{sec:preliminary} introduced the model, nonparametric assumptions, and the zooming dimension.  Our algorithms, \AdaQL and \AdaMB, are described in Section~\ref{sec:algorithm} with the regret bound and proof sketch given in Section~\ref{sec:main_results} and Section~\ref{sec:proof_sketch} respectively.  Proof details are included in \cref{sec:concentration,sec:partition,sec:optimism,sec:clipping,sec:regret_decomp,sec:lp_bound}, with miscellaneous technical results in \cref{app:technical_details}.  The experimental results are in~\cref{sec:experiments}.

\section{Preliminaries}
\label{sec:preliminary}

\subsection{MDP and Policies}

We consider the online episodic reinforcement learning setting, where an agent is interacting with an underlying finite-horizon Markov Decision Process (MDP) over $K$ sequential episodes, denoted $[K] = \{1, \ldots, K\}$\citep{puterman_1994}.

\begin{definition}

An \textbf{Episodic Markov Decision Process} (MDP) is given by a five-tuple $(\S, \A, H, T, R)$ where the horizon $H$ is the number of steps indexed $[H] = \{1,2,\ldots,H\}$ in each episode, and $(\S,\A)$ denotes the set of states and actions in each step.  State transitions are governed by a collection of transition kernels $T = \{T_h(\cdot \mid x,a)\}_{h \in [H], x \in \S, a \in \A}$, where $T_h(\cdot \mid x, a) \in \Delta(\S)$ gives the distribution over states in $\S$ if action $a$ is taken in state $x$ at step $h$.  The instantaneous rewards are bounded in $[0,1]$, and their distributions are specified by a collection of parameterized distributions $R = \{R_h\}_{h \in [H]}$, $R_h : \S \times \A \rightarrow \Delta([0, 1])$. We let $r_h(x,a) = \mathbb{E}_{r \sim R_h(x,a)}[r]$ denote the mean reward.
\end{definition}

The agent interacts with the MDP by selecting a policy, where a policy $\pi$ is a sequence of distributions $\pi = \{ \pi_h \mid h \in [H] \}$ where each $\pi_h : \S \rightarrow \Delta(\A)$ is a mapping from a given state $x \in \S$ to a distribution over actions in $\A$.

\subsection{Value Function and Bellman Equations}

For any policy $\pi$, let $A^{\pi}_h$ denote the (potentially random) action taken in step $h$ under policy $\pi$, i.e., $A^{\pi}_h = \pi_h(X_h^k)$.
\begin{definition}
We define the \textbf{policy value function} at step $h$ under policy $\pi$ to be the expected sum of future rewards under policy $\pi$ starting from $X_h = x$ in step $h$ until the end of the episode, which we denote $V_h^\pi: \S \rightarrow \RR$. Formally,
%\vspace*{-6pt}
\begin{align}
    V_h^\pi(x) := \Exp{\textstyle\sum_{h'=h}^H R_{h'} ~\Big|~ X_h = x} ~~\text{ for }~~ R_{h'} \sim R_{h'}(X_{h'},A_{h'}^{\pi}).
\end{align}
We define the \textbf{state-action value function} (or $Q$-function) $Q_h^\pi : \S \times \A \rightarrow \RR$ at step $h$ as the sum of the expected rewards received after taking action $A_h = a$ at step $h$ from state $X_h = x$, and then following policy $\pi$ in all subsequent steps of the episode. Formally,
%\vspace*{-6pt}
\begin{align}
    Q_h^\pi(x,a) := r_h(x,a) + \Exp{\textstyle\sum_{h'=h+1}^H R_{h'} ~\Big|~ X_{h+1} \sim T_h\left(\cdot \mid x,a\right)} ~~\text{ for }~~ R_{h'} \sim R_{h'}(X_{h'},A_{h'}^{\pi}).
\end{align}
\end{definition}
Under suitable assumptions on $\S \times \A$ there exists a deterministic optimal policy $\pi^\star$ which gives the optimal value $V_h^\star(x) = \sup_\pi V_h^\pi(x)$ for all $x \in \S$ and $h \in [H]$~\citep{puterman_1994}. For ease of notation we denote $Q^\star = Q^{\pi^\star}$ and $V^\star = V^{\pi^\star}$. 
Recall the Bellman equations~\citep{puterman_1994} which state that,
\begin{align}
\label{eqn:bellman_equation}
V_{h}^{\pi}(x) & =Q_{h}^{\pi}\left(x, \pi_{h}(x)\right) &\frall x\in \S\nonumber\\ 
Q_{h}^{\pi}(x, a) &= r_{h}(x, a) + \Exp{V_{h+1}^{\pi}(X_{h+1}) \mid X_h = x, A_h = a} &\frall (x,a)\in \S\times\A \\ 
V_{H+1}^{\pi}(x) & =0 &\frall x \in \S \nonumber.
\end{align}
For the optimal policy $\pi^{\star}$, it additionally holds that $V_h^\star(x) = \max_{a \in \A} Q_h^\star(x,a).$

\noindent \textbf{Online Interaction Structure}: We consider an agent interacting with the MDP in the online setting.  At the beginning of each episode $k$, the agent fixes a policy $\pi^k$ for the entire episode, and is given an initial (arbitrary) state $X_1^k \in \S$. In each step $h \in [H]$, the agent receives the state $X_h^k$, picks an action $A_h^k = \pi^k_h(X_h^k)$, receives reward $R_h^k \sim R_h(X_h^k,A_h^k)$, and transitions to a random state $X_{h+1}^k \sim T_h\left(\cdot \mid X_h^k, A_h^k\right)$ sampled from the transition distribution.  
This continues until the final transition to state $X_{H+1}^k$, at which point the agent chooses policy $\pi^{k+1}$ for the next episode after incorporating observed data (including the rewards and transitions), and the process is repeated.  Their goal is to maximize the total expected reward $\sum_{k=1}^K V_1^{\pi^k}(X_1^k)$. 
We benchmark the agent on their {regret}: the additive loss over all episodes the agent experiences using their policy instead of the optimal one.  As the policies cannot be anticipatory, we introduce $\F_k = \sigma((X_h^{k'}, A_h^{k'}, R_h^{k'})_{h \in [H], k' < k,}, X_1^k)$ to denote the information available to the decision maker at the start of episode $k$.
\begin{definition}
The $\regret$ for an algorithm that deploys a sequence of $\F_k$-measurable policies $\{\pi^k\}_{k \in [K]}$ given a sequence of initial states $\{X_1^k\}_{k \in [K]}$ is defined as:
%\vspace*{-6pt}
\begin{align}
\label{equation:regret}
\regret(K) = \textstyle\sum_{k=1}^K \left( V_1^\star(X_1^k) - V_1^{\pi^k}(X_1^k) \right).
\end{align}
\end{definition}
Our goal will be to develop algorithms which have regret $\regret(K)$ growing sublinearly in $K$, and low per-step storage and computational requirements.

\subsection{Metric Space and Lipschitz Assumptions}
\label{sec:assumptions}

In contrast to parametric algorithms, our algorithms require flexible assumptions on the underlying process.  At a high-level, we require the algorithm to have access to a metric on the state action space under which the underlying $Q_h^\star$ function (or rewards and dynamics) are Lipschitz continuous.  This is well motivated in problems in continuous domains (where the metric can be taken to be any $\ell_p$ metric on the Euclidean space).  For large discrete spaces, it requires an embedding of the space in a metric space that encodes meaningful relationships between the discrete states and actions~\citep{10.1145/1394608.1382172}.  Trivially any problem can be embedded in a metric space where the metric is taken to be the difference of $Q_h^\star$ values, but this requires knowing the optimal values for determining the adaptive partition.  Recent work has investigated the options of selecting a metric in terms of its induced topological structure on the space~\citep{lan2021metrics}.

We assume the state space $\S$ and the action space $\A$ are each separable compact metric spaces with metrics $\D_\S$ and $\D_\A$. We assume that the transition kernels $\{T_h(\cdot \mid x,a)\}_{x,a \in \S \times \A}$ are Borel measures with respect to the topology induced by $\D_\S$ on $\S$.  These metrics imposes a metric structure $\D$ on $\S \times \A$ via the product metric, or any sub-additive metric such that $\D((x,a), (x',a')) \leq \D_\S(x,x') + \D_\A(a,a').$     

The covering dimension of a compact metric space $\X$ is defined as $d_\X = \min \{ d > 0 : N_r(\X) \leq cr^{-d} \,\, \forall r > 0 \}$ with $N_r(\X)$ as the $r$-packing number of the set $\X$.  This metric structure on $\S \times \A$ ensures that the covering dimension of $\S \times \A$ is at most $d = d_\S + d_\A$, where $d_\S$ and $d_\A$ are the covering dimensions of $\S$ and $\A$ respectively.  For notational brevity we will omit which metric the packing numbers are computed with respect to, as it should be clear from the context.

We assume w.l.o.g. that $\S\times\A$ has diameter $1$, and we denote the diameter of $Y \subset \S$ as $\D(Y) = \sup_{a\in\A,(x,y)\in Y^2}\D((x,a),(y,a))$ and overload notation and use $\diam{B}=\max\{\D((x,a),(y,b)) \mid (x,a),(y,b)\in B\}$ to be the diameter of a region $B \subset \S \times \A$.  For more information on metrics and covering dimension, see \cite{slivkins_2014,Kleinberg:2019:BEM:3338848.3299873, Sinclair_2019,royden1988real} for a summary.

To motivate the discretization approach, we also assume Lipschitz structure on the system.  This can come in two forms, which we call \emph{model-free Lipschitz} and \emph{model-based Lipschitz} (required for \AdaQL and \AdaMB respectively).  We start with \emph{model-free Lipschitz}, which assumes the underlying $Q_h^\star$ and $V_h^\star$ functions are Lipschitz continuous:
\begin{assumption}[Model-Free Lipschitz]
	\label{assumption:Lipschitz_mf}
	$Q_h^\star$ and $V_h^\star$ are Lipschitz continuous with respect to $\D$ and $\D_\S$, i.e. for every $x,x',a,a',h \in \S^2 \times \A^2 \times [H]$,
	\begin{align*}
		|Q_h^\star(x,a) - Q_h^\star(x', a')| & \leq L_V \D((x,a), (x',a')) \\
		|V_h^\star(x) - V_h^\star(x')| & \leq L_V \D_\S(x, x').
	\end{align*}
\end{assumption}

The next assumption, \emph{model-based Lipschitz}, puts Lipschitz assumptions on the underlying rewards and dynamics of the system.
\begin{assumption}[Model-Based Lipschitz]
\label{assumption:Lipschitz_mb}
    The average reward function $r_h(x,a)$ is Lipschitz continuous with respect to $\D$, and the transition kernels $T_h(\cdot \mid x,a)$ are Lipschitz continuous in the $1$-Wasserstein metric $d_W$ with respect to $\D$,i.e. for every $x,x',a,a',h \in \S^2 \times \A^2 \times [H]$,
    \begin{align*}
        |r_h(x,a) - r_h(x', a')| & \leq L_r \D((x,a), (x',a')) \\
        d_W(T_h(\cdot \mid x,a), T_h(\cdot \mid x',a')) & \leq L_T \D((x,a),(x',a')).
    \end{align*}
    %We further assume that $Q_h^\star$ and $V_h^\star$ are also $L_V$-Lipschitz continuous for some constant $L_V$.  
\end{assumption}

It's important to note that \cref{assumption:Lipschitz_mb} implies \cref{assumption:Lipschitz_mf}.  %See \cite{lan2021metrics} for further discussion relating the Lipschitz models.

\begin{lemma}
\label{lem:relation_lipschitz}
Suppose that \cref{assumption:Lipschitz_mb} holds.  Then \cref{assumption:Lipschitz_mf} holds with $L_V = \sum_{h=0}^{H} L_r L_T^{h}$.
\end{lemma}

The next assumption is similar to previous literature for algorithms in general metric spaces \citep{Kleinberg:2019:BEM:3338848.3299873,slivkins_2014,Sinclair_2019}.  This assumes access to the similarity metrics.  Learning the metric (or picking the metric) is important in practice, but beyond the scope of this paper \citep{wanigasekara2019nonparametric,lan2021metrics}.

\begin{assumption}
The agent has oracle access to the similarity metrics via several queries that are used by the algorithm.
\end{assumption}

In particular, \AdaMB and \AdaQL require access to several covering and packing oracles that are used throughout the algorithm.  For more details on the assumptions required and implementing the algorithm in practice, see \cref{app:full_algo}.

\subsection{Zooming Dimension}
\label{sec:zooming_dim}

Our theoretical guarantees scale with respect to an instance dependent \emph{zooming dimension} of $\S \times \A$ instead of the ambient dimension.  This serves as an analog to the zooming dimension originally appearing in instance dependent bounds in the bandit literature~\citep{slivkins_2014} extended to dynamic settings, and a continuous analog to instance dependent guarantees developed for RL in the tabular setting~\citep{simchowitz2019}.  Analyzing the zooming dimension for reinforcement learning problems is much more technical than in the simpler bandit setting due to having to account for the dynamics of the problem.  We start by introducing the concept of a $\gap$, a quantity measuring the suboptimality of a given state action pair $(x, a) \in \S \times \A$.
\begin{definition}\label{def:gap}
For any $(x,a) \in \S \times \A$ and $h \in [H]$, the \textbf{stage dependent sub optimality gap} is $\gap_h(x,a) = V_h^\star(x) - Q_h^\star(x,a).$
\end{definition}
One can interpret $\gap_h(x,a)$ as a measure of \emph{regret} the algorithm experiences upon taking action $a$ in state $x$ in step $h$ instead of the optimal action.  This definition simplifies to the same definition of $\gap$ developed in the contextual bandit literature when the transition distribution of the problem is independent of the given state and action.

In bandits, existing results have shown that adaptive discretization algorithms only discretize a subset of the entire state and action set, defined as the set of points whose gap is small.  While we will later see in \cref{sec:proof_sketch} that the same does not extend to reinforcement learning, we are still able to bound the regret of the algorithm based on the size of a set of near optimal points.

\begin{definition}
We define the \textbf{near-optimal set} of $\S \times \A$ for a given value $r$ as $$Z_h^r = \{(x,a) \in \S \times \A \mid \gap_h(x,a) \leq C_L(H+1) r \}$$
where $\conslip$ is an absolute constant depending on the Lipschitz constants of the problem.
\end{definition}

Clearly we have that $Z_h^r \subset \S \times \A$.  However, for many problem instances, $Z_h^r$ could be a (much) lower dimensional manifold.  Finally, we are define the step $h$ zooming dimension as
\begin{definition}
\label{def:zooming}
The \textbf{step $h$ zooming dimension} with constant $c_h$ is $z_h$ such that $$z_h = \inf\{ d > 0 : N_r(Z_h^r) \leq c_h r^{-d} \,\, \forall r > 0\}.$$

We also denote $z_{max} = \max_{h \in [H]} z_h$ to be the worst-case zooming dimension across all of the steps.
\end{definition}

To give some intuition behind the zooming dimension first notice that while the covering dimension is focused on covering the entire metric space, the zooming dimension focuses instead on covering a near-optimal subset of it.  This serves as a way to quantify the benignness of a problem instance.  While it is trivially no larger than the zooming dimension, in many settings it can be significantly smaller.  % As some examples:

\begin{lemma}\label{lem:zooming_dim_value}
The following examples show improved scaling of the zooming dimension over the ambient dimension:
\begin{itemize}
    \item {Linear $Q_h^\star$:} Suppose that $Q_h^\star(x,a) = \theta^\top (x,a)$ for some vector $\theta \in \mathbb{R}^{d_\S + d_\A}$ with $\S \subset \mathbb{R}^{d_\S}$ and $\A \subset \mathbb{R}^{d_\A}$ under any $\ell_p$ norm.  Then $z_h \leq d_\S + d_\A - \norm{\theta_{\A}}_0$.
    \item {Low-Dimensional Optimality:} Suppose that there exists a set $Y \subset \A$ which contains all optimal or near-optimal actions for every state.  Then $z_h \leq d_\S + d_Y$.  
    \item {Strongly Concave:} Suppose that the metric space is $\S = [0,1]^{d_\S}$ and $\A = [0,1]^{d_\A}$ under any $\ell_p$ metric.  If $Q_h^\star(x,a)$ is $C^2$ smooth, and for all $x \in \S$ we have that $Q_h^\star(x, \cdot)$ has a unique maxima and is strongly concave in a neighborhood around the maxima, then $z_h \leq d_\S + \frac{d_\A}{2}.$
\end{itemize}
\end{lemma}

Analyzing the zooming dimension in reinforcement learning is more complicated than in bandit settings as you have to show properties of the $Q_h^\star$ function which is coupled by the dynamics of the system.  In the experiments in \cref{sec:experiments} we highlight problem instances with improved bounds on the zooming dimension.

Note that all of the examples presented in \cref{lem:zooming_dim_value} have $z_h \geq d_\S$.  This is as the zooming dimension does not take into account the distribution over states which are visited by the optimal policy.  As such, scaling with respect to $d_\S$ is inevitable since the set $\{(x, \pi_h^\star(x)) : x \in \S\}$ is contained in $Z_h^r$ for any $r > 0$.  This results in the following lower bound on the zooming dimension:
\begin{lemma}
\label{lem:zoom_d_s}
For any $h$ we have that $z_h \geq d_\S - 1$.
\end{lemma}

Even in the simpler contextual multi-armed bandit model the zooming dimension necessarily scales with the dimension of the context space, regardless of the support or mass over the context space the context distribution places.  While analytically we cannot show gains in the state space dimension, we see empirically in \cref{sec:experiments} that the algorithms only \emph{cover the state space in regions the optimal policy visits}, but it is unclear how to include this intuition formally in the definition.  Revisiting new notions of ``instance specific'' complexity is an interesting direction for future work in both tabular and continuous RL.
\begin{algorithm*}[t]

	\begin{algorithmic}[1]
		\Procedure{Adaptive Discretization for Online RL}{$\S, \A, \D, H, K, \pfail$}
			\State Initialize partitions $\P_h^0 = \S\times\A$ for $h\in[H]$, estimates $\Qhat{h}{0}(\cdot) = \Vhat{h}{k}(\cdot) = H-h+1$
			\For{each episode $k \gets 1, \ldots K$}
				\State Receive starting state $X_1^k$
				\For{each step $h \gets 1, \ldots, H$}
					\State Observe $X_h^k$ and determine $\relevant_h^k(X_h^k) = \{B\in \P_h^{k-1} \mid X_h^k\in B \}$
					\State Greedy \textsc{Selection Rule}: 
					%$$B_h^k = \argmax_{B \in \text{RELEVANT}_h^k(X_h^k)} \Qhat{h}{k}(B)$$
					pick $B_h^k = \argmax_{B \in \text{RELEVANT}_h^k(X_h^k)} \Qhat{h}{k-1}(B)$
                    \State Play action $A_h^k = \tilde{a}(B_h^k)$ associated with ball $B_h^k$; receive $R_h^k$ and transition to  $X_{h+1}^k$
					%\State Receive reward $R_h^k$ and transition to new state $X_{h+1}^k$
				\State \textsc{Update Estimates}$(X_h^k, A_h^k, X_{h+1}^k, R_h^k, B_h^k)$ via \AdaMB or \AdaQL
					\If{$\conf_h^k(B_h^k) \leq \diam{B_h^k}$}
					  \textproc{Refine Partition}$(B_h^k)$
					\EndIf 
				\EndFor
			\EndFor
		\EndProcedure
	\Procedure{Refine Partition}{$B$, $h$, $k$}
		    \State Construct $\P(B) = \{B_1, \ldots, B_m\}$ as the children of $B$ in the hierarchical partition
		    \State Update  $\P_h^k =\P_h^{k-1} \cup \P(B) \setminus B$ 
		    \State For each $B_i$, initialize estimates from parent ball
		    %(i.e., inhering from the parent ball).
		 \EndProcedure
	\Procedure{Update Estimates (\AdaMB)}{$X_h^k, A_h^k, X_{h+1}^k, R_h^k, B_h^k$}
		\For{each $h \gets 1, \ldots H$ and $B \in \P_h^k$}
		 : Update $\Qhat{h}{k}(B)$ and $\Vhat{h}{k}(\cdot)$ via~\cref{eq:q_update} and~\cref{eq:def-V}
		\EndFor
	\EndProcedure
	
	\Procedure{Update Estimates (\AdaQL)}{$X_h^k, A_h^k, X_{h+1}^k, R_h^k, B_h^k$}
		\State Update $\Qhat{h}{k}(B)$ via \cref{eqn:update}
	\EndProcedure
	\end{algorithmic}
	\caption{Adaptive Discretization for Online Reinforcement Learning (\AdaMB, \AdaQL)}
	\label{alg:brief}
\end{algorithm*}

\section{Algorithm}
\label{sec:algorithm}

Reinforcement learning algorithms come in two primary forms: policy-based learning or value-based learning. Policy-based learning focuses on directly iterating on the policy used in episode $k$ $\pi^k$ through optimizing over a set of candidate policies~\citep{pmlr-v130-bhandari21a}. Our algorithms use value-based learning, which focuses on constructing estimates  $\Qhat{h}{k}$ for $Q_h^\star$.  The algorithms then play the policy $\pi_h^k$ which is greedy with respect to the estimates, i.e. $$\pi_h^k(X_h^k) = \argmax_{a \in \A} \Qhat{h}{k-1}(X_h^k,a).$$

The motivation behind this approach is that the optimal policy plays $\pi_h^\star(X_h^k) = \argmax_{a \in \A} Q_h^\star(X_h^k,a)$.  The hope is that when $\Qhat{h}{}$ is uniformly close to $Q_h^\star$, then the value of the policy used will be similar to that of the optimal policy.  What distinguishes between different value-based algorithms is the method used to construct the estimates for $Q_h^\star$. We provide both model-based and model-free variations of our algorithm.  Note that for both algorithms we set $\Vhat{h}{k}(x) = \max_{a \in \A} \Qhat{h}{k}(x,a)$.
\begin{enumerate}
    \item \textit{Model-Based} (\AdaMB): Estimates the MDP parameters directly ($r_h$ and $T_h$) and plugs in the estimates to the Bellman Equations~\cref{eqn:bellman_equation} to set $$\Qhat{h}{k}(x, a) \approx \rbar{h}{k}(x, a) + \E_{Y \sim \Tbar{h}{k}(x, a)}[\Vhat{h+1}{k}(Y)]$$
    where $\rbar{}{}$ and $\Tbar{}{}$ denote explicit empirical estimates for the reward and transition kernel.  
    \item \textit{Model-Free} (\AdaQL): Foregoes estimating the MDP parameters directly, and instead does one-step updates in the Bellman Equations~\cref{eqn:bellman_equation} to obtain $\Qhat{h}{k}(x,a) \approx (1 - \alpha_t) \Qhat{h}{k-1}(x,a) + \alpha_t(R_h^k + \Vhat{h+1}{k}(X_{h+1}^{k}))$.  By decomposing the recursive relationship, you get  $$\Qhat{h}{k}(x, a) \approx \sum_{i=1}^t \alpha_t^i(R_h^{k_i} + \Vhat{h+1}{k_i}(X_{h+1}^{k_i}))$$ where $t$ is the number of times $(x,a)$ has been visited and $k_1, \ldots, k_t$ denote the episodes it was visited before, $\alpha_t$ is the learning rate, and $\alpha_t^i = \alpha_i \prod_{j = i+1}^t (1 - \alpha_j)$.  Note that this instead stores \emph{implicit} estimates of the average reward (weighted by the learning rate) and the transition kernel.  The key difference is that the estimate of the next step is not updated based on the current episode $k$ as $\Vhat{h+1}{k}(\cdot)$, but instead is updated based on the episode it was visited in the past by $\Vhat{h+1}{k_i}(\cdot)$. As a result the learning rates are chosen to impose \emph{recency bias} for the estimates~\citep{jin_2018}.  If $\Vhat{h+1}{k}(\cdot)$ was used the algorithm would need to store all of the data and recompute these quantities, leading to regret improvements only in logarithmic terms.  We will see later that this approach leads to substantial time and space complexity improvements.
\end{enumerate}

\noindent \textbf{Algorithm Description}: We now present our Model-Based and Model-Free Reinforcement Learning algorithms with Adaptive Partitioning, which we refer to as \AdaMB and \AdaQL respectively.  Our algorithms proceed in the following four steps:
\begin{enumerate}
    \item \textit{(Adaptive Partition)}: The algorithms maintain an adaptive hierarchical partition of $\S \times \A$ for each step $h$ which are used to represent the (potentially continuous) state and action space.  The algorithms also maintain estimates $\Qhat{h}{}$ and $\Vhat{h}{}$ for $Q_h^\star$ and $V_h^\star$ over each region which are constructed based on collected data.
    \item \textit{(Selection Rule)}: Upon visiting a state $X_h^k$ in step $h$ episode $k$, the algorithms pick a region containing $X_h^k$ which maximizes the estimated future reward $\Qhat{h}{k-1}$.  They then play any action contained in that region.
    \item \textit{(Update Estimates)}: After collecting data $(X_h^k, A_h^k, R_h^k, X_{h+1}^k)$ with the observed reward and transitions, the algorithms update the estimates $\Qhat{h}{k}$ and $\Vhat{h}{k}$ which are used for the next episode.
    \item \textit{(Splitting Rule)}: After observing data and updating the estimates, the algorithm determines whether or not to sub-partition the current region.  This is a key algorithmic step differentiating adaptive discretization algorithms from their uniform discretization counterparts.  Our splitting rule arises by splitting a region once the confidence in its estimate is smaller than its diameter, forming a bias-variance trade off.
\end{enumerate}

For a discussion on implementation details, see \cref{app:full_algo}.  We now describe the four steps of the algorithms in more detail.

\begin{figure*}[!t]
\centering
\subfigure[Illustrating the state-action partitioning scheme]{\label{fig:partition_diagram}
  \resizebox{.62\columnwidth}{!}{\input{parts/partitionexampledetailed}}
}
\subfigure[Partitioning in practice]{\label{fig:partition_practice}\includegraphics[width=.34\columnwidth]{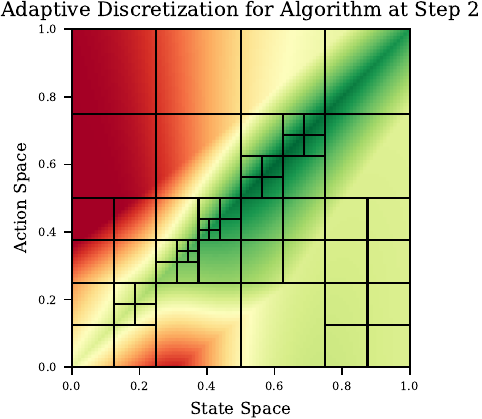}}
\caption{\em Partitioning scheme for $\S\times\A=[0,1]^2$: In \cref{fig:partition_diagram}, we illustrate our scheme. Partition $\Pkh[k-1]$ is depicted with corresponding tree (showing active balls in green, inactive parents in red). The algorithm plays ball $B_{h-1}^k$ in step $h-1$, leading to new state $X_h^k$. Since $\lev{B_{h-1}^k}=2$, in \AdaMB we store transition estimates $\Tbar{h-1}{k}(\cdot \mid B_{h-1}^k)$ for all subsets of $\S$ of diameter $2^{-2}$ denoted as $\S_{2}$ (depicted via dotted lines). The set of relevant balls $\relevant_h^k(X_h^k) = \{B_4,B_{21},B_{23}\}$ are highlighted in blue.  $\S(\P_h^{k-1})$ here would be $\{[0,\frac{1}{2}],[\frac{1}{2}, \frac{3}{4}],[\frac{3}{4}, 1]\}$.\\
In \cref{fig:partition_practice}, we show the partition $\P_{2}^K$ from one of our synthetic experiments. The colors denote the true $Q_2^\star$ values, with green corresponding to higher values. Note that the partition is more refined in areas which have higher $Q_2^\star$.}
\label{fig:partition}
\end{figure*}

\medskip

\subsection{Adaptive Partition}

We first start by introducing the adaptive hierarchical partitions which are used by the algorithm to efficiently discretize the state and action space for learning and computation~\citep{slivkins_2014,Kleinberg:2019:BEM:3338848.3299873,munos2014bandits}.  For each step $h \in [H]$, \AdaMB and \AdaQL maintains a partition of the space $\S \times \A$ into a collection of ``balls'', which is refined over episodes $k \in [K]$.  We use the term ball loosely here, as in general the regions form a \emph{non}-disjoint partition of $\S \times \A$.  However, in our diagrams we will use the $\ell_\infty$ metric on $\mathbb{R}^{2}$ under which the metric-``balls'' indeed form a discrete partition of the space.  In particular, we assume that the algorithms are given access to a hierarchical partition of $\S \times \A$ of the following form:
\begin{definition}
\label{def:partition}
A \textbf{hierarchical partition} of $\S \times \A$ is a collection of disjoint regions $\P_{\ell}$ for levels $\ell = 0, \ldots, K$ such that
\begin{itemize}
    \item Each region $B \in \P_\ell$ is of the form $\S(B) \times \A(B)$ where $\S(B) \subset \S$ and $\A(B) \subset \A$
    \item $\P_0 = \{ \S \times \A \}$
    \item For every $\ell$ we have that $\S \times \A = \cup_{B \in \P_\ell} B$
    \item For every $\ell$ and $B \in \P_\ell$ we have $\diam{B} \leq 2^{- \ell}$
    \item For any two regions $B_1, B_2 \in \P_\ell$ their centers are at a distance at least $2^{-\ell}$ from each other
    \item For each $\ell$ and $B \in \P_\ell$ there exists a unique $A \in \P_{\ell - 1}$ (referred to as the parent of $B$) such that $B \subsetneq A$
\end{itemize}
\end{definition}
While one can construct a hierarchical partition for any compact metric space, a canonical example to keep in mind is  $\S=[0,1]^{d_\S},\A=[0,1]^{d_\A}$ with the infinity norm $\D((x,a),(x',a')) = ||(x,a)-(x',a')||_{\infty}$. Here one can simply take $\P_\ell$ to be the dyadic partition of $\S \times \A$ into regions of diameter $2^{-\ell}$.  Moreover, each hierarchical partition admits a tree, whereby the root corresponds to the node representing $\S \times \A$ and there is a node for each region $B \in \P_\ell$ with an edge to their unique parent which is the region $A \in \P_{\ell - 1}$ such that $B \subsetneq A$.

Our algorithm maintains an adaptive partition $\P_h^k$ over the hierarchical partition for each step $h$ and $k$ which can be thought of as a sub-tree of the original hierarchical partition.  Originally we set $\P_h^0 = \{\S \times \A\}$ for every $h$.  Over time, the partition is refined by adding new nodes to the sub-tree originating from the hierarchical partition.  The leaf nodes represent the \emph{active} balls, and inactive \emph{parent} balls of $B\in\Pkh$ corresponding to $\{B'\in\Pkh[k] \mid B'\supset B\}$; moreover, $\lev{B}$ is the depth of $B$ in the tree (with the root at level $0$). See Figure~\ref{fig:partition} for an example partition and tree generated by the algorithm.  We let $\bcenter(B) = (\tilde{x}(B), \tilde{a}(B))$ be the center of $B$.  

\AdaMB requires the induced state partition of the adaptive partition when maintaining estimates of the value function, and a hierarchical partition over the state space for maintaining estimates of the transition kernel.  We use the following notation
\begin{align}
\S(\P_h^k) := \bigcup_{B \in \P_h^k ~\text{s.t.}~ \nexists B' \in \P_h^k, \S(B') \subset \S(B)} \S(B) \quad \text{ and } \quad \dyad{\ell} = \{\S(B) \mid B \in \P_\ell\}\label{eq:induced_state_partition_def}
\end{align}
to denote the partition over the state spaced induce by the current state-action partition $\P_h^k$ and the hierarchical partition over the state space at level $\ell$.  For example, in \cref{fig:partition_diagram} we have a representation of $\S(\P_h^{k-1})$ and $\S_2$.

\subsection{Selection Rule}

Upon visiting state $X_h^k$ in step $h$ of episode $k$, the algorithms find all of the relevant balls defined via $\relevant_h^k(X_h^k) = \{\text{ active } B\in\Pkh[k-1] \mid (X_h^k,a)\in B \text{ for some } a \in \A\}$.  The algorithm then selects the \textbf{selected ball}: $$B_h^k = \argmax_{B \in \relevant_h^k(X_h^k)} \Qhat{h}{k-1}(B).$$
Once the selected ball $B_h^k$ has been chosen, the algorithm plays either action $\tilde{a}(B_h^k)$, or any distribution over actions $a$ such that $(X_h^k, a) \in B_h^k$.

\subsection{Maintaining Estimates}

At a high level, for each active ball $B \in \P_h^k$, our algorithms maintain the following statistics:
\begin{itemize}
    \item $n_h^k(B)$: the number of times the ball $B$ or its ancestors in the tree have been \textit{selected} up to and including episode $k$
    \item $\Qhat{h}{k}(B)$: an estimate for $Q_h^\star(x,a)$ for points $(x,a) \in B$.
\end{itemize}
In addition, \AdaMB also maintains estimates of the rewards and transitions via:
\begin{itemize}
    \item $\rbar{h}{k}(B)$: the empirical reward earned from playing actions in $B$ and its ancestors
    \item $\Tbar{h}{k}(\cdot \mid B)$: the empirical fractions of transitions from playing actions in $B$ and its ancestors to sets in a $2^{-\lev{B}}$-coarse partition of $\S$ formed by taking the projection of the hierarchical partition to the state space (which we denoted as $\dyad{\lev{B}}$).
\end{itemize}

\noindent \textbf{Update Counts}: After playing action $A_h^k$ in state $X_h^k$ at step $h$, the algorithm transitions to a new state $X_{h+1}^k$ and observes a reward $R_h^k$. Each algorithm then updates the number of samples collected from the selected ball $B_h^k$ via $n_h^k(B_h^k) = n_h^{k-1}(B_h^k) + 1$.  In addition, \AdaMB updates the estimates for the average reward and transition as follows:
\begin{itemize}
\item Update average reward: $\rbar{h}{k}(B_h^k) = \frac{n_h^{k-1}(B_h^k)\rbar{h}{k-1}(B_h^k) + R_h^k}{n_h^{k}(B_h^k)}.$
\item Update $\Tbar{h}{k}(\cdot \mid B_h^k)$ as follows: For each set $A$ in a $2^{-\lev{B_h^k}}$-coarse partition of $\S$ denoted by $\dyad{\lev{B_h^k}}$, we set
\[\Tbar{h}{k}(A\mid B_h^k) = \frac{n_h^{k-1}(B_h^k)\Tbar{h}{k-1}(A\mid B_h^k)+\mathds{1}_{\{X_{h+1}^k\in A\}}}{ n_h^{k}(B_h^k)}.\]
This is maintaining an empirical estimate of the transition kernel for a ball $B_h^k$ at a level of granularity proportional to its diameter.  We use this to ensure that the error of the estimates for a ball $B$ is proportional to its diameter, but this also serves as an efficient compression of the data collected so that the algorithm isn't required to store all collected samples.  Other model-based algorithms forego this step, suffering from added computational and storage burden by maintaining the full empirical transition kernel~\citep{domingues2020regret,shah2018q}.
\end{itemize}

With this in place, we now describe how \AdaMB and \AdaQL use these statistics in order to generate estimates for $\Qhat{h}{k}(B)$ over the partition.

\subsubsection{Compute Estimates - \AdaMB}

We start by defining confidence terms according to:

\begin{align*}
    & \rbonus{h}{k}(B) = \sqrt{\frac{2\log(2HK^2/\pfail)}{n_h^k(B)}} \quad
    \tbonus{h}{k}(B) = \begin{cases}
    L_V \left(4 \sqrt{\frac{\log(2HK^2 / \pfail)}{n_h^k(B)}} + c \left(n_h^k(B)\right)^{-1/d_\S}\right) &\text{ if } d_\S > 2 \\
    L_V \left(4 \sqrt{\frac{\log(2HK^2/\pfail)}{n_h^k(B)}} + c \frac{\log(K)}{\sqrt{n_h^k(B)}} \right) &\text{ if } d_\S \leq 2 \end{cases} \\
    & \bias(B) =  4 L_r \diam{B} + L_V (5 L_T + 4) \diam{B}.
\end{align*}
where $c$ is an absolute constant and the difference in definitions of $\tbonus{h}{k}(\cdot)$ comes from the dimension dependence in Wasserstein concentration.  The first term corresponds to the uncertainty in the rewards, the second to the uncertainty in the dynamics of the environment, and the third the biases in the ball due to aggregation.  With these in place we compute the estimate $\Qhat{h}{k}(B)$ according to
\begin{align}
\label{eq:q_update}
    \Qhat{h}{k}(B) & := \begin{cases}
    & \rbar{h}{k}(B) + \rbonus{h}{k}(B) + \bias(B) \hfill \text{ if } h = H \\
    & \rbar{h}{k}(B) + \rbonus{h}{k}(B) + \E_{A \sim \Tbar{h}{k}(\cdot \mid B)}\big[\Vhat{h+1}{k}(A)\big]
		+ \tbonus{h}{k}(B) + \bias(B)\hfill \text{ if } h < H\end{cases}
\end{align}
The value function estimates are computed in a two-stage process.  We first define $\Vtilde{h}{k}(A)$ over the balls $A \in \S(\P_h^k)$ according to
\begin{align}
\label{eq:v_tilde_update}
    \Vtilde{h}{k}(A) := \min\{\Vtilde{h}{k-1}(A), \max_{B \in \P_h^k: \S(B) \supseteq A} \Qhat{h}{k}(B) \}.
\end{align}
We define $\Vhat{h}{k}(x)$ to extrapolate the estimate across all points in the state space in a Lipschitz manner.  For each point $x \in \S$ we define
\begin{align}
\label{eq:def-V}
    \Vhat{h}{k}(x) = \min_{A' \in \S(\P_h^k)} \left(\Vtilde{h}{k}(A') + L_V \D_S(x, \tilde{x}(A') \right).
\end{align}
The Lipschitz property is used to show concentration of expectations over $\Vhat{h}{k}(\cdot)$ taken with respect to transition kernel estimates.  As the support of $\Tbar{h}{k}(\cdot \mid B)$ is only over sets in $\dyad{\lev{B}}$ we overload notation to let $\Vhat{h}{k}(A) = \Vhat{h}{k}(\tilde{x}(A))$. We equivalently overload notation so that $x \sim \Tbar{h}{k}(\cdot \mid B)$ refers to sampling over the centers associated to balls in $\dyad{\lev{B}}$.

This corresponds to a value iteration step, where we replace the true rewards and transitions in the Bellman Equations (\cref{eqn:bellman_equation}) with their estimates. We compute full updates instead of one step updates as in  \cite{efroni2019tight} for ease of presentation, but see \cite{sinclair2020adaptive} for discussion on the one step update procedure.

\subsubsection{Compute Estimates - \AdaQL}

Next we discuss the update rule $\Qhat{h}{k}$ for \AdaQL.  Fix an episode $k$, step $h$, and ball $B \in \P_h^k$.  Let $t = n_h^k(B)$ be the number of times $B$ or its ancestors have been selected by the algorithm at step $h$ in episodes up to the current episode $k$.  The \textit{confidence} terms are defined according to:
\begin{align*}
\rbonus{h}{k}(B) = 2\sqrt{\frac{H\log(2HK^2 / \pfail)}{n_h^k(B)}} \quad 
\tbonus{h}{k}(B) = 2 \sqrt{\frac{H^3 \log(2HK^2/\pfail)}{n_h^k(B)}} \quad
\bias(B) = 2 L_V \diam{B}.
\end{align*}

The first term corresponds to the uncertainty in the rewards, the second to the uncertainty in the dynamics of the environment, and the third the biases in the ball due to aggregation.

Upon visiting state $X_h^k$ taking action $A_h^k$ with selected ball $B_h^k = B$ we update the estimate for the selected ball $\Qhat{h}{k}(B)$ (leaving all other estimates unchanged) via
\begin{align}
	\label{eqn:update}
	\Qhat{h}{k}(B) = (1 - \alpha_t)\Qhat{h}{k-1}(B) + \alpha_t(R_h^k + \rbonus{h}{k}(B) +  \Vhat{h+1}{k-1}(X_{h+1}^k) + \tbonus{h}{k}(B) + \bias(B))
\end{align}
where $R_h^k$ is the observed reward, $X_{h+1}^k$ is the state the agent transitions to, $\alpha_t = \frac{H+1}{H+t}$ is the learning rate, and \begin{align}
	\label{eqn:v_update}
	\Vhat{h+1}{k-1}(x) = \min(H, \max_{B \in \texttt{RELEVANT}_{h+1}^{k-1}(x)} \Qhat{h+1}{k-1}(B))
\end{align} is our estimate of the expected future reward for being in a given state.  Note that with this recursive update, we can equivalently write $\Qhat{h}{k}(B)$ by unraveling the recursion.  First, denote \begin{equation}
    \label{eq:lr}
    \alpha_t^i = \alpha_i \prod_{j=i+1}^t (1 - \alpha_j).
\end{equation} 
\begin{lemma}[Recursive Relationship for \AdaQL]
\label{lemma:recursive_relationship}
	For any $h, k \in [H] \times [K]$ and ball $B \in \P_h^k$ let $t = n_h^k(B)$ be the number of times that $B$ or its ancestors were encountered during the algorithm before episode $k$.  Further suppose that $B$ and its ancestors were encountered at step $h$ of episodes $k_1 < k_2 < \ldots < k_t \leq k$.  By the update rule of $Q$ we have that:
	$$ \Qhat{h}{k}(B) = \Ind{t = 0} H + \sum_{i=1}^t \alpha_t^i \left(R_h^{k_i} + \rbonus{h}{k_i}(B_h^{k_i}) + \Vhat{h+1}{k_i}(X_{h+1}^{k_i}) + \tbonus{h}{k_i}(B_h^{k_i}) + \bias(B_h^{k_i})\right).$$
\end{lemma}

\begin{rproof}{\cref{lemma:recursive_relationship}}
We show the claim by induction on $t = n_h^k(B)$.  

First suppose that $t = 0$, i.e. that the ball $B$ has not been encountered before by the algorithm.  Then initially $\Qhat{h}{0}(B) = H = \Ind{t = 0} H$.

Now for the step case we notice that $\Qhat{h}{k}(B)$ was last updated at episode $k_t$.  $k_t$ is either the most recent episode when ball $B$ was encountered, or the most recent episode when its parent was encountered if $B$ was activated and not yet played.  In either case by the update rule (Equation~\ref{eqn:update}) we have
\begin{align*}
\Qhat{h}{k}(B) & = (1 - \alpha_t)\Qhat{h}{k_t-1}(B) + \alpha_t\left(R_h^{k_t} + \rbonus{h}{k_t}(B_h^{k_t}) + \Vhat{h+1}{k_t}(X_{h+1}^{k_t}) + \tbonus{h}{k_t}(B_h^{k_t}) + \bias(B_h^{k_t})\right) \\
& = (1 - \alpha_t)\alpha_{t-1}^0 H + (1 - \alpha_t) \sum_{i=1}^{t-1} \alpha_t^i \left(R_h^{k_i} + \rbonus{h}{k_i}(B_h^{k_i}) + \Vhat{h+1}{k_i}(X_{h+1}^{k_i}) + \tbonus{h}{k_i}(B_h^{k_i}) + \bias(B_h^{k_i})\right) \\ & + \alpha_t\left(R_h^{k_t} + \Vhat{h+1}{k_t}(X_{h+1}^{k_t}) + \tbonus{h}{k_t}(B_h^{k_t}) + \bias(B_h^{k_t}) \right) \text{ by the induction hypothesis}\\
& = \Ind{t = 0} H + \sum_{i=1}^t \alpha_t^i \left(R_h^{k_i} + \rbonus{h}{k_i}(B_h^{k_i}) + \Vhat{h+1}{k_i}(x_{h+1}^{k_i}) + \tbonus{h}{k_i}(B_h^{k_i}) + \bias(B_h^{k_i})\right)
\end{align*}
by definition of $\alpha_t^i$. \Halmos
\end{rproof}

\subsection{Splitting Rule}
\label{subsec:split_rule}

To refine the partition over episodes, we split a ball when the confidence in its estimate is smaller than its bias.  Formally, due to the Lipschitz property on the $Q$ function from the assumption we know that the bias in the estimates is proportional to the diameter of the ball.  In episode $k$, step $h$, we split the selected ball $B_h^k$ if 
\begin{align}
\label{eq:splitting_rule}
    \conf_h^k(B_h^k) \leq \diam{B_h^k}.
\end{align}
where $\conf_h^k(B_h^k) = \confcons / n_h^k(B)^{\alpha}$ is the dominating term of $\rbonus{h}{k}(B)$ and $\tbonus{h}{k}(B)$ for some polylogarithmic constant $\confcons$.  In particular we take

\begin{align*}
        \conf_h^k(B) & := \begin{cases}
    & 4 \sqrt{\frac{H^3 \log(2HK^2 / \pfail)}{n_h^k(B)}}  \text{ \AdaQL: } \alpha = \frac{1}{2} \\
    & 4 L_V c \sqrt{\frac{\log(2HK^2 / \pfail)}{n_h^k(B)}} \text{ \AdaMB }\,\, d_\S \leq 2: \, \alpha = \frac{1}{2} \\
    & \frac{4L_V c \sqrt{\log(2HK^2 / \pfail)}}{n_h^k(B)^{1 / d_\S}} \text{ \AdaMB }\,\, d_\S > 2: \, \alpha = \frac{1}{d_\S}.
    \end{cases}
\end{align*}

This splitting rule differs from previous analysis of adaptive discretization for RL~\citep{Sinclair_2019,sinclair2020adaptive,cao2020provably}, and is key for achieving the proper instance dependent guarantees.  By taking the splitting threshold to depend explicitly on the dominating term in the confidence bounds we are able to upper bound the bias of the ball (proportional to the diameter) by these confidence terms, as we will later see in the regret analysis.

In episode $k$ step $h$, if we need to split $B^{par} = B_h^k$ then we add the child nodes in the hierarchical partition immediately under $B_h^k$ to form $\P_h^k$.  Each child ball then inherits all estimates and counts from its parent ball in the adaptive partition, with the exception of the estimate of the transition distribution in \AdaMB.

Recall that $\Tbar{h}{k}(\cdot \mid B^{par})$ defines a distribution over $\dyad{\lev{B^{par}}}$, whereas $\Tbar{h}{k}(\cdot \mid B)$ for children $B$ of $B^{par}$ should define a distribution over $\dyad{\lev{B}}$. As a result, for \AdaMB we need to additionally update the transitional kernel estimates to map to a distribution over $\dyad{\lev{B}}$ by splitting the mass equally over sub-regions according to
$$\Tbar{h}{k}(A \mid B) = 2^{-d_\S} \Tbar{h}{k}(A^{par} \mid B^{par})$$
where $B^{par}$ is the parent of $B$, and $A^{par}$ is the parent of $A$, i.e. the unique element in $\dyad{\lev{B^{par}}}$ such that $A \subset A^{par}$. 
Each element is weighted by $2^{-d_\S}$ to split the mass evenly, as each element of $\dyad{\lev{B^{par}}}$ has $2^{d_\S}$ children in $\dyad{\lev{B}}$.
\section{Main Results}
\label{sec:main_results}

In this section we outline the main results for our paper.  We provide guarantees for \AdaMB and \AdaQL on all three aspects of the \textbf{trifecta for RL in OR}, including regret bounds, storage requirements, and computational complexity for the algorithms.  We refer the readers to \cref{tab:comparison_of_bounds} for a summary of the main results.

\subsection{Regret Minimization Guarantees}

We start by providing instance dependent bounds on the regret for \AdaMB and \AdaQL.  These bounds explicitly depend on the zooming dimension as outlined in \cref{sec:zooming_dim} measuring the complexity of a problem instance instead of the ambient dimension.

\begin{theorem}
\label{thm:regret}
	Let $z_h$ be the step $h$ zooming dimension and $d_\S$ be the covering dimension of the state space.  Then the regret of \AdaMB and \AdaQL for any sequence of starting states $\{X_1^k\}_{k=1}^K$ is upper bounded with probability at least $1 - 3 \delta$ by
	\begin{align*}
	\regret(K) & \lesssim \begin{cases}
	H^{3/2} \sqrt{K} + L \sum_{h=1}^H K^{\frac{z_h+d_\S - 1}{d+d_\S}} \quad \AdaMB: d_\S > 2 \\
	H^{3/2} \sqrt{K} + L \sum_{h=1}^H K^{\frac{z_h+ 1}{z_h+ 2}} \quad \AdaMB: d_\S \leq 2 \\
	LH^{3/2} \sum_{h=1}^H K^{\frac{z_h + 1}{z_h+2}} \quad \AdaQL 
	\end{cases}
	\end{align*}
	where $L = 1 + L_r + L_V + L_V L_T$ and $\lesssim$ omits poly-logarithmic factors of $\frac{1}{\delta}, H,K,$ $d$, and any universal constants.
\end{theorem}

\noindent \textbf{Comparison Between Model-Free and Model-Based Methods}: As we see from \cref{thm:regret}, the bounds for \AdaMB have better dependence on the number of steps $H$.  This is expected, as current analysis for model-free and model-based algorithms under tabular settings shows that model-based algorithms achieve better dependence on the horizon. However, under the Lipschitz assumptions the constant $L$ scales with $H$ so the true dependence is masked (see \cref{lem:relation_lipschitz}).  When we compare the dependence on the number of episodes $K$ we see that the dependence is worse for \AdaMB, primarily due to the additional factor of $d_\S$, the covering dimension of the state-space.  This discrepancy arises as model-based algorithms maintain an estimate of the transition kernel, whose worst-case statistical complexity for Wasserstein concentration depends on $d_\S$ when $d_\S > 2$.  In \cref{app:parametric} we give improved bounds for \AdaMB under additional assumptions on the transition distribution that match the performance of \AdaMB.

\medskip

\noindent \textbf{Metric Specific Guarantees}: Our guarantees scale with respect to the zooming dimension instead of the ambient dimension of the metric space.  As the zooming dimension can be much smaller than the ambient dimension (see \cref{lem:zooming_dim_value}), our adaptive discretization algorithms are able to achieve exponentially better regret guarantees than other non-parametric algorithms.  Moreover, in the final regret bound in \cref{sec:regret_decomp} we provide more fine-tuned metric dependent guarantees.

\medskip

\noindent \textbf{Comparison to other Non-Parametric Methods}:  Current state of the art model-based algorithms achieve regret scaling like $H^{3} K^{2d/(2d+1)}$ \citep{domingues2020regret}.  We achieve better scaling with respect to both $H$ and $K$, and our algorithm has lower time and space complexity.  However, we require additional oracle assumptions on the metric space to be able to construct packings and coverings efficiently, whereas \textsc{Kernel-UCBVI} uses the data and the metric itself.  Better dependence on $H$ and $K$ is achieved by using recent work on concentration for the Wasserstein metric and by showing zooming dimension guarantees.  These guarantees allow us to construct tighter confidence intervals which are independent of $H$, obviating the need to construct a covering of $H$-uniformly bounded Lipschitz functions like prior work (see \cref{sec:concentration}).

In addition, \textsc{Kernel-UCBVI} uses a fixed bandwidth parameter in their kernel interpolation.  We instead keep an adaptive partition of the space, helping our algorithm maintain a smaller and more efficient discretization and adapting to the zooming dimension of the space instead of the ambient dimension.

\medskip

\noindent \textbf{Policy-Identification Guarantees}: Using similar arguments from \cite{jin_2018} it is straightforward to show sample complexity guarantees on learning a policy of a desired quality in the PAC guarantee framework for learning RL policies \citep{watkins1989learning}.  

\subsection{Lower Bounds}
\label{sec:lower_bounds}

Existing work for the contextual bandit literature has shown that the worst case regret scales exponentially with respect to the zooming dimension~\citep{slivkins_2014}.  This construction can be modified directly to obtain a lower bound for the RL setting as follows.

\begin{theorem}
\label{thm:lower_bound}
Let $(\S, \D_\S)$ and $(A, \D_\A)$ be arbitrary metric spaces, and $\D$ be the product metric.  Fix an arbitrary time horizon $K$ and number of steps $H$.  There exists a distribution $\I$ over problem instances on $(\S \times \A, \D)$ such that for any algorithm: \[
\E_{\I}[\regret(K)] \geq \Omega\left(\sum_{h=1}^H K^{\frac{z_h + 1}{z_h + 2}} / \log(K)\right)
\]
\end{theorem}
This builds on the lower bounding technique from \citep{jaksch2010near, Kleinberg:2019:BEM:3338848.3299873,slivkins_2014} by using a needle in the haystack example.  The haystack consists of several actions with an expected payoff of $\frac{1}{2}$ and the needle an action whose expected payoff is slightly higher.  In fact, the construction from \citep{slivkins_2014} can be used directly by developing a problem instance of length $H$ which is a sequence of $H$ contextual bandit problems.

\noindent \textbf{Comparison to Lower Bounds}: Comparing our regret bounds to the lower bound we see that \AdaQL and \AdaMB (for $d_\S \leq 2$) matches the lower bound with respect to the instance dependent zooming dimension $z_h$.  However, \AdaMB when $d_\S > 2$ has the additionally factors of $d_\S$ due to maintaining the explicit estimate of the transition kernel.

\noindent \textbf{Exponential Scaling on Episodes}: The regret is always upper bounded by $HK$.  In order for \cref{thm:regret} to give a nontrivial regret guarantee then the number of episodes needs to be on the order of $H^{z_{max}}$.  This exponential dependence is a fundamental factor in all estimation problems with nonparametric statistics.  Our work focuses on improving on the exponent by showing exponential sample complexity gains from scaling with the zooming dimension instead of the ambient dimension.  Moreover, the algorithms require no prior knowledge of the zooming dimension in order to achieve these regret bounds.  Theoretically speaking, the value of these results are limited when the dimension of the space starts to grow.  However, experimental results show the value of nonparametric algorithms in practice~\cite{araujo2020control,araujo2020single,10.1145/1394608.1382172}.

\subsection{Space and Time Complexity Guarantees}

Next we consider the storage and time complexity of both of our algorithms.  In particular, we are able to show the following:
\begin{theorem}
\label{thm:space_storage}
The storage and time complexity for \AdaMB and \AdaQL can be upper bound via 
	\begin{align*}
	\textsc{Space}(K) & \lesssim \begin{cases}
	HK \quad \AdaMB: d_\S > 2 \\
	HK^{\frac{d+d_\S}{d+d_\S+2}} \quad \AdaMB: d_\S \leq 2 \\
	HK^{\frac{d}{d+2}} \quad \AdaQL 
	\end{cases} \quad 	\textsc{Time}(K) & \lesssim \begin{cases}
	HK^{\frac{d+2d_\S}{d+d_\S}} \quad \AdaMB: d_\S > 2 \\
	HK^{\frac{d+d_\S+2}{d+d_\S}} \quad \AdaMB: d_\S \leq 2 \\
	HK\log_d(K) \quad \AdaQL
	\end{cases}
	\end{align*}
\end{theorem}

\noindent \textbf{Comparison Between Model-Free and Model-Based Methods}: As we can see, both the storage and time complexity for \AdaQL are uniformly better than that of \AdaMB.  This should not come as a surprise, as even in the tabular setting model-free algorithms have better storage and computational requirements than model-based ones as they forego maintaining and using explicit estimates of the transition kernel.

\medskip

\noindent \textbf{Comparison to other Non-Parametric Methods}: As seen in \cref{tab:comparison_of_bounds}, our bounds are uniformly better for both storage and time complexity than other nonparametric algorithms.  These gains are primarily due to using the discretization to maintain an efficient compression of the data while maintaining statistical accuracy, and utilizing quantizing techniques to speed up the algorithms.

\medskip

\noindent \textbf{Monotone Increasing Run-Time and Storage Complexity}: The run-time and storage complexity guarantees presented are monotonically increasing with respect to the number of episodes $K$.  However, to get sublinear minimax regret in a continuous setting for nonparametric Lipschitz models, the model complexity must grow over episodes.  In practice, one would run our adaptive discretization algorithms until running out of space - and our experiments show that the algorithms use resources (storage and computation) much better than a uniform discretization.

\medskip
\noindent \textbf{Comparison to Lower-Bounds}: To the best of our knowledge, there are no existing results showing storage or computational lower-bounds for an RL algorithm in continuous spaces.

\section{Proof Sketch}
\label{sec:proof_sketch}

We start with giving the proof sketch of \cref{thm:regret} before going into the proof of \cref{thm:space_storage} in \cref{sec:space_storage}.  The high level proof of \cref{thm:regret} is divided into three sections.  First, we show \emph{concentration, clean-events, and optimism}, under which our estimates constitute upper bounds on their relevant quantities.  Afterwards, we show a \emph{regret decomposition with clipping}, which upper bounds the difference between the estimated value and the value accumulated by the algorithm as a function of the confidence terms.  This uses the \emph{clipping} operator first introduced in \cite{simchowitz2019} for obtaining instance-dependent regret guarantees in tabular settings.  Lastly, we use an argument to bound the \emph{sum of confidence terms} which is used for the final regret bound.  We include a brief discussion here into each of the three parts, but the full details are included in the appendix.  

\subsection{Concentration, Clean Events, and Optimism (\cref{sec:concentration,sec:optimism})}

\AdaMB explicitly maintains estimates $\rbar{h}{k}(B)$ and $\Tbar{h}{k}(\cdot \mid B)$ for the unknown rewards and transitions of the underlying MDP.  Similarly, \AdaQL \emph{implicitly} maintains estimates for the rewards and transitions, where the rewards are taken via $\sum_{t} \alpha_t^i R_h^{k_i}$ and the transitions are taken using old estimates of $\Vhat{h}{k_i}$ instead of $\Vhat{h}{k}$.  In order to ensure that the one-step value iteration update in \cref{eq:q_update,eqn:update} concentrates we need to verify that these estimates provide good approximations to their true quantities.  In particular, we show that
\begin{align*}
    \text{ \AdaMB : } & \begin{cases}
    |\rbar{h}{k}(B) - r_h(x,a)| & \lesssim \rbonus{h}{k}(B) + L_r \diam{B} \\
    d_W(\Tbar{h}{k}(\cdot \mid B) - T_h(\cdot \mid x,a)) & \lesssim \tbonus{h}{k}(B) + L_V L_T \diam{B} \end{cases} \\
    \text{ \AdaQL : } & \begin{cases}
    |\sum_{i=1}^t \alpha_t^i(R_h^{k_i} - r_h(x,a))| & \lesssim \rbonus{h}{k}(B) \\
    |\sum_{i=1}^t \alpha_t^i(V_{h+1}^\star(X_{h+1}^{k_i}) - \E_{Y \sim T_h(\cdot \mid X_h^{k_i}, A_h^{k_i})}[V_{h+1}^\star(Y)]| & \lesssim \tbonus{h}{k}(B).
    \end{cases}
\end{align*}

However, recall that our estimates for $\Qhat{h}{k}(B)$ are constructed via
$$\Qhat{h}{k}(B) = \tilde{r}_{h}^{k}(B) + \rbonus{h}{k}(B) + \tilde{T}_{h}^{k}(\Vhat{h+1}{\tilde{k}}(B)) + \tbonus{h}{k}(B) + \bias(B)$$
where $\tilde{r}$ and $\tilde{T}$ vary for the two different algorithms.  As such, 
the concentration results leads to upper and lower bounds on $\Qhat{h}{k}(B)$ via the bonus terms of the form:
\begin{informaltheorem}
For any $(h, k) \in [H] \times [K]$, $B \in \P_h^k$ and $(x,a) \in B$ we have that
$$0 \leq \Qhat{h}{k}(B) - Q_h^\star(x,a) \lesssim \conf_h^k(B) + \bias(B) + f_{h+1}^k$$
where $f_{h+1}^{k}$ is an algorithm-dependent term depending on the estimates at step $h+1$.  
\end{informaltheorem}

\subsection{Upper Bound via Clipping (\cref{sec:clipping})}

We use an argument based on the \emph{clipping} operator first introduced in \cite{simchowitz2019} for obtaining instance-dependent regret guarantees for tabular reinforcement learning.  We define $\clip{\mu \mid \nu} = \mu \Ind{\mu \geq \nu}$.  The value of this function is zero until $\mu \geq \nu$, and afterwards it takes on the value of $\mu$.  We use this operator to bound the regret at step $h$ in episode $k$ in two terms.  The first part corresponds to the \emph{clipped} bonus terms and the bias of the estimate using concentration and Lipschitz properties of the estimation procedure.  The second term is $f_{h+1}^k$, the algorithm-dependent quantity measuring the down-stream effects of errors at the next time-step.

In particular, consider the ball $B_h^k$ selected by the algorithm in step $h$ of episode $k$.  Letting $(X_h^k, A_h^k)$ be the state-action pair played at that timestep, we note that
\begin{align*}
    \gap_h(B_h^k) & = \min_{(x,a) \in B_h^k} \gap_h(x,a) \leq \gap_h(X_h^k, A_h^k) = V_h^\star(X_h^k) - Q_h^\star(X_h^k, A_h^k) \\
    & \leq \Qhat{h}{k-1}(B_h^k) - Q_h^\star(X_h^k, A_h^k) \leq \conf_h^{k-1}(B_h^k) + \bias(B_h^k) + f_{h+1}^{k-1}.
\end{align*}
Via some simple algebraic manipulations, we are able to show that this gives
\begin{equation}
\label{eq:clip_bound}
\Qhat{h}{k-1}(B_h^k) - Q_h^\star(X_h^k, A_h^k) \leq \clip{\conf_h^{k-1}(B_h^k) + \bias(B_h^k) \mid \frac{\gap_h(B_h^k)}{H+1}} + \left(1 + \frac{1}{H}\right) f_{h+1}^{k-1}.
\end{equation}

This expression can be thought of as bounding the one-step regret of the algorithm by a term scaling with respect to the confidence in the estimates, and a second term scaling with the downstream misestimation errors.

\subsection{Regret Bound via Splitting Rule (\cref{sec:regret_decomp})}

Lastly, we use the previous equation to develop a final regret bound for the algorithm.  In particular, consider the quantity $\Delta_h^k = \Vhat{h}{k-1}(X_h^k) - V_h^{\pi^k}(X_h^k)$.  Via optimism and the greedy selection rule we have that $$\regret(K) = \sum_{k=1}^K V_1^\star(X_1^k) - V_1^{\pi^k}(X_1^k) \leq \sum_{k=1}^K \Vhat{1}{k-1}(X_1^k) - V_1^{\pi^k}(X_1^k) = \sum_{k=1}^K \Delta_1^k.$$

We use the bound on $\Qhat{h}{k-1}(B_h^k) - Q_h^\star(X_h^k, A_h^k)$ and the definition of $f_{h+1}^{k-1}$ to show that for each algorithm
$$\sum_{k=1}^K \Delta_h^k \lesssim \sum_{k=1}^K \clip{\conf_h^{k-1}(B_h^k) + \bias(B_h^k) \mid \frac{\gap_h(B_h^k)}{H+1}} + \left(1 + \frac{1}{H}\right) \sum_{k=1}^K \xi_{h+1}^{k} + \sum_{k=1}^K \Delta_{h+1}^k$$ where $\xi_{h+1}^{k}$ is an algorithm dependent martingale difference sequence.  Using this and recursing backwards we have that
$$\regret(K) \lesssim \sum_{h=1}^H \sum_{k=1}^K \clip{\conf_h^{k-1}(B_h^K) + \bias(B_h^k) \mid \frac{\gap_h(B_h^k)}{H+1}} + \text{lower order terms}.$$

By properties of the splitting rule for $B_h^k$, $\diam{B_h^k} \leq \conf_h^{k-1}(B_h^k) \leq 4 \diam{B_h^k}$.  Moreover, the $\bias(B_h^k)$ is of the form $\conslip \diam{B_h^k}$ for some Lipschitz-dependent constant $\conslip$.  Thus we get the term inside of the clipping operator can be upper bound by $(\conslip + 1) \conf_h^{k-1}(B_h^k) \leq 4(\conslip + 1) \diam{B_h^k}$.

By definition of the clip operator we only need to consider when $4(\conslip + 1) \diam{B_h^k} \geq \frac{\gap_h(B_h^k)}{H+1}$.  However, this implies that $\gap_h(B_h^k) \leq 4 (H+1)(\conslip + 1) \diam{B_h^k}$.  Letting $(x_c, a_c)$ denote the center of $B_h^k$ we can show using the Lipschitz assumption that 
$$\gap_h(x_c, a_c) \leq \gap_h(B_h^k) + 2 L_V \diam{B_h^k} \leq (4(H+1)(\conslip + 1) + 2 L_V) \diam{B_h^k} \leq \conslip(H+1)\diam{B_h^k}$$ and so $(x_c, a_c)$ lies in the set $Z_h^r$ for $r = \diam{B_h^k}$ by redefining the constant $\conslip$.
The final regret bound follows by replacing the clipping operator with the indicator that the center of the ball lies in the near-optimal set, and considering the scaling of the confidence terms.

This regret derivation for our adaptive discretization algorithm serves in contrast to typical guarantees seen for ``zooming algorithms'' in contextual bandits~\cite{slivkins_2014}.  In particular, in contextual bandits we can show that a ball $B$ is sampled proportional to its diameter for action elimination (essentially eliminating the second term in \cref{eq:clip_bound}).  However, in RL we have to account for downstream uncertainty, requiring a more nuanced analysis using the clipping argument for the final regret bound.

\subsection{Bound on Time and Storage Complexity (\cref{sec:lp_bound})}
\label{sec:space_storage}

We use properties of the splitting rule in order to generate bounds on the size of the partition which are needed for the time and space complexity guarantees.  In particular, we formulate these quantities as a linear program where the objective function is to maximize a sum of terms associated to a valid adaptive partition (represented as a tree) constructed by the algorithm.  The constraints follow from conditions on the number of samples required before a ball is split into subsequent children balls.  To derive an upper bound on the value of the LP we find a tight dual feasible solution.  This argument could be more broadly useful and modified for problems with additional structures by including additional constraints into the LP.  Start by defining a quantity $\nplus{\lev{B}}$ as an upper bound of the number of times a ball at level $\lev{B}$ is sampled before it is split (which we call the splitting thresholds).  Based on the splitting rule we have that $\nplus{\lev{B}} \approx \diam{B}^{-1/\alpha}$ where $\alpha$ is determined by the dominating term of the bonus terms (see \cref{subsec:split_rule}).  However, noting that $\diam{B} \approx 2^{-\lev{B}}$ we get that $\nplus{\lev{B}} \approx 2^{\gamma \lev{B}}$.

\begin{lemma}
\label{lem:LPbound_here}
Consider any partition $\Pkh$ for any $k\in[K], h\in[H]$ induced with splitting thresholds $\nplus{\lev{B}} = \phi 2^{\gam \lev{B}}$, and consider any `penalty' vector $\{a_{\ell}\}_{\ell\in\NN_0}$ that satisfies $a_{\ell+1} \geq a_{\ell} \geq 0$ and $2a_{\ell+1}/a_{\ell} \leq \nplus{\ell}/\nplus{\ell-1}$ for all $\ell\in\NN_0$. Define $\ell^{\star} = \inf\{\ell \mid 2^{d(\ell - 1)} \nplus{\ell-1} \geq k\}$. Then
\begin{align*}
\sum_{\ell=0}^{\infty}\sum_{B\in \Pkh : \ell(B) = \ell} a_{\ell} \leq 2^{d\ell^{\star}}a_{\ell^{\star}}     \end{align*}
\end{lemma}

One immediate corollary of~\cref{lem:LPbound_here} is a bound on the size of the partition $\P_h^k$ for any $h$ and $k$ by taking $a_\ell = 1$ for every $\ell$.  In particular, one can show that
\begin{corollary}
\label{lem:size_partition_here}
For any $h$ and $k$ we have that 
$|\P_h^k| \leq 4^d\left(\frac{k}{\phi}\right)^{\frac{d}{d+\gam}}$ and that 
$\ell^\star \leq \frac{1}{d + \gam} \log_2(k / \phi) + 2.$ 
\end{corollary}

These results are used to bound the storage and time complexity of the algorithms by noting the dominating complexity terms in the algorithm description, writing the total accumulation by formulating them as an LP, and applying the result from \cref{lem:LPbound_here}.  
\section{Experimental Results}
\label{sec:experiments}

In this section we give full details on the experiments and simulations performed.  For full code implementation and more results please see the Github repository at \url{https://github.com/seanrsinclair/ORSuite}.  This codebase contains an implementation for \AdaMB and \AdaQL, and the uniform discretization algorithms under arbitrary-dimensional metric spaces~\cite{orsuite}.  We also note that our algorithms have also been tested heuristically in the infinite horizon discounted setting with a single partition and additional Boltzman exploration \citep{araujo2020control,araujo2020single}.

Efficient algorithms for reinforcement learning in operations tasks is still largely unexplored. It is, however, a very natural objective in designing systems where agents must learn to navigate an uncertain environment to maximize their utility. 
The main objective for continuous spaces in reinforcement learning is to meaningfully store continuous data in a discrete manner while still producing optimal results in terms of performance and reward.  We consider ``oil discovery'' and ``ambulance routing'' problems that are simple enough so we can realistically produce uniform discretization benchmarks to test our adaptive algorithms against.  At the same time, they provide interesting continuous space scenarios that suggest there can be substantial improvements when using adaptive discretization in real world problems by analyzing their instance dependent zooming dimension. The problems also allows us to naturally increase the state and action space dimension and consequently test our algorithms in a slightly more complex setting.  

We compare \AdaMB, \AdaQL, \textsc{Model-Free $\epsilon$-Net} \citep{song2019efficient}, and an $\epsilon$-net variant of UCBVI \citep{azar2017minimax}.  We also include a comparison of the Proximal Policy Optimization algorithm from the \textsc{Stable-Baselines} package which uses a neural network policy class and performs gradient descent on the policy value objective~\citep{stable-baselines3,schulman2017proximal}.  We refer to the simulations as \AdaMB, \AdaQL, \EpsQL, \EpsMB, and \textsc{SB PPO} respectively in the figures.

We also note that these experiments serve as a proof of concept for the performance of adaptive discretization algorithms over their uniform counterparts.  Our experiments were conducted in low dimensions to allow us to plot the underlying $Q_h^\star$ functions and exhibit how the adaptive discretizations match the contours of the $Q_h^\star$ function.  Our comparisons against a deep RL algorithm also help highlight the pitfalls of these approaches on systems applications by comparing the algorithms on the \textbf{trifcta for RL in OR}.  For more robust experiments, see \cite{araujo2020control,araujo2020single,10.1145/1394608.1382172} which compare variants of adaptive discretization algorithms on systems applications, showing improved empirical performance against other state-of-the-art benchmarks.  

\noindent \textbf{Metrics Included}: For each of the simulations we provide results on the following three metrics:
\begin{itemize}
    \item \textsc{Reward}: Computes $\Exp{\sum_{h=1}^H r_h(X_h^k, \pi_h^k(X_h^k))}$, the expected cumulative rewards for the policy at a given episode $k$.  As the regret is difficult to calculate for general problem instances we compare the algorithms on the observed cumulative rewards.
    \item \textsc{Time}: Computes the average per-step time complexity used by the algorithms when selecting an action and updating their internal estimates based on observed data measured by the python package \textsc{time}.
    \item \textsc{Space}: Computes the average per-step space complexity used by the algorithms using the python package \textsc{tracemalloc}. 
\end{itemize}
The later two measurements, \textsc{Time} and \textsc{Space} should only be used when measuring the scaling of these metrics with respect to the episodes $K$.  The adaptive discretization code as implemented could be further optimized by using more efficient data structures (but obviating readability in the code).  In contrast, the uniform discretization algorithms use simple numpy arrays and are already optimized so improved performance is as to be expected.  These serve as a proxy for more \emph{fundamental} measures like number of multiplications which translates more directly to how much the algorithms can be accelerated by appropriate hardware, and is not specific to a particular software implementation.  These plots are included to highlight the memory gains in the adaptive algorithms (by maintaining a significantly coarser discretization than their uniform counterparts) while still ensuring more efficient learning.  In \cref{tab:size_partition} we compare the final size of the partition between the adaptive and uniform algorithms, and highlight in the discussion particular problem instances which led to substantial improvements.

\subsection{Oil Discovery}

This problem adapted from \cite{mason2012collaborative} is a continuous variant of the ``Windy Grid World'' environment popular for benchmarking RL algorithms in the \textsc{Open AI Gym}~\cite{sutton2018reinforcement}.  It comprises of an agent surveying a $d$-dimensional map in search of hidden ``oil deposits''.  The world is endowed with an unknown survey function which encodes the probability of observing oil at that specific location.  For agents to move to a new location they pay a cost proportional to the distance moved, and surveying the land produces noisy estimates of the true value of that location.  In addition, due to varying terrain the true location the agent moves to is perturbed as a function of the state and action.  This adds additional noise to the problem instead of deterministic transitions, coining the phrase ``windy'' environments.

To formalize the problem, here the state space $\S = [0,1]^d$ and action space $\A = [0,1]^d$, where the product space is endowed with the $\ell_\infty$ metric.  The reward function is defined as 
\begin{align*}
    r_h(x,a) = \max\{ \min\{ f_h(x,a) - \alpha \norm{x-a}_\ell + \epsilon, 1\}, 0\} 
\end{align*}
where $f_h(x,a)$ is the survey function corresponding to the probability of observing an oil deposit at that specific location, $\alpha$ is a parameter used to govern the transportation cost, $\epsilon$ is independent Gaussian noise, and $\ell$ is the choice of norm.  The transition function is defined as
\begin{align*}
    \Pr_h(\cdot \mid x,a) = \max\{ \min\{ \delta_{a} + N(0, \sigma_h(x,a)^2), 1\}, 0\}
\end{align*}
where again we have truncated the new state to fall within $[0,1]^d$ and the noise function $\sigma_h(x,a)$ allows for varying terrain in the environment leading to noisy transitions. This allows the problem set-up to scale under non-deterministic transitions.  If we take $\sigma_h(x,a) = 0$ we recover deterministic transitions from taking action $a$ to the next state being $a$.  We can also show that in simple settings the zooming dimension leads to theoretical improvements over the uniform discretization algorithms.  In particular, we assume that the survey function $f_h(x,a)$ is independent of the action $a$.

\begin{lemma}
\label{lem:oil_zoom}
In the oil environment with $\sigma_h(x,a) = 0, \alpha = 0$, and $f_h(x,a) = f_h(x)$ is twice continuously differentiable and strongly concave, the zooming dimension is $\frac{3d}{2}$.
\end{lemma}
\begin{rproof}{\cref{lem:oil_zoom}}
We show that $Q_h^\star$ is twice continuously differentiable and strongly concave in $a$, in which the bound on the zooming dimension then follows from \cref{lem:zooming_dim_value}. Indeed, we show the result by induction on $h$.

\noindent \textit{Base Case:} $h = H$.  Here we have that $Q_h^\star(x,a) = f_h(x)$ which is trivially strongly concave and twice continuously differentiable by assumption on $f_h(x)$.  Moreover, this implies that $V_h^\star(x) = f_h(x)$ is similarly twice continuously differentiable and strongly concave.

\noindent \textit{Step Case:} $h+1 \rightarrow h$.  Again we have that $Q_h^\star(x,a) = f_h(x) + V_{h+1}^\star(a). $
This is strongly concave with respect to $a$ as $V_{h+1}^\star(a)$ is by induction.  It is also twice continuously differentiable.  Similarly we have that $V_h^\star(x) = f_h(x) + \max_{a \in \A} V_{h+1}^\star(a) $ is strongly concave in $x$ and twice continuously differentiable.  \Halmos
\end{rproof}
In this problem the ambient dimension of the space is $2d$ so this offers improvement with the zooming dimension up to a factor of $\frac{d}{2}$.  While the optimal policy under these assumptions is trivial $(\pi_h^\star(x) = \argmax_{a \in \A} f_{h+1}(a))$, we believe that under different scenarios the zooming dimension continues to be much smaller than the ambient dimension.

\noindent \textbf{Experiment Set-Up}: We performed several experiments, where we took $\epsilon = N(0,.1)$, the variance on the transitions to be $\sigma_h(x,a) = 0$ or $\frac{1}{2}\norm{x+a}_2$, the dimension $d \in \{1,2\}$, cost parameter $\alpha \in \{0, .1, .5\}$, and the oil distributions as either:
\begin{align*}
    f_h(x,a) & = e^{-2\norm{x-\frac{1}{9}h}_2} \\
    f_h(x,a) & = 1 - \norm{x - \frac{1}{9}h}_2.
\end{align*}
Note that when $\alpha = 0$ and $\sigma_h(x,a) = 0$ we recover the set-up from \cref{lem:oil_zoom}.

\noindent \textbf{Discussion.} Our experiments help illustrate several key phenomena on the performance of our algorithms:
\begin{itemize}
    \item \textbf{Performance of discretization based Algorithms}: In Figure~\ref{fig:oil_perf} we see that \EpsMB and \EpsQL are outperformed by their adaptive counterparts in terms of cumulative rewards.  Moreover, in the setting when $d = 2$ the uniform discretization algorithms are not able to learn efficiently due to the size of the partition which is maintained by these algorithms.
    \item \textbf{Observed Discretization}: In Figure~\ref{fig:disc_oil} we see that the adaptive discretization algorithms exhibit increased granularity in the partition in regions where the underlying $Q_h^\star$ value is large.  This matches the intuition behind the zooming dimension, showing that the discretization directly matches the contours of the underlying $Q_h^\star$ values.
    \item \textbf{Comparison of Size of Partition}: When comparing the size of the partition between the adaptive and uniform algorithms, we see that the adaptive algorithms exhibited substantial improvements over the uniform algorithms, particularly on the problems where $\sigma_h(x,a) = 0$.
    \item \textbf{Comparison to PPO}: We see that in this setting the PPO algorithm requires additional tuning and more iterations in order to guarantee convergence to a near-optimal policy.  This baseline shows the ease-of-use of nonparametric discretization based algorithms as there is only a single hyperparameter to be tuned (the scaling factor in the confidence terms).  Other function approximation techniques requires testing neural network architectures, learning rates, optimizers, etc, all of which requires additional time and storage complexity in order to even match the performance.
\end{itemize}

\subsection{Ambulance Routing}

This problem is a widely studied question in operations research and control, and is closely related to the $k$-server or metrical task system problems.  A controller positions a fleet of $k$ ambulances over $H$ time periods, so as to minimize the transportation costs and time to respond to incoming patient requests.  In our setting, the controller first chooses locations to station the ambulances.  Next, a single request is realized drawn from a fixed $h$-dependent distribution.  Afterwards, one ambulance is chosen to travel to meet the demand, while other ambulances can re-position themselves.

Here the state space $\S = [0,1]^k$ and action space $\A = [0,1]^k$ where $k$ is the number of ambulances, and the product space is endowed with the $\ell_\infty$ metric.  The reward function and transition is defined as follows.  First, all ambulances travel from their initial state $x_i$ to their desired location $a_i$, paying a transportation cost to move the ambulance to location $a_i$.  Afterwards, a patient request location $p_h \sim \F_h$ is drawn from a fixed distribution $\F_h$.  The closest ambulance to $p_h$ is then selected to serve the patient, i.e. let $i^\star = \argmin_{i \in [k]} |a_i - p_h|$ denote the ambulance traveling to serve the patient.  The rewards and transitions are then defined via:
\begin{align*}
    x_i^{new} & = \begin{cases}
        a_i \qquad & i \neq i^\star \\
        p_h \qquad & i = i^\star
        \end{cases} \\
    r_h(x, a) & = 1 - \left(\frac{\alpha}{k^{1 / \ell}} \norm{x - a}_\ell + (1-\alpha) |a_{i^\star} - p_h| \right)
\end{align*}
where $\alpha$ serves as a tunable parameter to relate the cost of initially traveling from their current location $x$ to the desired location $a$, and the cost of traveling to serve the new patient $p_h$, and $\ell$ is the choice of norm.  Note that $\alpha = 1$ corresponds to only penalizing the ambulances for traveling to the initial location and $\alpha = 0$ only penalizes agents for traveling to serve the patient.

Under certain additional assumptions on the environment we are able to show a bound on the zooming dimension which leads to exponential improvement on the theoretical guarantees for the adaptive discretization based algorithms.

\begin{lemma}
\label{lem:ambulance_zoom}
The ambulance environment with $\alpha = 1$ has a zooming dimension of $k$.
\end{lemma} 
\begin{rproof}{\cref{lem:ambulance_zoom}}
A straightforward calculation shows that $\gap_h(x,a) = \norm{x-a}_\ell$.  Thus, the near-optimal set is defined via $Z_h^r = \{(x,a) \in \S \times \A : \norm{x-a}_\ell \leq cr \}$ for some constant $c$.  Noting that $\norm{x}_\infty \leq \norm{x}_\ell$ it suffices to instead cover the set where $\norm{x-a}_\infty \leq cr$.  Let $\A$ be an $r$ packing of $[0,1]^k$ and $\B$ be an  $r$ packing of $[-cr, cr]^k$. We can construct a packing of $Z_h^r$ by considering the set $\{(a, a+b) \mid a \in \A, b \in \B\}$.  A simple derivation shows that this is indeed a packing and its size is bounded by $\tilde{c}r^{-k}$ for some constant $\tilde{c}$.\Halmos
\end{rproof}
Note that in this problem the ambient dimension is $2k$ so this offers improvement with the zooming dimension up to a factor of $k$.  Again, while the optimal policy in this set-up is trivial $(\pi_h^\star(x) = x)$, the hope is that under different scenarios the zooming dimension continues to be smaller than the ambient dimension.

\noindent \textbf{Additional Algorithms}: In addition to the discretization based algorithms, we also tested the algorithms against two heuristic algorithms:
\begin{itemize}
    \item \textsc{Stable}: This algorithm always selects $\pi_h(x) = x$, which is provably optimal when $\alpha = 1$.
    \item \textsc{Median}: This algorithm maintains a dataset $D_h$ of arrivals observed in step $h$ up to the current episode $k$.  The algorithm then takes the past call arrival data and splits it into $k$ quantiles, using the center data point in each quantile as the location for an ambulance.  This algorithm is motivated under when $\alpha = 0$, where with a single ambulance the optimal policy is to move the ambulance to the true median of the arrival distribution.  This algorithm serves as a non-anticipatory counterpart to this, as it replaces the true median with an estimated median based on collected data.
\end{itemize}
  
\noindent \textbf{Experiment Set-Up}: We performed several experiments, where we took the arrival distribution $\F_h = \text{Beta}(5,2)$ or Uniform over shifting windows.  We took the norm $\ell = 2$, the number of ambulances $k \in \{1,2\}$, and the alpha parameter $\alpha \in \{0, 0.25, 1\}$.  However, see the attached code-base for further instrumentation to simulate under other parameter regimes.

\noindent \textbf{Discussion.} Similar to the oil problem, the experiments illustrate key aspects in the performance of adaptive discretization algorithms:
\begin{itemize}
    \item \textbf{Comparison of Performance}: In Figure~\ref{fig:amb_perf} we see that \EpsMB and \EpsQL are outperformed by their adaptive counterparts in terms of cumulative rewards.  Unsurprisingly, the time complexity of the adaptive algorithms scales better with respect to the number of episodes than their uniform counterparts.
    \item \textbf{Observed Discretization}: In Figure~\ref{fig:disc_amb} we see that the adaptive discretization algorithms exhibit increased granularity in the partition in regions where the underlying $Q_h^\star$ value is large.  Moreover, we saw substantial improvements in the size of the partition on problems where $\alpha = 0$.
      \item \textbf{Comparison to PPO}: We see that in this setting the stable baseline algorithm requires additional tuning and more iterations in order to guarantee convergence to a near-optimal policy.
    \item \textbf{Dependence on $d_\S$}: When $\alpha = 0$ the $Q_h^\star$ value is independent of the state.  The zooming dimension here gives a dependence on $d_\S$, the dimension of the state space.  However, under different arrival distributions the state-visitations will be drastically different.  Moreover, the observed discretizations match the distribution of the arrivals.  This highlights a downfall with the definition of the zooming dimension applied to reinforcement learning as it doesn't take into account the state-visitation frequency under the optimal policy. 
\end{itemize}
\section{Conclusion}
\label{sec:conclusion}

In this paper we presented a unified analysis of model-based and model-free reinforcement learning algorithms using adaptive discretization.  In worst case instances, we showed regret bounds for our algorithms with exponential improvements over other online nonparametric RL algorithms (i.e. the underlying model is Lipschitz continuous with a known metric of the space).  This was partially due to our instance-dependent regret bounds, exhibiting how the discretization and regret scales with respect to the zooming dimension of the problem instead of the ambient dimension.  We provided simulations comparing model-based and model-free methods using an adaptive and fixed discretizations of the space on several canonical control problems.  Our experiments showed that adaptive partitioning empirically performs better than fixed discretizations in terms of both faster convergence and lower memory.

One future direction for this work is analyzing the discrepancy between model-based and model-free methods in continuous settings, as model-based algorithms so far have sub-optimal dependence on the dimension of the state space.  While in \cref{app:parametric} we gave specific instances where the regret of \AdaMB matches \AdaQL, in general the regret has additional dependence on $d_\S$ due to uniform Wasserstein concentration on the state space.  Moreover, we are interested in deriving the ``optimal'' space and time complexity for an algorithm in continuous settings.  We also believe that new hardware techniques can help improve the complexities of implementing these adaptive discretization algorithms in practice.

\section*{Acknowledgments}

Part of this work was done while Sean Sinclair and Christina Yu were visiting the Simons Institute for the Theory of Computing for the semester on the Theory of Reinforcement Learning.  We also gratefully acknowledge funding from the NSF under grants ECCS-1847393, DMS-1839346, CCF-1948256, and CNS-1955997, and the ARL under grant W911NF-17-1-0094.

% 	\pagebreak	
    \bibliographystyle{plain}
    {\bibliography{references}}
    
    \newpage
    \onecolumn
    \appendix

\ifdefined\informs
    \renewcommand{\arraystretch}{.8} % \section{Table of Notation}
\else 
 \renewcommand{\arraystretch}{1.2} \section{Table of Notation}
\label{app:notation}
\fi

{

\begin{table*}[h!]
\begin{tabular}{l|l}
\textbf{Symbol} & \textbf{Definition} \\ \hline
\multicolumn{2}{c}{Problem setting specifications}\\
\hline
$\S,\A,H,K$  & State space, action space, steps per episode, number of episodes\\
$r_h(x,a)\,,\,T_h(\cdot \mid x,a)$ & Average reward/transition kernel for action $a$ in state $x$ at step $h$\\
$\pi_h,V_h^\pi(\cdot),Q_h^\pi(\cdot,\cdot)$ & Arbitrary step-$h$ policy, and Value/$Q$-function at step $h$ under $\pi$ \\
$\pi^{\star}_h, V_h^\star(\cdot),Q_h^\star(\cdot,\cdot)$ & Optimal step-$h$ policy, and corresponding Value/Q-function \\
$L_r,L_T, L_V$ & Lipschitz constants for $r$, $T$, and $V^\star$ respectively\\
$\D_\S$, $D_\A$, $\D$ & Metrics on $\S$, $\A$, and $\S \times \A$ respectively\\
$\regret(K)$ & Regret for the algorithm over $K$ episodes \\
\hline
\multicolumn{2}{c}{Algorithm variables and parameters}\\
\hline
$k,h$ & Index for episode, index for step in episode \\
$(X_h^k, A_h^k, R_h^k)$ & State, action, and received reward at step h in episode k \\
$\Pkh$ & Partition tree of $\S \times \A$ for step $h$ at end of episode $k$\\
$\relevant_h^k(x)$ & Set of balls relevant for $x$ at $(h,k)$ (i.e., $\{B\in\Pkh[k-1] \mid \exists a : (x,a)\in B\}$)\\
$\tilde{x}(B), \tilde{a}(B)$ & Associated state/action for ball $B$ (i.e., `center' of ball $B$) \\
$n_h^k(B)$ & Number of times $B$ has been chosen {by the end of episode $k$}\\
$\rbonus{h}{k}(B), \tbonus{h}{k}(B), \bias(B)$ & Bonus terms for rewards, transition, and bias of a ball $B$ \\
$\conf_h^k(B), \alpha, \gamma$ & Dominating term of bonuses scaling as $n_h^k(B)^{-\frac{1}{\alpha}}$, $\gamma = 1 / \alpha$ \\
$B_h^k$ & Ball in $\Pkh[k-1]$ selected at $(h,k)$
($\argmax_{B \in \text{RELEVANT}_h^k(X_h^k)} \Qhat{h}{k-1}(B)$)
\\
$\Qhat{h}{k}(B), \Vtilde{h}{k}(A), \Vhat{h}{k}(x)$ & $Q$-function estimates for ball $B\in\Pkh$, {at end of} episode $k$\\
$\rbar{h}{k}(B),\Tbar{h}{k}(\cdot \mid B)$ & Inherited reward/transition estimates for $B\in\Pkh$ {at end of episode $k$}\\
\hline
\multicolumn{2}{c}{Definitions used in the analysis}\\
\hline
$\Delta(\S)$ & Set of probability measures on $\S$ \\
$\dyad{\ell}, \S(\P_h^k)$ & Set of dyadic cubes of $\S$ of diameter $2^{-\ell}$, induced state partition from $\P_h^k$\\
$\S(B), \A(B)$ & Projection of a ball $B = B_\S \times B_\A$ to $B_\S$ and $B_\A$ accordingly\\
$\diam{B}$ & The diameter of a ball $B$ \\
$\lev{B}$ & The depth in the tree of ball $B$, equivalent to $\log_2(\D(\S\times\A)/\D(B))$ \\
$\Exp{V_{h+1}(Y) \mid x,a}$ & $\mathbb{E}_{Y \sim T_h(\cdot \mid x,a)} [V_{h+1}(Y)]$\\
$\Delta_h^k$ & $\Vhat{h}{k-1}(X_h^k) - V_h^{\pi^k}(X_h^k)$ \\ 
$\clip{\mu \mid \nu}$ & $\mu \Ind{\mu \geq \nu}$ \\
$\F_{k}$ & Sigma-field  generated by all information up to start of episode $k$\\
\hline
\end{tabular}
\caption{List of common notation}
\label{table:notation}
\end{table*}

}
\section{Properties of the Adaptive Partition}
\label{sec:partition}

In this section we outline some of the invariants established by the adaptive partition maintained by the algorithm.  The first lemma states that the algorithm maintains a partition of the state-action space at every iteration of the algorithm, and that the balls of similar radius are sufficiently far apart.  Before starting, recall the definition of $\conf_h^k(B) = \confcons / n_h^k(B)^\alpha$, where $\confcons$ and  $\alpha$ depend on the algorithm used (see \cref{subsec:split_rule}).

\begin{lemma}
\label{lemma:partition}
For every $(h, k) \in [H] \times [K]$ the following invariants are maintained:
\begin{enumerate}
    \item (Radius to Confidence): For any ball $B \in \P_h^k$, $\conf_h^k(B) \leq \diam{B}$ if and only if $B$ is a parent ball.  Moreover, for any ball $B$ we have $\conf_h^k(B) \leq 2 \diam{B}$.
	\item (Covering): The domains of each ball in $\P_h^k$ cover $\S \times \A$.
	\item (Separation): For any two balls of radius $r$, their centers are at distance at least r.
\end{enumerate}
\end{lemma}
\begin{rproof}{\cref{lemma:partition}}
Let $(h, k) \in [H] \times [K]$ be arbitrary.

\noindent \textit{(Radius to Confidence)}: The first part of this property holds trivially by the splitting rule of the algorithm.  For the second part we note that either $B$
 is a parent, where it is true trivially, or $B$ is a child ball.  For the case when $B$ is a child ball then denote $B^{par}$ as its parent, in which we have
\begin{align*}
    \conf_h^k(B) & = \frac{\confcons}{n_h^k(B)^\alpha} \leq \frac{\confcons}{n_h^k(B^{par})^\alpha} \\
    & = \conf_h^k(B^{par}) \leq \diam{B^{par}} \leq 2\diam{B}.
\end{align*}

\noindent \textit{(Covering)}: For the covering invariant notice that $\P_h^k$ contains a ball which covers the entire space $\S \times \A$ from the initialization in the algorithm from depth zero in the hierarchical partition.  Thus we have that $\S \times \A \subset \cup_{B \in \P_h^k} B$.

\noindent \textit{(Separation)}: This property follows immediately as the adaptive partition is a subtree of the hierarchical partition, and the hierarchical partition maintains that any two balls of a given radius have centers at least $r$ from each other. \Halmos
\end{rproof}

The third property is useful as it maintains that the centers of the balls of radius $r$ form an $r$ packing of $\S \times \A$ and so there are at most $N_r(\S \times \A)$ balls activated of radius $r$.  The next theorem gives an analysis on the number of times that a ball of a given radius will be selected by the algorithm.
\begin{lemma}
	\label{lemma:bound_ball}
	For any $h \in [H]$ and child ball $B \in \P_h^K$ (the partition at the end of the last episode $K$) we have that $$n_{min}(B) \leq n_h^k(B) \leq n_{max}(B)$$ where
	\begin{align*}
	    n_{min}(B)  = \left( \frac{\confcons}{2\diam{B}} \right)^{\frac{1}{\alpha}} \text{ and } n_{max}(B)  = \left(\frac{\confcons}{\diam{B}}\right)^{\frac{1}{\alpha}}.
	\end{align*}
	For the case when $B$ is the initial ball which covers the entire space then the number of episodes that $B$ is selected is only one.
\end{lemma}
\begin{rproof}{\cref{lemma:bound_ball}}
Consider an arbitrary $h \in [H]$ and child ball $B \in \P_h^k$.  Furthermore, let $k$ be the episode for which ball $B$ was activated.  Then $B_h^k$, the ball selected by the algorithm at step $h$ in episode $k$ is the parent of $B$.  Moreover, by the splitting rule we know that $\conf_h^k(B_h^k) \leq \diam{B_h^k}.$  However, plugging in the definition of $\conf_h^k(B_h^k) = \frac{\confcons}{n_h^k(B_h^k)^\alpha}.$  Noting that $n_h^k(B) \geq n_h^k(B_h^k)$ and rearranging the inequality gives that $$n_h^k(B) \geq n_h^k(B_h^k) \geq \left( \frac{\confcons}{\diam{B_h^k}} \right)^{\frac{1}{\alpha}} \geq \left( \frac{\confcons}{2\diam{B}} \right)^{\frac{1}{\alpha}}.$$

Similarly, we know that for any ball $B$ that the number of times it's sampled is at most the number required to split it by the splitting threshold, i.e. when $\conf_h^k(B) \leq \diam{B}$.  However, plugging in the definition and rearranging shows that $$n_h^k(B) \leq \left(\frac{\confcons}{\diam{B}}\right)^{\frac{1}{\alpha}}.$$
	
Lastly if $B$ is the initial ball which covers the entire space then $\diam{B} = 1$ initially and so the ball is split after it is selected only once. \Halmos
\end{rproof}

Next we show that the cumulative bias incurred for an estimate of a given ball $B$ based on all of its ancestors can be bound by the diameter of the ball.  This expression is used as the algorithm aggregates estimates for a region based on its ancestors, so the cumulative \emph{ancestral bias} is at most the bias of the ball.

\begin{lemma}
\label{lem:sum_bias_term_ball}
For all $h, k \in [K] \times [K]$ and $B \in \P_h^k$, let $t = n_h^k(B)$ and episodes $k_1 < \ldots < k_t \leq k$ where $B$ and its ancestors were encountered at step $h$.  Then for \AdaQL we have that \begin{align*}
    \sum_{i=1}^t \alpha_t^i \diam{B_h^{k_i}} & \leq 4 \diam{B}
\end{align*}
and for \AdaMB that
\begin{align*}
    \frac{1}{t} \sum_{i=1}^t \diam{B_h^{k_i}} & \leq 4 \diam{B}.
\end{align*}
\end{lemma}
\begin{rproof}{\cref{lem:sum_bias_term_ball}}
\noindent \textbf{\AdaQL}: For \AdaQL we use the fact that via rearranging \cref{lemma:bound_ball} and definition of the confidence term that $\diam{B_h^{k_i}} \leq \frac{\confcons}{\sqrt{n_h^{k_i}(B_h^{k_i})}}$.  Using that $n_h^{k_i}(B_h^{k_i}) = i$ we have
\begin{align*}
    \sum_{i=1}^t \alpha_t^i \diam{B_h^{k_i}} & = \confcons \sum_{i=1}^t \frac{\alpha_t^i}{\sqrt{i}} \leq \frac{2 \confcons}{\sqrt{t}} \leq 4 \diam{B}
\end{align*}
where the first inequality uses technical properties of the chosen learning rate (\cref{lemma:lr}) and the second uses \cref{lemma:partition}.

\noindent \textbf{\AdaMB}: For \AdaMB we rewrite the summation over the depths in the adaptive partition.  In particular, denoting $B_0, B_1, \ldots, B_{\lev{B}} = B$ as the ancestors of $B$ in the adaptive partition we have
$$\frac{1}{t} \sum_{i=1}^t \diam{B_h^{k_i}} = \frac{1}{\sum_{\ell = 0}^{\lev{B}} |\{k' : B_h^{k'} = B_\ell\}|} \sum_{\ell = 0}^{\lev{B}} |\{k' : B_h^{k'} = B_\ell\}| \diam{B_\ell}.$$  This average can be upper bound by only averaging over the ancestors of $B$, i.e.
$$\frac{1}{t} \sum_{i=1}^t \diam{B_h^{k_i}} \leq \frac{1}{\sum_{\ell = 0}^{\lev{B}-1} |\{k' : B_h^{k'} = B_\ell\}|} \sum_{\ell = 0}^{\lev{B}-1} |\{k' : B_h^{k'} = B_\ell\}| \diam{B_\ell}.$$
However, using \cref{lemma:bound_ball} and the geometric series we can bound
$$ |\{k' : B_h^{k'} = B_\ell\}| = n_{max}(B_\ell) - n_{min}(B_\ell) = \frac{\confcons^{1/\alpha}}{\diam{B_\ell}^{1/\alpha}}\left(1 - \frac{1}{2^{1/\alpha}}\right)$$ and get using $\diam{B_\ell} = 2^{- \ell}$ and setting $\gamma = 1 / \alpha$:
\begin{align*}
    \frac{1}{t} \sum_{i=1}^t \diam{B_h^{k_i}} & \leq \frac{\sum_{i=0}^{\lev{B} - 1} 2^{-i} 2^{\gam i}}{\sum_{i=0}^{\lev{B} - 1} 2^{\gam i}} \\
    &\leq \frac{2^{(\gam-1)(\lev{B}-1)} \sum_{i=0}^{\infty} 2^{-(\gam - 1)i}}{2^{\gam (\lev{B} - 1)}} \\
    &\leq \frac{2 \cdot 2^{(\gam-1)(\lev{B}-1)}}{2^{\gam (\lev{B} - 1)}} \qquad\text{ because $2^{-(\gam-1)} \leq \tfrac12$}\\
    &= 4 \cdot 2^{-\lev{B}} = 4 \diam{B}. \Halmos
\end{align*}
\end{rproof}

\section{Concentration and Clean Events}
\label{sec:concentration}

In this section we show that the bonus terms added on ($\rbonus{h}{k}(B)$ and $\tbonus{h}{k}(B)$) ensure that the estimated rewards and transitions are upper bounds for their true quantities up to an additive term scaling with the Lipschitz constant and diameter of the ball.  This follows a proof technique commonly used for multi-armed bandits and reinforcement learning where algorithm designers ensure that relevant quantities are estimated optimistically with a bonus that decays as the number of samples increases.

For all proofs we recall that $\F_k$ is the filtration induced by all of the information available to the algorithm at the start of episode $k$, i.e. $\F_k = \sigma((X_h^{k'}, A_h^{k'}, R_h^{k'})_{h \in [H], k' < k} \cup X_1^k)$.  With this filtration, all of the estimates $\Qhat{h}{k-1}, \Vhat{h}{k-1}$, and the policy $\pi^k$ are measurable with respect to the filtration $\F_k$.

\subsection{Concentration on Reward Estimates}

We start by showing that with probability at least $1 - \pfail$, the difference between the \emph{explicit} (for \AdaMB) or \emph{implicit} (for \AdaQL) estimate of the average reward and the true reward is bounded by a bonus term scaling as $1 / \sqrt{t}$ and the bias of the ball.

\begin{lemma}[Concentration of Rewards]
\label{lemma:reward_concentration_bound}
For all $(x,a, h, k) \in \S \times \A \times [H] \times [K]$, if $(x,a) \in B$ for some $B \in \P_h^k$,
let us denote $t = n_h^k(B)$, and let $k_1 < \ldots < k_t \leq k$ denote the episodes before $k$ for which $B$ or its ancestors were encountered at step $h$. For all $\pfail \in (0, 1)$ and bounded sequence $(\alpha_t^i)_{i=1}^t$, with probability at least $1 - \pfail$,
$$\left| \sum_{i=1}^t \alpha_t^i \left( R_h^{k_i} - r_h(X_h^{k_i}, A_h^{k_i}) \right) \right| \leq \sqrt{2 \sum_{i=1}^t (\alpha_t^i)^2 \log(2HK^2/\pfail)}.$$
\end{lemma}
\begin{rproof}{\cref{lemma:reward_concentration_bound}}
Consider the sequence $$k_i = \min(k, \min\{ \hat{k} : n^{\hat{k}}_h(B^a) = i \text{ and } B^a \text{ is an ancestor of } B \text{ or } B \text{ itself}\}).$$  $k_i$ denotes the episode for which $B$ or its ancestors were encountered at step $h$ for the $i$-th time, as once a ball is split it is never chosen by the algorithm again.  Setting $Z_i = \Ind{k_i \leq k}\left( R_h^{k_i} - r_h(X_h^{k_i}, A_h^{k_i}) \right)$ then $Z_i$ is a martingale difference sequence with respect to the filtration $\tilde{\mathcal{F}}_i = \F_{k_i}$ which we denote as the information available to the agent up to an including the step $k_i$.  Moreover, as the sum of a martingale difference sequence is a martingale then for any $\tau \leq K$, $\sum_{i=1}^\tau \alpha_\tau^i Z_i$ is a martingale.  As the difference between subsequent terms is bounded by $ \alpha_\tau^{i}$, by applying Azuma-Hoeffding's inequality it follows that for a fixed $\tau \leq K$
\begin{align*}
\Pr \left( \left| \sum_{i=1}^\tau \alpha_\tau^i Z_i \right| \leq \sqrt{2 \sum_{i=1}^\tau (\alpha_\tau^i)^2 \log\left(\frac{2 HK^2 }{\pfail}\right)} \right) \geq 1 - 2\exp \left( - \frac{2 \sum_{i=1}^\tau (\alpha_\tau^i)^2 \log(\frac{2HK^2}{\pfail})}{2 \sum_{i=1}^\tau (\alpha_\tau^i)^2} \right) = 1 - \frac{\pfail}{HK^2}.
\end{align*}

Taking a union bound over the number of episodes, $H$, and the possible values for the stopping time $\tau$ the inequality follows.  Note that we only need to union bound over the number of episodes instead of the number of balls as the inequality is satisfied for all balls not selected in a given round as it inherits its concentration from its parent ball because the value for $t$ does not change. \Halmos
\end{rproof}

Note that if $\alpha_t^i = \frac{1}{t}$ for \AdaMB then the right hand side can be upper bounded by $\sqrt{\frac{2 \log(2HK^2 / \pfail)}{t}}$.  The case when $\alpha_t = \frac{H+1}{H+t}$ for \AdaQL and $\alpha_t^i$ is defined via \cref{eq:lr} we can upper bound it via $\sqrt{\frac{4H \log(2HK^2 / \pfail)}{t}} = \rbonus{h}{k}(B)$ using \cref{lemma:lr}.

Separately for \AdaMB we also include the following, where we include the bias incurred by extrapolating the estimate to a point $(x,a) \in B$.

\begin{lemma}[Concentration of Rewards for \AdaMB]
\label{lemma:ada_mb_reward_confidence}
With probability at least $1 - \pfail$ we have that for any $h, k \in [H] \times [K]$, ball $B \in \P_h^k$, and any $(x,a) \in B$, 
\begin{align*}
    \left|\rbar{h}{k}(B) - r_h(x,a) \right| \leq \rbonus{h}{k}(B) + 4 L_R \diam{B},
\end{align*}
where we define $\rbonus{h}{k}(B) = \sqrt{\frac{2 \log(2HK^2/\pfail)}{n_h^k(B)}}$.
\end{lemma}
\begin{rproof}{\cref{lemma:ada_mb_reward_confidence}}
Let $h, k \in [H] \times [K]$, $B \in \P_h^k$ be fixed, and $(x,a) \in B$ be arbitrary.  Let $t = n_h^k(B)$ and $k_1 < \ldots < k_t \leq k$ be the episodes where $B$ or its ancestors were encountered in step $h$.  Then we have that:
\begin{align*}
    \left| \rbar{h}{k}(B) - r_h(x,a) \right| & = \left| \frac{1}{t} \sum_{i=1}^t R_h^{k_i} - r_h(x,a) \right| \leq \left| \frac{1}{t} \sum_{i=1}^t R_h^{k_i} - r_h(X_h^{k_i}, A_h^{k_i}) + \frac{1}{t} \sum_{i=1}^t r_h(X_h^{k_i}, A_h^{k_i}) - r_h(x,a) \right| \\
    & \leq \left| \frac{1}{t} \sum_{i=1}^t R_h^{k_i} - r_h(X_h^{k_i}, A_h^{k_i}) \right| + \left| \sum_{i=1}^t r_h(X_h^{k_i}, A_h^{k_i}) - r_h(x,a) \right| 
\end{align*}
The first term can be upper bound by $\sqrt{\frac{2 \log(2HK^2 / \pfail)}{t}}$ via \cref{lemma:reward_concentration_bound} where we take $\alpha_t^i = \frac{1}{t}$.  The second term can be upper bound by:
\begin{align*}
    \left| \sum_{i=1}^t r_h(X_h^{k_i}, A_h^{k_i}) - r_h(x,a) \right| & \leq L_r \sum_{i=1}^t \diam{B_h^{k_i}} \leq 4 L_r \diam{B} \text{ by \cref{lem:sum_bias_term_ball}}.
\end{align*}

Combining these we have for any $(h,k) \in [H] \times [K]$ and ball $B \in \P_h^k$ such that $(x,a) \in B$
\begin{align*}
    |\rbar{h}{k}(B) - r_h(x,a)| &  \leq \sqrt{\frac{2\log(2HK^2/\pfail)}{n_h^k(B)}} + 4 L_r \diam{B} = \rbonus{h}{k}(B) + 4 L_r \diam{B}. \Halmos
    % \label{eq:in-app-rucb-raw}
    % &\le \rbonus{h}{k}(B).
\end{align*}
\end{rproof}

\subsection{Concentration of Expected Values w.r.t. Transition Estimates}

Next we show concentration of the estimates of the transition kernel.  For the case of \AdaMB where we maintain explicit estimates of $\Tbar{h}{k}(\cdot \mid B)$ we use results bounding the Wasserstein distance between empirical and true measures.  For the case of \AdaQL, as the empirical transition kernel \emph{implicitly} enters the computation via a martingale update, we instead use Azuma-Hoeffding's inequality.

\begin{lemma}[Concentration of Transition for \AdaQL]
\label{lemma:ada_ql_trans_concentration_bound}
For all $(h, k) \in [H] \times [K]$, and $B \in \P_h^k$, let us denote $t = n_h^k(B)$, and let $k_1 < k_2 < \ldots < k_t \leq k$ denote the episodes before $k$ for which $B$ and its ancestors were encountered at step $h$.
For all $\pfail \in (0, 1)$, with probability at least $1 - \pfail$,
$$\left| \sum_{i=1}^t \alpha_t^i \left( V_{h+1}^\star(X_{h+1}^{k_i}) - \E_{Y \sim T_h(X_h^{k_i}, A_h^{k_i})}[V_{h+1}^\star(Y) \mid X_h^{k_i}, A_h^{k_i}] \right) \right| \leq H\sqrt{2 \sum_{i=1}^t (\alpha_t^i)^2 \log(2HK^2/\pfail)}.$$
If $\alpha_t^i$ is defined in \cref{eq:lr} then the right hand side is upper bound by $\tbonus{h}{k}(B)$.
\end{lemma}
\begin{rproof}{\cref{lemma:ada_ql_trans_concentration_bound}}
Consider the sequence $$k_i = \min(k, \min\{ \hat{k} : n^{\hat{k}}_h(B^a) = i \text{ and } B^a \text{ is an ancestor of } B \text{ or } B \text{ itself}\})$$ as before and set $$Z_i = \Ind{k_i \leq k}\left( V_{h+1}^\star(X_{h+1}^{k_i}) - \Exp{V_{h+1}^\star(Y) \mid X_h^{k_i}, A_h^{k_i}} \right)$$ then $Z_i$ is a martingale difference sequence with respect to the filtration $\tilde{\F}_i = \mathcal{F}_{k_i}$.  As the difference between subsequent terms is bounded by $H \alpha_\tau^{i}$, it follows by Azuma-Hoeffding's inequality that for a fixed $\tau \leq K$

\begin{align*}
\Pr \left( \left| \sum_{i=1}^\tau \alpha_\tau^i Z_i \right| \leq H\sqrt{2 \sum_{i=1}^\tau (\alpha_\tau^i)^2 \log\left(\frac{2 HK^2 }{\pfail}\right)} \right) \geq 1 - 2\exp \left( - \frac{2H^2 \sum_{i=1}^\tau (\alpha_\tau^i)^2 \log(\frac{2HK^2}{\pfail})}{2 H^2 \sum_{i=1}^\tau (\alpha_\tau^i)^2} \right) = 1 - \frac{\pfail}{HK^2}.
\end{align*}

Again the result follows via a union bound over the number of episodes, $H$, and the possible values for the stopping time $\tau$.

By \cref{lemma:lr}, $\sum_{i=1}^t (\alpha_t^i)^2 \leq \frac{2H}{t}$ for $t \geq 0$ so that $$H\sqrt{2 \sum_{i=1}^t (\alpha_t^i)^2 \log(2HK^2/\pfail)} \leq H \sqrt{2 \frac{2H}{t} \log(2HK^2/\pfail)} = 2 \sqrt{\frac{H^3 \log(2HK^2/\pfail)}{t}} = \tbonus{h}{k}(B).$$\Halmos
\end{rproof}

Next for \AdaMB we show that $\tbonus{h}{k}(B)$ indeed serves as a bound on the misspecification of the transition kernel.  As the proof is lengthy, it is deferred to \cref{app:technical_details}.  The main proof uses recent work on bounding the Wasserstein distance between empirical and true measures~\citep{weed2019sharp,boissard2014mean}.  For the case when $d_\S > 2$ the concentration inequalities in \cite{weed2019sharp} hold up to a level of $n_h^k(B)^{-\frac{1}{d_\S}}$ with high probability.  Unfortunately, the scaling does not hold for the case when $d_\S \leq 2$.  In this situation we use work from \cite{boissard2014mean} showing that concentration holds up to a factor of $n_h^k(B)^{-\frac{1}{2}}$ with an additional logarithmic factor.

We use these results by chaining the Wasserstein distance of various measures together (linking the true empirical transition kernel estimate computed by averaging over samples collected from a ball to a \emph{ghost sample} collected at the center).  This allows us to get the improvement on the performance of model-based algorithms to the literature~\cite{domingues2020regret}.

The results from \citep{weed2019sharp,boissard2014mean} have corresponding lower bounds, showing that in the worst case scaling with respect to $d_\S$ is inevitable.  This highlights a phase transition that occurs in the concentration as the dimension of the state space increases.  We also note that as the transition bonus term is the dominating term in the confidence radius, improving on our result necessitates creating concentration intervals around the expectation of the value function instead of the model.  This has been explored recently in \cite{ayoub2020model}.  In \cref{app:parametric} we show improvements on the scaling of the transition bonus terms under additional assumptions on the underlying transitions $T_h(\cdot \mid x,a)$.

\begin{lemma}[Concentration of Transition for \AdaMB]
\label{lemma:transition_confidence}
With probability at least $1 - \pfail$ we have that for any $h,k \in [H] \times [K]$ and ball $B \in \P_h^k$ with $(x,a) \in B$ that 
\begin{align*}
    d_W(\Tbar{h}{k}(\cdot \mid B), T_h(\cdot \mid x,a)) \leq \frac{1}{L_V}(\tbonus{h}{k}(B) + (5 L_T + 4) \diam{B}).
\end{align*}
\end{lemma}

Concentration with respect to the Wasserstein distance as shown in \cref{lemma:transition_confidence} is used to bound the difference in expectations of the estimated value function with respect to the empirical transition kernel.

\section{Optimism}
\label{sec:optimism}

The concentration bounds derived in \cref{sec:concentration} allow us to prove that our value function estimates are an upper bound on the true value function, such that our algorithms display \emph{optimism}.

Before starting, we show the relationship between the optimal $Q$ value $Q_h^\star(x,a)$ and the estimate of the $Q$ values constructed by \AdaQL for any ball $B$ and any $(x,a) \in B$.  For ease of notation we denote $(\Vhat{h}{k}(x) - V_h^\star(x)) = (\Vhat{h}{k} - V_h^\star)(x)$.

\begin{lemma}
\label{lemma:recursive_difference}
	For any $(x, a, h, k) \in \S \times \A \times [H] \times [K]$ and ball $B \in \P_h^k$ such that $(x, a) \in B$, let $t = n_h^k(B)$ and let $k_1 < k_2 < \ldots < k_t \leq k$ denote the episodes before $k$ for which $B$ and its ancestors were encountered at step $h$. For estimates $\Qhat{h}{k}(B)$ constructed according to \AdaQL, 
	\begin{align*}
	\Qhat{h}{k}(B) - Q_h^\star(x,a) & = \Ind{t = 0}(H - Q_h^\star(x,a)) +  \sum_{i=1}^t \alpha_t^i \Big( R_h^{k_i} - r_h(X_h^{k_i}, A_h^{k_i}) + \rbonus{h}{k_i}(B_h^{k_i}) \\
		& + (\Vhat{h+1}{k_i} - V_{h+1}^\star)(X_{h+1}^{k_i}) \\ & + V_{h+1}^\star(X_{h+1}^{k_i}) - \E_{Y \sim T_h(\cdot \mid X_h^{k_i}, A_h^{k_i})}[V_{h+1}^\star(Y)] + \tbonus{h}{k_i}(B_h^{k_i}) \\
		& + Q_h^\star(X_h^{k_i}, A_h^{k_i}) -  Q_h^\star(x,a) + \bias(B_h^{k_i})\Big).
	\end{align*}
\end{lemma}
\begin{rproof}{\cref{lemma:recursive_difference}}
Consider any $(x, a) \in \S \times \A$ and $h \in [H]$ arbitrary and an episode $k \in [K]$.  Furthermore, let $B$ be any ball such that $(x, a) \in B$.  First notice by Lemma~\ref{lemma:lr} that $\Ind{t = 0} + \sum_{i=1}^t \alpha_t^i = 1$ for any $t \geq 0$.
	
For when $t = 0$, $\Qhat{h}{k}(B) = H$ as it has not been encountered before by the algorithm, such that $$\Qhat{h}{k}(B) - Q_h^\star(x,a) = H - Q_h^\star(x,a) = \Ind{t = 0}(H - Q_h^\star(x,a)).$$
	
Otherwise we use that $Q_h^\star(x,a) = \sum_{i=1}^t \alpha_t^i Q_h^\star(x,a) + \Ind{t = 0} Q_h^\star(x,a)$.  Subtracting this from $\Qhat{h}{k}(B)$ and using \cref{lemma:recursive_relationship} yields
\begin{align*}
	\Qhat{h}{k}(B) - Q_h^\star(x,a) & = \Ind{t = 0} H + \sum_{i=1}^t \alpha_t^i \left(R_h^{k_i} + \rbonus{h}{k_i}(B_h^{k_i}) + \Vhat{h+1}{k_i}(x_{h+1}^{k_i})) + \tbonus{h}{k_i}(B_h^{k_i})\right) \\ 
	& - \Ind{t = 0} Q_h^\star(x,a) - \sum_{i=1}^t \alpha_t^i Q_h^\star(x,a) \\
	& = \Ind{t = 0}(H - Q_h^\star(x,a)) + \sum_{i=1}^t \alpha_t^i \Big( R_h^{k_i} + \rbonus{h}{k_i}(B_h^{k_i}) + \Vhat{h+1}{k_i}(X_{h+1}^{k_i}) + \tbonus{h}{k_i}(B_h^{k_i}) \\
	&\qquad\qquad + Q_h^\star(X_h^{k_i}, A_h^{k_i}) - Q_h^\star(X_h^{k_i}, A_h^{k_i}) - Q_h^\star(x,a) + \bias(B_h^{k_i})\Big).
\end{align*}
However, using the Bellman Equations (\ref{eqn:bellman_equation}) we know that \[Q_h^\star(X_h^{k_i}, A_h^{k_i}) = r_h(X_h^{k_i}, A_h^{k_i}) + \E_{Y \sim T_h(\cdot \mid X_h^{k_i}, A_h^{k_i})}[V_{h+1}^\star(Y)].\] 
Substituting this above, we get that
\begin{align*}
		\Qhat{h}{k}(B) - Q_h^\star(x,a) & = \Ind{t = 0}(H - Q_h^\star(x,a)) + \sum_{i=1}^t \alpha_t^i \Big(R_h^{k_i} - R_h(X_h^{k_i}, A_h^{k_i}) + \rbonus{h}{k_i}(B_h^{k_i}) \\
		& + \Vhat{h+1}{k_i}(X_{h+1}^{k_i}) - \E_{Y \sim T_h(\cdot \mid X_h^{k_i}, A_h^{k_i})}[V_{h+1}^\star(Y)]  + \tbonus{h}{k_i}(B_h^{k_i})\\
		& + Q_h^\star(X_h^{k_i}, A_h^{k_i}) -  Q_h^\star(x,a) + \bias(B_h^{k_i})\Big) \\
		& = \Ind{t = 0}(H - Q_h^\star(x,a)) +  \sum_{i=1}^t \alpha_t^i \Big( R_h^{k_i} - r_h(X_h^{k_i}, A_h^{k_i}) + \rbonus{h}{k_i}(B_h^{k_i}) \\
		& + (\Vhat{h+1}{k_i} - V_{h+1}^\star)(X_{h+1}^{k_i}) \\ & + V_{h+1}^\star(X_{h+1}^{k_i}) - \E_{Y \sim T_h(\cdot \mid X_h^{k_i}, A_h^{k_i})}[V_{h+1}^\star(Y)] + \tbonus{h}{k_i}(B_h^{k_i}) \\
		& + Q_h^\star(X_h^{k_i}, A_h^{k_i}) -  Q_h^\star(x,a) + \bias(B_h^{k_i})\Big). \Halmos
\end{align*}
\end{rproof}

With this we are ready to show:
\begin{lemma}[Optimism]
\label{lemma:optimism}
For any $\pfail \in (0,1)$, with probability at least $1 - 2 \pfail$ the following holds simultaneously for all $(x,a,h,k) \in \S \times \A \times [H] \times [K]$, ball $B$ such that $(x,a) \in B$ where $t = n_h^k(B)$ and $k_1 < \ldots \leq k_t$ are the episodes where $B$ or its ancestors were encountered previously by the algorithm:
\begin{align*}
    0 & \leq \Qhat{h}{k}(B) - Q_h^\star(x,a) \leq \conf_h^k(B) + \conslip\bias(B_h^k) + f_{h+1}^k \\
    0 & \leq \Vhat{h}{k}(x) - V_h^\star(x)  \\
    0 & \leq \Vtilde{h}{k}(x) - V_h^\star(x) \quad \text{in \AdaMB} 
\end{align*}
where $\conslip$ is a constant dependent on the Lipschitz constants and $f_{h+1}^k$ is an algorithm-dependent term depending on estimates at step $h+1$, i.e.:
\begin{align*}
    f_{h+1}^k & = \begin{cases}
                \Ind{t = 0}H + \sum_{i=1}^t \alpha_t^i (\Vhat{h+1}{k_i} - V_{h+1}^\star)(X_{h+1}^{k_i}) & \quad \text{in \AdaQL} \\
                \E_{Y \sim T_h(\cdot \mid x,a)}[\Vhat{h+1}{k}(Y)] - \E_{Y \sim T_h(\cdot \mid x,a)}[V_{h+1}^\star(Y)] & \quad \text{in \AdaMB}
                \end{cases}
\end{align*}In addition if $x = X_h^k$ we have that
\begin{align*}
    \Vhat{h}{k-1}(X_h^k) \leq \Qhat{h}{k-1}(B_h^k) + L_V \diam{B_h^k}.
\end{align*}
\end{lemma}
\begin{rproof}{\cref{lemma:optimism}}
Recall that the ``good events'' in (\cref{lemma:ada_mb_reward_confidence,lemma:transition_confidence} for \AdaMB) or (\cref{lemma:ada_ql_trans_concentration_bound,lemma:reward_concentration_bound} for \AdaQL) simultaneously hold with probability $1 - 2 \delta$.  First recall how the estimates for $\Qhat{h}{k}$ are constructed, and subtract out $Q_h^\star$ by the Bellman equations (\cref{eqn:bellman_equation}).

\noindent \textbf{\AdaQL}: From \cref{lemma:recursive_difference}:
\begin{align*}
	\Qhat{h}{k}(B) - Q_h^\star(x,a) & = \underbrace{\Ind{t = 0}(H - Q_h^\star(x,a))}_{(a)} +  \sum_{i=1}^t \alpha_t^i \Big( \underbrace{R_h^{k_i} - r_h(X_h^{k_i}, A_h^{k_i}) + \rbonus{h}{k_i}(B_h^{k_i})}_{(b)} \\
		& + \underbrace{(\Vhat{h+1}{k_i} - V_{h+1}^\star)(X_{h+1}^{k_i})}_{(c)} \\ & + \underbrace{V_{h+1}^\star(X_{h+1}^{k_i}) - \E_{Y \sim T_h(\cdot \mid X_h^{k_i}, A_h^{k_i})}[V_{h+1}^\star(Y)] + \tbonus{h}{k_i}(B_h^{k_i})}_{(d)} \\
		& + \underbrace{Q_h^\star(X_h^{k_i}, A_h^{k_i}) -  Q_h^\star(x,a) + \bias(B_{h}^{k_i})}_{(e)} \Big) \text{ and that } \\
	\Vhat{h}{k}(x) & = \argmax_{B \in \relevant_h^k(x)} \Qhat{h}{k}(B).
\end{align*}

\noindent \textbf{\AdaMB}: Using the Bellman equations on $Q_h^\star(x,a)$ and the definition of $\Qhat{h}{k}, \Vhat{h}{k}$ and $\Vtilde{h}{k}$: 
\begin{align*}
    \Qhat{h}{k}(B) - Q_h^\star(x,a) & = \underbrace{\rbar{h}{k}(B) - r_h(x,a) + \rbonus{h}{k}(B)}_{(a)} \\
    & + \underbrace{\E_{Y \sim \Tbar{h}{k}(\cdot \mid B)}[\Vhat{h+1}{k}(Y)] - \E_{Y \sim T_h(\cdot \mid x,a)}[V_{h+1}^\star(Y)] + \tbonus{h}{k}(B)}_{(b)} \\
    & + \underbrace{\bias(B)}_{(c)} \\
    \Vtilde{h}{k}(A) & = \max_{B \in \P_h^k : \S(B) \supseteq A} \Qhat{h}{k}(B) \\
    \Vhat{h}{k}(x) & = \min_{A \in \S(\P_h^k)} \Vtilde{h}{k}(A) + L_V d_\S(x, \tilde{x}(A)).
\end{align*}

\noindent \textbf{Lower Bound}: We start by showing the lower bounds via induction on $h = H+1, \ldots, 1$.

\noindent \textit{Base Case}: $h = H + 1$.

For the case when $h = H+1$ then we have that $Q_{H+1}^\star(x,a) = 0 = \Qhat{H+1}{k}(B)$ for every $k$, ball $B$, and $(x,a) \in \S \times \A$ trivially at the end of the episode as the expected future reward is always zero.  Similarly, $\Vhat{H+1}{k}(x) = \Vtilde{H+1}{k}(x) = V_{H+1}^\star(x) = 0$ for any $x \in \S$.

\noindent \textit{Step Case}: $h+1 \rightarrow h$.

We start with \AdaQL.  For term $(a)$ we trivially have that $H - Q_h^\star(x,a) \geq 0$ as the rewards are bounded.  For term $(b)$ via \cref{lemma:reward_concentration_bound} we have that 

$$\sum_{i=1}^t \alpha_t^i(R_h^{k_i} - r_h(X_h^{k_i})) \geq - \sqrt{2 \sum_{i=1}^t (\alpha_t^i)^2 \log(2HK^2/ \delta)} \geq - \sqrt{\frac{4 H \log(2HK^2/ \delta)}{t}}$$
and that
$$\sum_{i=1}^t \alpha_t^i \rbonus{h}{k_i}(B_h^{k_i}) = \sum_{i=1}^t \alpha_t^i \sqrt{\frac{4H\log(4HK/\delta)}{i}} \geq \sqrt{\frac{4H \log(2HK^2/ \delta)}{t}}.$$

Similarly for term $(d)$ via \cref{lemma:ada_ql_trans_concentration_bound} we have that
$$\sum_{i=1}^t \alpha_t^i \left( V_{h+1}^\star(x_{h+1}^{k_i}) - \E_{Y \sim T_h(\cdot \mid X_h^{k_i}, A_h^{k_i})}[V_{h+1}^\star(Y)] \right)  \geq - H\sqrt{2 \sum_{i=1}^t (\alpha_t^i)^2 \log(4HK/\pfail)} \geq - \sqrt{\frac{4H^3 \log(2HK^2/ \pfail)}{t}}$$
and that
$$\sum_{i=1}^t \alpha_t^i \tbonus{h}{k_i}(B_h^{k_i}) = \sum_{i=1}^t \alpha_t^i \sqrt{\frac{4H^3 \log(2HK^2/ \pfail)}{i}} \geq \sqrt{\frac{4H^3 \log(2HK^2/ \pfail)}{t}}.$$

For term $(c)$ we have that this is trivially lower bounded by zero via the induction hypothesis applied to $\Vhat{h+1}{k}(x)$.

For term $(e)$ recall that via Assumption~\ref{assumption:Lipschitz_mf} that $|Q_h^\star(X_h^{k_i}, A_h^{k_i}) - Q_h^\star(x,a)| \leq 2 L_V \diam{B_h^{k_i}}$ as $(x,a) \in B \subset B_h^{k_i}$.  Thus we have that \begin{align*}
    \sum_{i=1}^t \alpha_t^i (Q_h^\star(X_h^{k_i}, A_h^{k_i}) - Q_h^\star(x,a) + \bias(B_h^{k_i}) & \geq \sum_{i=1}^t \alpha_t^i(-2 L_V \diam{B_h^{k_i}} + 2 L_V \diam{B_h^{k_i}}) \geq 0.
\end{align*}

Lastly, we show that $\Vhat{h}{k}(x) \geq V_h^\star(x)$.  By the Bellman Equations (\ref{eqn:bellman_equation}) we know that $$V_{h}^\star(x) = \max_{a \in \A} Q_h^\star(x,a) = Q_h^\star(x, \pi_h^\star(x)).$$
If $\Vhat{h}{k}(x) = H$ then the inequality trivially follows as $V_h^\star(x) \leq H$.  Otherwise, we have that $\Vhat{h}{k}(x) = \max_{B \in \relevant_h^k(x)} \Qhat{h}{k}(B)$.  Let $B^\star$ be the ball with smallest radius in $\P_h^k$ such that $(x, \pi_h^\star(x)) \in B^\star$.  Such a ball exists as $\P_h^k$ covers $\S \times \A$ via \cref{lemma:partition}.  From the previous discussion we know that $\Qhat{h}{k}(B^\star) \geq Q_h^\star(x, \pi_h^\star(x)).$  Hence we have that $$\Vhat{h}{k}(x) \geq \max_{B \in \relevant_h^k(x)} \Qhat{h}{k}(B) \geq \Qhat{h}{k}(B^\star) \geq Q_h^\star(x, \pi_h^\star(x)) = V_h^\star(x).$$

Next we show the result for \AdaMB.  For term $(a)$ we notice that via \cref{lemma:reward_concentration_bound} that $\rbar{h}{k}(B) - r_h(x,a) \geq - \rbonus{h}{k}(B) - 4 L_r \diam{B}$ and so term $(a)$ is lower bounded by $-4 L_r \diam{B}$.

For term $(b)$ we use the fact that $\Vhat{h+1}{k}$ and $V_{h+1}^\star$ are $L_V$ Lipschitz and so we have that:
\begin{align*}
    & \E_{Y \sim \Tbar{h}{k}(\cdot \mid B)}[\Vhat{h+1}{k-1}(Y)] - \E_{Y \sim T_h(\cdot \mid B)}[V_{h+1}^\star(Y)] + \tbonus{h}{k}(B) \\
    & \geq \E_{Y \sim \Tbar{h}{k}(\cdot \mid B)}[V_{h+1}^\star(Y)] - \E_{Y \sim T_h(\cdot \mid B)}[V_{h+1}^\star(Y)] + \tbonus{h}{k}(B) \text{ by I.H.} \\
    & \geq - \tbonus{h}{k}(B) - L_V (5 L_T + 4) \diam{B} + \tbonus{h}{k}(B) \geq - L_V (5 L_T + 4) \diam{B}.
\end{align*}

Thus we have by the definition of $\bias(B)$: $$\Qhat{h}{k}(B) - Q_h^\star(x,a) \geq - 4 L_r \diam{B} - L_V (5 L_T + 4) \diam{B} + \bias(B) = 0.$$

For any $A \in \S(\P_h^k)$ and any $x \in A$, if $\Vtilde{h}{k}(A) = \Vtilde{h}{k-1}(A)$ then optimism clearly holds by the induction hypothesis, and otherwise
\begin{align*}
    \Vtilde{h}{k}(A) & = \max_{B \in \P_h^k: \S(B) \supseteq A} \Qhat{h}{k}(B) \\
    & \geq \Qhat{h}{k}(B^\star) ~~\text{ for } (x, \pi_h^\star(x)) \in B^\star \\
    & \geq Q_h^\star(x, \pi_h^\star(x)) = V_h^\star(x).
\end{align*}
For $x \in A \in \S(\P_h^k)$, and for the ball $B^\star \in \P_h^k$ that satisfies $(x, \pi_h^\star(x)) \in B^\star$, it must be that $\S(B^\star) \supseteq A$ because of the construction of the induced partition $\S(\P_h^k)$ via \cref{eq:induced_state_partition_def}.

And lastly we have that for any $x \in \S$, %$\Vhat{h}{k}(\cdot)$,
\begin{align*}
    \Vhat{h}{k}(x) & = \Vtilde{h}{k}(A) + L_V d_\S(x, \tilde{x}(A)) \quad\text{ for some ball } A \in \S(\P_h^k) \\
    & \geq V_h^\star(\tilde{x}(A)) + L_V d_\S(x, \tilde{x}(A)) \quad\text{ by optimism of }\Vtilde{h}{k}\\
    & \geq V_h^\star(x) \quad\text{ by the Lipschitz property of } V_h^\star.
\end{align*}

\noindent \textbf{Upper Bound}: We start with \AdaQL.  For the first term $(a)$ we trivially upper bound by $\Ind{t=0}H$.  For term $(b)$ we use the same concentration to get
\begin{align*}
    \sum_{i=1}^t \alpha_t^i(R_h^{k_i} - r_h(X_h^{k_i}, A_h^{k_i}) + \rbonus{h}{k_i}(B_h^{k_i}) & \leq \sqrt{\frac{4H \log(2HK^2/ \pfail)}{t}} + \sum_{i=1}^t \alpha_t^i \sqrt{\frac{4H \log(2HK^2/ \pfail)}{i}} \\
    & \leq \sqrt{\frac{4H \log(2HK^2/ \pfail)}{t}} + \sqrt{\frac{8H \log(2HK^2/ \pfail)}{t}} \\
    & = \sqrt{\frac{12H \log(2HK^2/ \pfail)}{t}}.
\end{align*}

For term $(c)$ we repeat the same process to get
\begin{align*}
    & \sum_{i=1}^t \alpha_t^i \left( V_{h+1}^\star(x_{h+1}^{k_i}) - \E_{Y \sim T_h(\cdot \mid X_h^{k_i}, A_h^{k_i})}[V_{h+1}^\star(Y)] + \tbonus{h}{k_i}(B_h^{k_i}) \right)  \\
    & \leq \sqrt{\frac{4H^3 \log(2HK^2/ \pfail)}{t}} + \sum_{i=1}^t \alpha_t^i \sqrt{\frac{4H^3 \log(2HK^2/ \pfail)}{i}} \\
    & \leq \sqrt{\frac{12H^3 \log(2HK^2/ \pfail)}{t}}
\end{align*}

For term $(e)$ we have that
\begin{align*}
    \sum_{i=1}^t \alpha_t^i \left( Q_h^\star(X_h^{k_i}, A_h^{k_i}) - Q_h^\star(x,a) + \bias(B_h^{k_i}) \right) & \leq \sum_{i=1}^t 4 \alpha_t^i L_V \diam{B_h^{k_i}} \\
    & \leq 16 L_V \diam{B} \text{ by \cref{lem:sum_bias_term_ball}}
\end{align*}

Combining the parts together we have:
\begin{align*}
    \Qhat{h}{k}(B) - Q_h^\star(x,a) & \leq \Ind{t=0}H + \sqrt{\frac{12H \log(2HK^2/ \pfail)}{t}} + \sqrt{\frac{12H^3 \log(2HK^2/ \pfail)}{t}} \\
    & + 16 L_V \diam{B} + \sum_{i=1}^t \alpha_t^i (\Vhat{h+1}{k_i} - V_{h+1}^\star)(X_{h+1}^{k_i}) \\
    & \leq \Ind{t = 0}H + 16 L_V \diam{B} + \conf_h^k(B) + \sum_{i=1}^t \alpha_t^i (\Vhat{h+1}{k_i} - V_{h+1}^\star)(X_{h+1}^{k_i}) \\
    & = 8 \bias(B) + \conf_h^k(B) + f_{h+1}^k.
\end{align*}
where $\conf_h^k(B) = 2\sqrt{\frac{12H^3 \log(2HK^2/ \pfail)}{t}}$.

Next for \AdaMB.  For term $(a)$ we note that the difference can be upper bound by $2 \rbonus{h}{k}(B) + 4 L_r \diam{B}$ via \cref{lemma:ada_mb_reward_confidence}.  For term $(b)$ we use \cref{lemma:transition_confidence} and have:
\begin{align*}
    & \E_{Y \sim \Tbar{h}{k}(\cdot \mid B)}[\Vhat{h+1}{k}(Y)] - \E_{Y \sim T_h(\cdot \mid x,a)}[V_{h+1}^\star(Y)] \\
    & =  \E_{Y \sim \Tbar{h}{k}(\cdot \mid B)}[\Vhat{h+1}{k}(Y)] - \E_{Y \sim T_h(\cdot \mid x,a)}[\Vhat{h+1}{k}(Y)] + \E_{Y \sim T_h(\cdot \mid x,a)}[\Vhat{h+1}{k}(Y)] - \E_{Y \sim T_h(\cdot \mid x,a)}[V_{h+1}^\star(Y)] \\
    & \leq \tbonus{h}{k}(B) + L_V (5 L_T + 4) \diam{B} + \E_{Y \sim T_h(\cdot \mid x,a)}[\Vhat{h+1}{k}(Y)] - \E_{Y \sim T_h(\cdot \mid x,a)}[V_{h+1}^\star(Y)].
\end{align*}

Combining this we have that
\begin{align*}
    & \Qhat{h}{k}(B) - Q_h^\star(x,a) \\
    & \leq 2 \rbonus{h}{k}(B) + 2 \tbonus{h}{k}(B) + 2 \bias(B) + \E_{Y \sim T_h(\cdot \mid x,a)}[\Vhat{h+1}{k}(Y)] - \E_{Y \sim T_h(\cdot \mid x,a)}[V_{h+1}^\star(Y)] \\
    % & \leq \frac{\confcons}{n_h^k(B)^\alpha \diam{B}^\beta} + (4L_r + L_V(5L_T + 6)) \diam{B} + E_{Y \sim T_h(\cdot \mid x,a)}[\Vhat{h+1}{k}(Y)] - \E_{Y \sim T_h(\cdot \mid x,a)}[V_{h+1}^\star(Y)] \\
    & \leq \conf_h^k(B) + (4L_r + L_V(5L_T + 6)) \diam{B} + E_{Y \sim T_h(\cdot \mid x,a)}[\Vhat{h+1}{k}(Y)] - \E_{Y \sim T_h(\cdot \mid x,a)}[V_{h+1}^\star(Y)] \\
    & = \conf_h^k(B) + \bias(B) + f_{h+1}^k
\end{align*}
by setting $\conf_h^k(B)$ as the dominating term of $\rbonus{h}{k}(B) + \tbonus{h}{k}(B)$.

\noindent \textbf{Last Property}: Lastly we note that when $x = X_h^k$ then for \AdaMB:
\begin{align*}
    \Vhat{h}{k-1}(X_h^k) & = \min_{A \in \S(\P_h^k)} \Vtilde{h}{k-1}(A) + L_V d_\S(x, \tilde{x}(A)) \leq \Vtilde{h}{k-1}(\S(B_h^k)) + L_V \diam{B_h^k} \\
    & = \max_{B \in \P_h^k : \S(B) \supseteq \S(B_h^k)} \Qhat{h}{k-1}(B) + L_V \diam{B_h^k} = \Qhat{h}{k-1}(B_h^k) + L_V \diam{B_h^k}. 
\end{align*}

Similarly in \AdaQL, $$\Vhat{h}{k-1}(X_h^k) = \max_{B \in \relevant_h^k(X_h^k)} \Qhat{h}{k-1}(B) = \Qhat{h}{k-1}(B_h^k)$$ by the greedy selection rule. \Halmos
\end{rproof}

\section{Upper Bound via Clipping}
\label{sec:clipping}

We first start with a slight detour into \emph{clipping}, a concept first introduced in \cite{simchowitz2019} for developing instance dependent regret bounds in the tabular setting.  This was later extended to \AdaQL in \cite{cao2020provably}, where here we extend the result to other nonparametric algorithms using adaptive discretization and corrected some minor methodological errors.

Recall the quantity $\gap_h(x,a) = V_h^\star(x) - Q_h^\star(x,a)$ from \cref{sec:zooming_dim}.  Notice that for any optimal $a^\star \in \pi_h^\star(x)$ then we have that $\gap_h(x, a^\star) = 0$.  This notion corresponds to the notion of \emph{gap} in the simpler multi-armed bandit settings, whereas here the definition is taken with respect to $Q_h^\star$ instead of the immediate reward.
As the algorithm maintains data collected over regions, we first define the gap of a region $B$ via $\gap_h(B) = \min_{(x,a) \in B} \gap_h(x,a)$.  Moreover, we define the clipping operator $\clip{\mu \mid \nu} = \mu \Ind{\mu \leq \nu}$.

We start with the following technical lemma:

\begin{lemma}
\label{lem:prop_clip}
Define $\clip{\mu \mid \nu} = \mu \Ind{\mu \geq \nu}$.  Then if $\mu_1 \leq \mu_2$ and $\nu_1 \geq \nu_2$ we have that $\clip{\mu_1 \mid \nu_1} \leq \clip{\mu_2 \mid \nu_2}.$
\end{lemma}

\begin{rproof}{\cref{lem:prop_clip}}
First note that whenever $\mu_1 \geq \nu_1$ then we must have that $\mu_2 \geq \nu_2$.  Hence via simple algebra we find $\mu_1 \Ind{\mu_1 \geq \nu_1} \leq \mu_2 \Ind{\mu_2 \geq \nu_2}. \Halmos$
\end{rproof}

Next we show that the clipping operator can be used by the algorithms for estimates which are optimistic with respect to the gap.  
\begin{lemma}[Clipping to the Gap]
\label{lem:general_clip}
Suppose that $\gap_h(B) \leq \phi \leq \mu_1 + \mu_2$ for any quantities $\phi, \mu_1$, and $\mu_2$.  Then we have that $$\phi \leq \clip{\mu_1 \mid \frac{\gap_h(B)}{H+1}} + \left(1 + \frac{1}{H} \right) \mu_2.$$
\end{lemma}
\begin{rproof}{\cref{lem:general_clip}}

First suppose that $\mu_1 \geq \gap_h(B) / (H+1)$, then the result trivially follows.  Otherwise we have that
\begin{align*}
    \gap_h(B) \leq \mu_1 + \mu_2 \leq \frac{\gap_h(B)}{H+1} + \mu_2.
\end{align*}
Taking this inequality and rearranging yields that $$\gap_h(B) \leq \left(\frac{H}{H-1}\right) \mu_2.$$  Plugging this back into the earlier expression gives that
\begin{align*}
\phi & \leq \mu_1 + \mu_2 \leq \frac{\gap_h(B)}{H+1} + \mu_2 \leq \frac{1}{H+1}\left(\frac{H+1}{H}\right) \mu_2 + \mu_2 = \left(1 + \frac{1}{H}\right) \mu_2. \Halmos
\end{align*}
\end{rproof}

This lemma highlights the approach of clipping based arguments.  Recall that via optimism we have that $V_h^\star(x) - Q_h^\star(x,a) \leq \Vhat{h}{k-1}(x) - Q_h^\star(x,a)$.  However, based on the update rules in the algorithm we are able to decompose this sum via
$$\gap_h(x,a) \leq \Vhat{h}{k-1}(x) - Q_h^\star(x,a) \lesssim \underbrace{\conf_h^{k-1}(x,a) + \conslip \bias(x,a)}_{(\mu_1)} + \underbrace{f_{h+1}^{k-1}}_{(\mu_2)} $$
where $\conf_h^{k-1}(x,a) + \bias(x,a)$ denotes our confidence in the estimate for this point, and $f_{h+1}^{k-1}$ denotes the downstream error terms.  We will clip these terms separately using \cref{lem:general_clip}.  Indeed, using this we are able to show the following.

\begin{lemma}
\label{lem:apply_clip}
With probability at least $1 - 2 \delta$ for both \AdaQL and \AdaMB we have that:
\begin{align*}
\Vhat{h}{k-1}(X_h^k) - Q_h^\star(X_h^k, A_h^k) & \leq \clip{\Ind{t = 0}H\Ind{\AdaQL} + \conslip \bias(B_h^k) + \conf_h^{k-1}(B_h^k) \mid \frac{\gap_h(B_h^k)}{H+1}} \\
& \quad + \left(1 + \frac{1}{H}\right) f_{h+1}^{k-1}
\end{align*}
where $\conslip$ is an absolute constant depending on the Lipschitz constants, and
\begin{align*}
    f_{h+1}^{k-1} & = \sum_{i=1}^t \alpha_t^i(\Vhat{h+1}{k_i} - V_{h+1}^\star)(X_{h+1}^{k_i}) \hfill \text{ for \AdaQL} \\
    f_{h+1}^{k-1} & = \E_{Y \sim T_h(\cdot \mid X_h^k,A_h^k)}[\Vhat{h+1}{k-1}(Y)] - \E_{Y \sim T_h(\cdot \mid X_h^k,A_h^k)}[V_{h+1}^\star(Y)] \hfill \text{ for \AdaMB}
\end{align*}
\end{lemma}
\begin{rproof}{\cref{lem:apply_clip}}

The result follows for \AdaQL and \AdaMB by noting that
\begin{align*}
    \gap_h(B_h^k) & \leq \gap_h(X_h^k, A_h^k) = V_h^\star(X_h^k) - Q_h^\star(X_h^k, A_h^k) \\
    & \leq \Vhat{h}{k-1}(X_h^k) - Q_h^\star(X_h^k, A_h^k) \\
    & \leq L_V \diam{B_h^k} \Ind{\AdaMB} + \Qhat{h}{k-1}(B_h^k) - Q_h^\star(X_h^k, A_h^k)
\end{align*}
and then using the upper bounds established in \cref{lemma:optimism}. \Halmos
\end{rproof}

\section{Regret Decomposition}
\label{sec:regret_decomp}

We are now ready to get the final regret decomposition, which bounds the regret of the algorithm by a function of the size of the partition and the sum of the bonus terms used in constructing the high probability estimates.  Throughout the proof we use the constant $\conslip$ to track absolute constants depending on the Lipschitz constants appearing in front of the $\diam{B}$ terms, which is potentially different from line to line. Start by setting $\Delta_h^k = \Vhat{h}{k-1}(X_h^k) - V_h^{\pi^k}(X_h^k)$.  With this definition we can upper bound the regret as
\begin{align*}
    \regret(K) & = \sum_{k=1}^K V_1^\star(X_1^k) - V_1^{\pi^k}(X_1^k) \leq \sum_{k=1}^K \Vhat{1}{k-1}(X_1^k) - V_1^{\pi^k}(X_1^k) = \sum_{k=1}^K \Delta_1^k.
\end{align*}

However, for both algorithms we also have using \cref{lem:apply_clip}
\begin{align*}
    \Delta_h^k & = \Vhat{h}{k-1}(X_h^k) - V_h^{\pi^k}(X_h^k) \\
    & = \Vhat{h}{k-1}(X_h^k) - Q_h^\star(X_h^k, A_h^k) + Q_h^\star(X_h^k, A_h^K) - V_h^{\pi^k}(X_h^k) \\
    & \leq \clip{\Ind{t = 0}H + \conslip \bias(B_h^k) + \conf_h^{k-1}(B_h^k) \mid \frac{\gap_h(B_h^k)}{H+1}} \\
    & + \left(1 + \frac{1}{H}\right)f_{h+1}^{k-1}  + Q_h^\star(X_h^k, A_h^k) - V_h^{\pi^k}(X_h^k).
\end{align*}

The next goal is to take the leftover terms on the second line and show that their summation leads to a term scaling like $\Delta_{h+1}^{k}$ (i.e., the errors at the next step $h$) and a martingale difference sequence absolutely bounded by $2H$.

\begin{lemma}
\label{lemma:martingale_diff_result}
For both \AdaMB and \AdaQL we have that
\begin{align*}
    \sum_{k=1}^K \left(1 + \frac{1}{H}\right) f_{h+1}^{k-1} + Q_h^\star(X_h^k, A_h^k) - V_h^{\pi^k}(X_h^k) \leq \left(1 + \frac{1}{H}\right)^2 \sum_{k} \Delta_{h+1}^{k} + \xi_{h+1}^{k}
\end{align*}
where $\xi_{h+1}^{k}$ is a martingale difference sequence bounded absolutely by $2H$.
\end{lemma}
\begin{rproof}{\cref{lemma:martingale_diff_result}}

\noindent \textbf{\AdaMB}: By substituting the expression for $f_{h+1}^{k-1}$ associated to \AdaMB,
\begin{align*}
    & \left(1 + \frac{1}{H}\right) \sum_{k} f_{h+1}^{k-1} + Q_h^\star(X_h^k, A_h^k) - V_h^{\pi^k}(X_h^k) \\
    & = \left(1 + \frac{1}{H}\right) \sum_{k=1}^K \E_{Y \sim T_h(\cdot \mid X_h^k, A_h^k)}[\Vhat{h+1}{k-1}(Y)] - \E_{Y \sim T_h(\cdot \mid X_h^k, A_h^k)}[V_{h+1}^\star(Y)] + \E_{Y \sim T_h(\cdot \mid X_h^k, A_h^k)}[V_{h+1}^\star(Y) - V_{h+1}^{\pi^k}(Y)] \\
    & = \left(1 + \frac{1}{H} \right) \sum_{k=1}^K \E_{Y \sim T_h(\cdot \mid X_h^k, A_h^k)}[\Vhat{h+1}{k-1}(Y)] - \E_{Y \sim T_h(\cdot \mid X_h^k, A_h^k)}[V_{h+1}^{\pi^k}(Y)] \\
    & \leq \left(1 + \frac{1}{H} \right)^2 \sum_{k=1}^K \Delta_{h+1}^{k} + \xi_{h+1}^{k}
\end{align*}
where $\xi_{h+1}^{k} = \E_{Y \sim T_h(\cdot \mid X_h^k, A_h^k)}[\Vhat{h+1}{k-1}(Y)] - \E_{Y \sim T_h(\cdot \mid X_h^k, A_h^k)}[V_{h+1}^{\pi^k}(Y)] - (\Vhat{h+1}{k-1}(X_{h+1}^k) - V_{h+1}^{\pi^k}(X_{h+1}^k)$.  This is trivially bounded by $2H$ due to boundedness of the value function, and forms a martingale difference sequence as $X_{h+1}^{k} \sim T_h(\cdot \mid X_h^k, A_h^k)$.

\noindent \textbf{\AdaQL}: By substituting the expression for $f_{h+1}^{k-1}$ associated to \AdaQL,
\begin{align*}
    & \left(1 + \frac{1}{H}\right) \sum_{k=1}^K \sum_{i=1}^{n_h^{k-1}(B_h^k)} \alpha_{n_h^{k-1}(B_h^k)}^i (\Vhat{h+1}{k_i} - V_{h+1}^\star)(X_{h+1}^{k_i(B_h^k)}) + Q_h^\star(X_h^k, A_h^k) - Q_h^{\pi^k}(X_h^k, A_h^k) \\
    & \leq \left(1 + \frac{1}{H}\right)^2 \sum_{k=1}^K (\Vhat{h+1}{k-1} - V_{h+1}^\star)(X_{h+1}^{k}) + Q_h^\star(X_h^k, A_h^k) - Q_h^{\pi^k}(X_h^k, A_h^k) \\
    & = \left(1 + \frac{1}{H}\right)^2 \sum_{k=1}^K \Delta_{h+1}^{k} + \xi_{h+1}^{k}
\end{align*}
where $\xi_{h+1}^{k} = \E_{Y \sim T_h(\cdot \mid X_h^k, A_h^k)}[V_{h+1}^\star(Y)] - \E_{Y \sim T_h(\cdot \mid X_h^k, A_h^k)}[V_{h+1}^{\pi^k}(Y)] - (V_{h+1}^\star(X_{h+1}^{k}) - V_{h+1}^{\pi^k}(X_{h+1}^k)$. \Halmos
\end{rproof}

Combining this we can upper bound the regret via \cref{lemma:regret_decomp}.
\begin{lemma}
\label{lemma:regret_decomp}
With probability at least $1 - 2 \delta$ for both $\AdaMB$ and $\AdaQL$ we have that $$\regret(K) \leq 6 \sum_{h=1}^H \sum_{k=1}^K \clip{\Ind{n_h^{k-1}(B_h^k) = 0}H \Ind{\AdaQL} + \conslip \diam{B_h^k} + \conf_h^{k-1}(B_h^k) \mid \frac{\gap_h(B_h^k)}{H+1}} + \xi_{h+1}^{k}.$$
\end{lemma}
\begin{rproof}{\cref{lemma:regret_decomp}}
Using the definition of regret and $\Delta_h^k$ we have that
\begin{align*}
    & \regret(K) = \sum_{k=1}^K V_1^\star(X_1^k) - V_1^{\pi^k}(X_1^k) \leq \sum_{k=1}^K \Vhat{1}{k-1}(X_1^k) - V_1^{\pi^k}(X_1^k)  = \sum_{k=1}^K \Delta_1^k \\
    & \leq \sum_{k=1}^K \clip{\Ind{n_1^{k-1}(B_1^k) = 0}H \Ind{\AdaQL} + \conslip \diam{B_1^k} + \conf_1^{k-1}(B_1^k) \mid \frac{\gap_h(B_1^k)}{H+1}} \\
    & + \left(1 + \frac{1}{H}\right)^2 \sum_{k=1}^K \Delta_{2}^{k} + \xi_{2}^{k} \\
    & \leq \sum_{k=1}^K \sum_{h=1}^H \left(1 + \frac{1}{H}\right)^{2(h-1)} \clip{\Ind{n_h^{k-1}(B_h^k) = 0}H \Ind{\AdaQL} + \conslip \diam{B_h^k} + \conf_h^{k-1}(B_h^k) \mid \frac{\gap_h(B_h^k)}{H+1}} \\
    & + \sum_{k=1}^H \sum_{h=1}^H \left(1 + \frac{1}{H}\right)^{2h} \xi_{h+1}^k.
\end{align*}
However, noting that $(1 + 1 / H)^H \leq e$ the result follows. \Halmos
\end{rproof}

We are now ready to derive the final regret guarantees for \AdaMB and \AdaQL.  This guarantee is \emph{metric-specific}, in that the upper bound allows the metric space to admit various levels of coarseness across different levels of the partition.  Recall that we assumed $\conf_h^k(B)$ is of the form $\confcons / n_h^k(B)^\alpha$ depending on the algorithm used (see \cref{subsec:split_rule}).

\begin{theorem}[Final Regret Bound]
\label{thm:regret_bound}
With probability at least $1 - 3 \pfail$ the regret for \AdaMB and \AdaQL can be upper bound as follows:
\begin{align*}
    \regret(K) & \leq 12 \sqrt{2H^3 K \log(2HK^2 / \pfail)} + 6H^2\Ind{\AdaQL} \\
    & + \sum_{h=1}^H 6 \conslip \inf_{r_0 > 0} \left( 4 K r_0  + \frac{ \confcons^{\frac{1}{\alpha}}}{1 - \alpha} \sum_{r \geq r_0} N_r(Z_h^r) \frac{1}{r^{1/\alpha - 1}}\right)
\end{align*}
where $z_h$ is the step-$h$ zooming dimension and $\alpha$ is the scaling on $n_h^k(B)$ in the definition of $\conf_h^k(B)$.  Moreover, taking $r_0 = K^{-1 / (z_h +1/\alpha)}$ we get the final regret is bounded by

$$\regret(K) \lesssim \sqrt{H^3 K \log(HK^2 / \pfail)} + H^2\Ind{\AdaQL} + \sum_{h=1}^H K^{\frac{z_h -1 + 1 / \alpha}{z_h + 1 / \alpha}}$$
where $\lesssim$ omits poly-logarithmic factors and linear dependence on the Lipschitz constants.
\end{theorem}

\begin{rproof}{\cref{thm:regret_bound}}
Recall from \cref{lemma:regret_decomp}: $$\regret(K) \leq 6 \sum_{h=1}^H \sum_{k=1}^K \clip{\underbrace{\Ind{t = 0}H \Ind{\AdaQL}}_{(a)} + \underbrace{\conslip \bias(B_h^k) + \conf_h^{k-1}(B_h^k)}_{(b)} \mid \frac{\gap_h(B_h^k)}{H+1}} + \underbrace{\xi_{h+1}^{k}}_{(c)}.$$

In the following we let $\conslip$ be an arbitrary constant depending linearly with respect to the Lipschitz constant, and potentially changing line by line.  A simple martingale analysis shows that for term $(c)$ (see \cref{lemma:martingale_difference_sequence})
$$\sum_{h=1}^H \sum_{k=1}^K \xi_{h+1}^{k} \leq 2 \sqrt{2H^3 K \log(2HK^2 / \pfail)}$$ with probability at least $1 - \pfail$.

For term $(a)$ we note that $\sum_h \sum_k \Ind{n_h^{k-1}(B) = 0} \leq H$ and so the entire term is bounded by $H^2$.  We thus turn our focus to term $(b)$.  By the splitting property for every ball $B_h^k$ as it was active when selected we know that $\conslip \bias(B_h^k) = \conslip \diam{B_h^k} \leq \conslip \conf_h^{k-1}(B_h^k)$.  Thus we get using \cref{lem:prop_clip}
$$(b) \leq \sum_{h=1}^H \sum_{k=1}^K \clip{\conslip \conf_h^{k-1}(B_h^k) \mid \frac{\gap_h(B_h^k)}{H+1}}.$$

However, consider the inside term based on cases.  By definition of the clip operator we only need to consider when $\conslip \conf_h^{k-1}(B_h^k) \geq \frac{\gap_h(B_h^k)}{H+1}$.  However, this implies that $\gap_h(B_h^k) \leq (H+1)\conslip \conf_h^{k-1}(B_h^k) \leq 4 (H+1) \conslip \diam{B_h^k}$ (using \cref{lemma:partition}).  However, letting $(x_c, a_c)$ denote the center of $B_h^k$ this implies that 
$$\gap_h(x_c, a_c) \leq \gap_h(B_h^k) + 2 L_V \diam{B_h^k} \leq \conslip(H+1)\diam{B_h^k}$$ and so we see that $(x_c, a_c)$ lies in the set $Z_h^r$ for $r = \diam{B_h^k}$.

Thus we can rewrite the clip as
$$\sum_{h} \sum_{k} \conslip \conf_h^{k-1}(B_h^k) \Ind{\text{center}(B_h^k) \in Z_h^{\diam{B_h^k}}}.$$  Consider the summation for a fixed $h$.  We then have
\begin{align*}
    \sum_{k=1}^K \conslip \conf_h^{k-1}(B_h^k) \Ind{\text{center}(B_h^k) \in Z_h^{\diam{B_h^k}}} & = \sum_{r} \sum_{B : \diam{B} = r} \sum_{k : B_h^k = B} \conslip \conf_h^{k-1}(B) \Ind{\text{center}(B) \in Z_h^{r}}.
\end{align*}
We break the summation into when $r \geq r_0$ and $r \leq r_0$.

\noindent \textbf{Case $r \leq r_0$:}  

For the case when $r \leq r_0$ we have:
$$ \sum_{r \leq r_0} \sum_{B: \diam{B} = r} \sum_{k : B_h^k = B} \conslip \conf_h^{k-1}(B) \Ind{\text{center}(B) \in Z_h^{r}} \leq 4 \conslip K r_0$$ since $\conf_h^{k-1}(B_h^k) \leq 4 \diam{B_h^k}$ via \cref{lemma:partition}.

\noindent \textbf{Case $r \geq r_0$:}

For the case when $r \geq r_0$ we have that:

\begin{align*}
    & \conslip \sum_{r \geq r_0} \sum_{B : \diam{B} = r} \sum_{k : B_h^k = B} \conf_h^{k-1}(B) \Ind{\text{center}(B) \in Z_h^r} \\
    & = \conslip \sum_{r \geq r_0} \sum_{B : \diam{B} = r} \Ind{\text{center}(B) \in Z_h^{r}} \sum_{k : B_h^k = B} \frac{\confcons}{n_h^{k-1}(B)^\alpha}
\end{align*}
by definition of $\conf_h^{k-1}(B)$.  Recall that the total number of times a ball $B_h^k$ can be selected by the algorithm can be upper bound via $n_{max}(B) - n_{min}(B).$  In addition, inside of the summation we know $\diam{B} = r$.  Using this and the definition of $n_{max}(B)$ from \cref{lemma:bound_ball} we get
\begin{align*}
    & = \conslip \confcons \sum_{r \geq r_0} \sum_{B : \diam{B} = r} \Ind{\text{center}(B) \in Z_h^{r}}  \sum_{k : B_h^k = B} \frac{1}{n_h^{k-1}(B)^\alpha} \\
    & \leq \conslip \confcons \sum_{r \geq r_0} \sum_{B : \diam{B} = r} \Ind{\text{center}(B) \in Z_h^{r}}  \sum_{i=0}^{n_{max}(B) - n_{min}(B)} \frac{1}{(i + n_{min}(B))^\alpha}
\end{align*}
by using that the maximum number of times a ball $B$ is selected can be upper bound by $n_{max}(B) - n_{min}(B)$ and that $n_h^k(B)$ will start at $n_{min}(B)$.  However, bounding the summation by the integral we are able to get
\begin{align*}
    & \leq \conslip \confcons \sum_{r \geq r_0} \sum_{B : \diam{B} = r} \Ind{\text{center}(B) \in Z_h^{r}}  \int_{x=0}^{n_{max}(B) - n_{min}(B)} \frac{1}{(x + n_{min}(B))^\alpha} \, dx \\
    & = \frac{\conslip \confcons}{1 - \alpha} \sum_{r \geq r_0} \sum_{B : \diam{B} = r} \Ind{\text{center}(B) \in Z_h^{r}}  n_{max}(B)^{1 - \alpha} \\
    & = \frac{\conslip \confcons}{1 - \alpha} \sum_{r \geq r_0} \sum_{B : \diam{B} = r} \Ind{\text{center}(B) \in Z_h^{r}}  \left(\frac{\confcons}{r}\right)^{\frac{1 - \alpha}{\alpha}} \\
    & = \frac{\conslip \confcons^{\frac{1}{\alpha}}}{1 - \alpha} \sum_{r \geq r_0} \sum_{B : \diam{B} = r} \Ind{\text{center}(B) \in Z_h^{r}} \frac{1}{r^{1/\alpha - 1}} \\
    & \leq \frac{\conslip \confcons^{\frac{1}{\alpha}}}{1 - \alpha} \sum_{r \geq r_0} N_r(Z_h^r) \frac{1}{r^{1/\alpha - 1}}
\end{align*}
The final result follows via combining these two parts and taking the infimum over $r_0$.  The simplified version arises by using that $N_r(Z_h^r) \leq C_{zoom} r^{-z_h}$ for the zooming dimension $z_h$ and calculating the geometric series for the particular value of $r_0$. \Halmos
    
\end{rproof}

\section{Bound on Sample and Storage Complexity}
\label{sec:lp_bound}

In this section we finish the proof of \cref{thm:space_storage} outlining the storage and time complexity of \AdaMB and \AdaQL.  Recall that the splitting threshold is defined to be split a ball once we have that $\conf_h^k(B) \leq \diam{B}$.  In \cref{lemma:bound_ball} we define $n_{min}(B)$ and $n_{max}(B)$ as bounds on the number of times a ball $B$ has been sampled.  Thus, we can upper bound $|\{k : B_h^k = B\}|$, the number of times a ball $B$ has been selected by the algorithm, by $n_{max}(B) - n_{min}(B)$.  As these quantities only depend on the diameter of the ball (or equivalently its level in the partition since $\diam{B} = 2^{-\lev{B}}$) we abuse notation and use $\nplus{\ell} = n_{max}(B) - n_{min}(B)$ for the ball $B$ taken at level $\ell$.  More specifically we have that $\nplus{\ell} = \confcons^{1 / \alpha} 2^{\frac{\ell}{\alpha}} = \phi 2^{\gam \ell}$ for $\gam = \frac{1}{\alpha}$ by using the results in \cref{sec:partition} where $\phi$ is a constant.

We first provide a general bound for counts over any partition $\Pkh$.
\begin{lemma}
\label{lem:LPbound}
Consider any partition $\Pkh$ for any $k\in[K], h\in[H]$ induced under \AdaMB or \AdaQL with splitting thresholds $\nplus{\ell}$, and consider any ``penalty'' vector $\{a_{\ell}\}_{\ell\in\NN_0}$ that satisfies $a_{\ell+1} \geq a_{\ell} \geq 0$ and $2a_{\ell+1}/a_{\ell} \leq \nplus{\ell}/\nplus{\ell-1}$ for all $\ell\in\NN_0$. Define $\ell^{\star} = \inf\{\ell \mid 2^{d(\ell - 1)} \nplus{\ell-1} \geq k\}$. Then
\begin{align*}
\sum_{\ell=0}^{\infty}\sum_{B\in \Pkh : \ell(B) = \ell} a_{\ell} \leq 2^{d\ell^{\star}}a_{\ell^{\star}}     \end{align*}
\end{lemma}
\begin{rproof}{\cref{lem:LPbound}}
For $\ell\in\NN_0$, let $x_\ell$ denote the number of active balls at level $\ell$ in $\Pkh$. Then $\sum_{B\in \Pkh:\lev{B}=\ell} a_{\ell} = a_{\ell} x_\ell$. 
Now we claim that under any partition, this sum can be upper bound via the following linear program (LP):
\begin{align*}
  \text{maximize: } & \quad \sum_{\ell=0}^{\infty} a_{\ell}x_{\ell} \\
  \text{subject to: } & \quad  \sum_{\ell}2^{-\ell d} x_{\ell} \leq 1 \; , \\
  & \quad \sum_{\ell}\nplus{\ell-1}2^{-d} x_{\ell} \leq k \; , \\
  & \quad    x_{\ell} \geq 0 \,\forall\,\ell
\end{align*}
The first constraint arises via the Kraft-McMillan inequality for  prefix-free codes (see Chapter 5 in~\cite{cover2012elements}).  Since each node can have at most $D \sim 2^d$ (where $d = d_\S + d_\A$) children by definition of the covering dimension, the partition created can be thought of as constructing a prefix-free code on a $D$-ary tree. 
The second constraint arises by a conservation argument on the number of samples.  Recall that $\nplus{B}$ is the minimum number of samples required before $B$ is split into $2^d$ children, an alternate way to view this is that each ball at level $\ell$ requires a ``sample cost'' of $\nplus{\ell-1}/2^d$ unique samples in order to be created. The sum of this sample cost over all active balls is at most the number of episodes $k$.

Next, via LP duality, we get that the optimal value for this program is upper bounded by $\alpha + \beta$ for any $\alpha$ and $\beta$ such that:
\begin{align*}
    2^{-\ell d} \alpha + n_+(\ell - 1) 2^{-d} \beta & \geq a_\ell \quad \forall \ell \in \NN_0 \\
    \alpha, \beta & \geq 0.
\end{align*}

Recall the definition of $\ell^\star = \inf\{\ell \mid 2^{d(\ell - 1)} \nplus{\ell-1} \geq k\}$ and consider 
\begin{align*}
    \hat\alpha = \frac{2^{d\ell^{\star}}a_{\ell^{\star}}}{2} \quad \hat\beta = \frac{2^{d}a_{\ell^{\star}}}{2\nplus{\ell^{\star}-1}}.
\end{align*}

We claim that this pair satisfies the constraint that $2^{-\ell d} \hat\alpha + \nplus{\ell-1}2^{-d} \hat\beta \geq a_{\ell}$ for any $\ell$, and hence by weak duality we have that
\[\sum_{B\in \Pkh:\lev{B}=\ell} a_{\ell}\leq \hat\alpha+ k \hat\beta \leq 2\hat\alpha = 2^{d\ell^{\star}}a_{\ell^{\star}}.\]

To verify the constraints on $(\hat\alpha,\hat\beta)$ we check it by cases.
First note that for $\ell = \ell^\star$, we have $2^{-\ell^{\star} d} \hat\alpha + \nplus{\ell^{\star}-1}2^{-d}\hat\beta = a_{\ell^\star}$. 

Next, for any $\ell < \ell^{\star}$, note that $2^{-\ell d} \geq 2^{-(\ell^{\star}-1) d} > 2\cdot(2^{-\ell^{\star} d})$, and hence $2^{-\ell d} \hat\alpha \geq 2\cdot(2^{-\ell^{\star} d} \hat\alpha) = a_{\ell^{\star}} \geq a_{\ell}$ by construction of the penalty vector.

Similarly, for any $\ell > \ell^{\star}$, we have by assumption on the costs and $\nplus{\ell}$ that
\begin{align*}
    \frac{\nplus{\ell-1}}{a_{\ell}} \geq \frac{2^{\ell - \ell^\star}\nplus{\ell^{\star}-1}}{a_{\ell^\star}} \geq 2\frac{\nplus{\ell^\star - 1}}{a_{\ell^\star}}.
\end{align*}
Then we get by plugging in our value of $\hat\beta$ that
\begin{align*}
    \nplus{\ell - 1}2^{-d} \hat\beta & = \frac{a_{\ell^\star} \nplus{\ell - 1}}{2\nplus{\ell^\star - 1}}
     \geq a_\ell
\end{align*}This verifies the constraints for all $\ell\in\NN_0$. \Halmos
\end{rproof}
Note also that in the above proof, we actually use the condition $2a_{\ell+1}/a_{\ell} \leq \nplus{\ell}/\nplus{\ell-1}$ for $\ell\geq \ell^{\star}$; we use this more refined version in~\cref{lem:countbound} below.

One immediate corollary of~\cref{lem:LPbound} is a bound on the size of the partition $|\Pkh|$ for any $h,k$.
\begin{corollary}
\label{lem:size_partition}
For any $h$ and $k$ we have that 
\begin{align*}
    |\P_h^k| \leq 4^d\left(\frac{k}{\phi}\right)^{\frac{d}{d+\gam}} \quad \text{ and } \quad \ell^\star \leq \frac{1}{d + \gam} \log_2(k / \phi) + 2.
\end{align*}
\end{corollary}
\begin{rproof}{\cref{lem:size_partition}}
Note that the size of the partition can be upper bounded by the sum where we take $a_\ell = 1$ for every $\ell$.  Clearly this satisfies the requirements of Lemma~\ref{lem:LPbound}.  Moreover, using the definition of $\ell^\star$ we have that $2^{d(\ell^\star - 2)}\nplus{\ell^\star - 2} \leq k$ as otherwise $\ell^\star - 1$ would achieve the infimum.  Taking this equation and plugging in the definition of $\nplus{\ell}$ by the splitting rule yields that 
\begin{align*}
    \ell^\star \leq \frac{1}{d+\gam} \log_2\left(\frac{k}{\phi}\right) + 2.
\end{align*}
Then by plugging this in we get that
\begin{align*}
    |\P_h^k| & \leq 2^{d\ell^\star}
     \leq 2^{\frac{d}{d+\gamma}\log_2(k / \phi) + 2d}
     = 4^d \left(\frac{k}{\phi}\right)^{d/(d+\gam)}. \Halmos
\end{align*}
\end{rproof}
In other words, the worst case partition size is determined by a \emph{uniform} scattering of samples, wherein the entire space is partitioned up to equal granularity (in other words, a uniform $\epsilon$-net). 

More generally, we can use \cref{lem:LPbound} to bound various functions of counts over balls in $\Pkh$.
\begin{corollary}
\label{lem:countbound}
For any $h\in[H]$, consider any sequence of partitions $\Pkh, k\in[K]$ induced under \AdaMB or \AdaQL with splitting thresholds $\nplus{\ell}= \phi2^{\gam\ell}$.
Then, for any $h\in[H]$ we have:

For any $\alpha,\beta\geq 0$ s.t. $\alpha \leq 1$ and $\alpha\gamma-\beta\geq 1$, 
\begin{align*}
\sum_{k=1}^K \frac{2^{\beta\lev{B_h^k}}}{\left(n_h^k(B_h^k)\right)^{\alpha}} = O\left(\phi^{\frac{-(d\alpha+\beta)}{d+\gam}} K^{\frac{d+(1-\alpha)\gam+\beta}{d+\gam}}\right).
\end{align*}
For any $\alpha,\beta\geq 0$ s.t. $\alpha \leq 1$ and $\alpha\gamma - \beta/\ell^{\star}\geq 1$ (where $\ell^{\star} =2+\frac{1}{d+\gam }\log_2\left(\frac{K}{\phi}\right)$),
\begin{align*}
\sum_{k = 1}^K \frac{\lev{B_h^k}^\beta}{\left(n_h^k(B_h^k)\right)^{\alpha}} = O\left(\phi^{\frac{-d\alpha}{d+\gam}} K^{\frac{d+(1-\alpha)\gam}{d+\gam}}\left(\log_2K\right)^{\beta}\right).
\end{align*}
\end{corollary}
\begin{rproof}{\cref{lem:countbound}}
The proof of both the inequalities follows from a direct application of~\cref{lem:LPbound}, after first rewriting the summation over balls in $\Pkh$ as a summation over active balls in $\Pkh[K]$.

\noindent \textbf{First Inequality}: First, observe that we can write
\begin{align*}
    \sum_{k=1}^K \frac{2^{\beta \ell(B_h^k)}}{\left(n_h^k(B_h^k)\right)^\alpha} & = \sum_{\ell \in \NN_0} \sum_{B: \ell(B) = \ell} \sum_{k=1}^K \Ind{B_h^k = B} \frac{2^{\beta \ell(B)}}{\left(n_h^k(B)\right)^\alpha}.
\end{align*}
Now, in order to use~\cref{lem:LPbound}, we first need to rewrite the summation as over \emph{active balls} in the terminal partition $\Pkh[K]$ (i.e., balls which are yet to be split). Expanding the above, we get
\begin{align*}
    \sum_{k=1}^K \frac{2^{\beta \ell(B_h^k)}}{\left(n_h^k(B_h^k)\right)^\alpha} 
%& = \sum_{\ell \in \NN_0} \sum_{B: \ell(B) = \ell} \sum_{k=1}^K \Ind{B_h^k = B} \frac{2^{\beta \ell(B)}}{\left(n_h^k(B)\right)^\alpha} \\
    & = \sum_{\ell \in \NN_0} \sum_{B \in \P_h^K : \ell(B) = \ell} \sum_{B' \supseteq B} 2^{d(\ell(B') - \ell(B))} \sum_{k=1}^K \Ind{B_h^k = B'} \frac{2^{\beta \ell(B')}}{\left(n_h^k(B')\right)^\alpha} \\
    & \leq \sum_{\ell \in \NN_0} \sum_{B \in \P_h^K : \ell(B) = \ell} \sum_{B' \supseteq B} 2^{d(\ell(B') - \ell(B))} 2^{\beta \ell(B')}\sum_{j=1}^{\nplus{\ell(B')}} \frac{1}{j^\alpha} \\
    & \leq \frac{\phi^{1-\alpha}}{1 - \alpha} \sum_{\ell \in \NN_0} \sum_{B \in \P_h^K : \ell(B) = \ell} \sum_{B' \supseteq B}2^{d(\ell(B') - \ell(B))} 2^{\beta \ell(B')} 2^{\gamma \ell(B')(1 - \alpha)}
\end{align*}
where we used the fact that once a ball has been partitioned it is no longer chosen by the algorithm and an integral approximation to the sum of $1 / j^{\alpha}$ for $\alpha \leq 1$.  Next, we plug in the levels to get
\begin{align*}
\sum_{k=1}^K \frac{2^{\beta \ell(B_h^k)}}{\left(n_h^k(B_h^k)\right)^\alpha} &\leq \frac{\phi^{1-\alpha}}{1 - \alpha} \sum_{\ell \in \NN_0} \sum_{B \in \P_h^K : \ell(B) = \ell} \sum_{j=0}^{\ell}2^{d(j - \ell)} 2^{\beta j} 2^{\gamma j(1 - \alpha)} \\
& = \frac{\phi^{1-\alpha}}{1 - \alpha} \sum_{\ell \in \NN_0} \sum_{B \in \P_h^K : \ell(B) = \ell} \frac{1}{2^{d \ell}}\sum_{j=0}^{\ell}2^{j(d+\beta + \gam(1 - \alpha))} \\
& \leq \frac{\phi^{1 - \alpha}}{(2^{d+\beta+\gam(1 - \alpha)} - 1)(1 - \alpha)} \sum_{\ell \in \NN_0} \sum_{B \in \P_h^K : \ell(B) = \ell} \frac{1}{2^{d \ell}} 2^{(\ell + 1)(d+\beta + \gam(1 - \alpha))} \\
%& = \frac{\phi^{1 - \alpha}}{(2^{d+\beta+\gam(1 - \alpha)} - 1)(1 - \alpha)} \sum_{\ell \in \NN_0} \sum_{B \in \P_h^K : \ell(B) = \ell} \frac{1}{2^{d \ell}} 2^{(\ell + 1)(d+\beta + \gam(1 - \alpha))} \\
& \leq \frac{2\phi^{1 - \alpha}}{(1 - \alpha)} \sum_{\ell \in \NN_0} \sum_{B \in \P_h^K : \ell(B) = \ell} 2^{\ell(\beta + \gam(1 - \alpha))}.
\end{align*}
We set $a_\ell = 2^{\ell(\beta + \gamma(1 - \alpha))}$.  Clearly we have that $a_\ell$ are increasing with respect to $\ell$.  Moreover,
\begin{align*}
    \frac{2a_{\ell+1}}{a_\ell} & = \frac{2 \cdot 2^{(\ell + 1)(\beta + \gam(1 - \alpha))}}{2^{(\ell )(\beta + \gam(1 - \alpha))}} = 2^{1 + \beta + \gam(1 - \alpha)}.
\end{align*}
Setting this quantity to be less than $\nplus{\ell}/\nplus{\ell - 1} = 2^{\gam}$ we require that
\begin{align*}
    2^{1 + \beta + \gam(1 - \alpha)} & \leq 2^{\gam} \Leftrightarrow 
    %2^{1 + \beta - \alpha \gam} & \leq 1 \\
    1 + \beta - \alpha \gam \leq 0
\end{align*}
Now we can apply~\cref{lem:LPbound} to get that
\begin{align*}
    \sum_{k=1}^{K} \frac{2^{\beta \ell(B_h^k)}}{\left(n_h^k(B_h^k)\right)^{\alpha}} & \leq \frac{2\phi^{1 - \alpha}}{(1 - \alpha)} 2^{d \ell^\star} a_{\ell^\star} \\
    & = \frac{2^{2(d+\beta+\gam(1 - \alpha))}\phi^{1 - \alpha}}{(1 - \alpha)} \left(\frac{K}{\phi}\right)^{\frac{d+\beta + \gam(1 - \alpha)}{d+\gam}} \\
    & = O\left( \phi^{\frac{-(d\alpha+\beta)}{d+\gam}} K^{\frac{d+(1-\alpha)\gam+\beta}{d+\gam}} \right).
\end{align*}

\noindent \textbf{Second Inequality}: As in the previous part, we can rewrite as the summation we have
\begin{align*}
    \sum_{k=1}^K \frac{\ell(B_h^k)^\beta}{\left(n_h^k(B_h^k)\right)^\alpha} & = \sum_{\ell \in \NN_0} \sum_{B: \ell(B) = \ell} \sum_{k=1}^K \Ind{B_h^k = B} \frac{\ell(B)^\beta}{\left(n_h^k(B)\right)^\alpha}\\
    & \leq \sum_{\ell \in \NN_0} \sum_{B \in \P_h^K : \ell(B) = \ell} \sum_{B' \supseteq B} 2^{d(\ell(B') - \ell(B))} \ell(B')^\beta \sum_{j=1}^{\nplus{\ell(B')}} \frac{1}{j^\alpha} \\
    & \leq \sum_{\ell \in \NN_0} \sum_{B \in \P_h^K : \ell(B) = \ell} \sum_{B' \supseteq B} 2^{d(\ell(B') - \ell(B))} \ell(B')^\beta \frac{\nplus{\ell(B')}^{1-\alpha}}{1 - \alpha} \\
    & = \frac{\phi^{1-\alpha}}{1-\alpha} \sum_{\ell \in \NN_0} \sum_{B \in \P_h^K : \ell(B) = \ell} \sum_{B' \supseteq B} 2^{d(\ell(B') - \ell(B))} \ell(B')^\beta 2^{\ell(B')\gam(1 - \alpha)}.
\end{align*}
As before, we plug in the levels to get
\begin{align*}
\sum_{k=1}^K \frac{\ell(B_h^k)^\beta}{\left(n_h^k(B_h^k)\right)^\alpha} & =    \frac{\phi^{1-\alpha}}{1-\alpha} \sum_{\ell \in \NN_0} \sum_{B \in \P_h^K : \ell(B) = \ell} \sum_{j=0}^\ell 2^{d(j - \ell)} j^\beta 2^{j\gam(1 - \alpha)}\\
& \leq \frac{\phi^{1-\alpha}}{1 - \alpha} \sum_{\ell \in \NN_0} \sum_{B \in \P_h^K : \ell(B) = \ell} \frac{\ell^{\beta}}{2^{d \ell}} \sum_{j=0}^{\ell} 2^{j(d+\gam(1 - \alpha))} \\
%& \leq \frac{\phi^{1-\alpha}}{(2^{d + \gam(1 - \alpha)} - 1)(1 - \alpha)} \sum_{\ell \in \NN_0} \sum_{B \in \P_h^K : \ell(B) = \ell} \frac{\ell^{\beta}}{2^{d \ell}} 2^{(\ell + 1)(d + \gam(1 - \alpha))} \\
& \leq \frac{2\phi^{1-\alpha}}{(1 - \alpha)} \sum_{\ell \in \NN_0} \sum_{B \in \P_h^K : \ell(B) = \ell} \ell^{\beta} 2^{\ell\gam(1 - \alpha)}.
\end{align*}
We take the term $a_\ell = \ell^{\beta} 2^{\ell\gam(1 - \alpha)}$.  Clearly we have that $a_\ell$ are increasing with respect to $\ell$.  Moreover, 
\begin{align*}
    \frac{2a_{\ell + 1}}{a_\ell} & = \left(1 + \frac{1}{\ell}\right)^\beta 2^{1+\gam (1 - \alpha)}.
\end{align*}
We require that this term is less than $\nplus{\ell + 1}/\nplus{\ell} = 2^{\gam}$ for all $\ell\geq \ell^{\star}$ (see note after~\cref{lem:LPbound}). This yields the following sufficient condition (after dividing through by $2^{\gam}$)
\begin{align*}
    \left(1 + \frac{1}{\ell}\right)^\beta 2^{1- \alpha\gamma} \leq 1\frall \ell\geq\ell^{\star}
\end{align*}
or equivalently, $\alpha\gamma - \beta\log_2(1+1/\ell^{\star}) \geq 1$. Finally note that $\log_2(1+x)\leq x/\ln 2 \leq x$ for all $x\in[0,1]$. Thus, we get that a sufficient condition is that $\alpha\gamma - \beta/\ell^{\star} \geq 1$. Assuming this holds, we get by~\cref{lem:LPbound} that
\begin{align*}
    \sum_{k=1}^{K} \frac{ \ell(B_h^k)^{\beta}}{\left(n_h^k(B_h^k)\right)^{\alpha}} 
    & \leq \left(\frac{2\phi^{1 - \alpha}}{(1 - \alpha)}\right) 2^{d \ell^\star} a_{\ell^\star} \\
    % & = \left(\frac{2^{d+\gam(1 - \alpha)}}{(2^{d+\gam(1 - \alpha)} - 1)(1 - \alpha)}\right) \phi^{1 - \alpha} \left(\frac{2^dK}{\phi}\right)^{\frac{d + \gamma(1-\alpha)}{d + \gam}}(\ell^{\star})^{\beta} \\
    & = \left(\frac{2\phi^{1 - \alpha}}{1 - \alpha}\right) 4^{d+\gam(1 - \alpha)} \left(\frac{K}{\phi}\right)^{\frac{d+\gam(1 - \alpha)}{d+\gam}} \left(\frac{\log_2(K / \phi)}{d+\gam} + 2\right)^\beta \\
    & = O\left( \phi^{\frac{-d\alpha}{d+\gam}} K^{\frac{d+(1-\alpha)\gam}{d+\gam}} \left(\log_2K\right)^{\beta}\right). \Halmos
\end{align*}
\end{rproof}

Using these results we are finally ready to show \cref{thm:space_storage}, which we rewrite here for convenience.

\begin{theorem}
\label{thm:space_storage_here}
The storage complexity for \AdaMB and \AdaQL can be upper bound via 
	\begin{align*}
	\textsc{Space}(K) & \lesssim \begin{cases}
	HK \quad \AdaMB: d_\S > 2 \\
	HK^{\frac{d+d_\S}{d+d_\S+2}} \quad \AdaMB: d_\S \leq 2 \\
	HK^{\frac{d}{d+2}} \quad \AdaQL 
	\end{cases}
	\end{align*}
The time complexity for \AdaMB and \AdaQL can be upper bound as
	\begin{align*}
	\textsc{Time}(K) & \lesssim \begin{cases}
	HK^{\frac{d+2d_\S}{d+d_\S}} \quad \AdaMB: d_\S > 2 \\
	HK^{\frac{d+d_\S}{d+d_\S+2}} \quad \AdaMB: d_\S \leq 2 \\
	HK\log_d(K) \quad \AdaQL 
	\end{cases}
	\end{align*}
\end{theorem}
\begin{rproof}{\cref{thm:space_storage_here}}

\noindent \textbf{Storage Requirements \AdaMB}:  The algorithms maintain a partition $\P_h^k$ of $\S \times \A$ for every $h$, and the respective induced partition $\S(\P_h^k)$ whose size is trivially upper bounded by the size of the total partition.  Each element $B \in \P_h^k$ maintains four estimates.  The first three $n_h^k(B)$, $\rbar{h}{k}(B)$, and $\Qhat{h}{k}(B)$, are linear with respect to the size of the partition.  The last one, $\Tbar{h}{k}(\cdot \mid B)$ has size $|\dyad{\lev{B}}| \lesssim O(2^{d_\S \lev{B}}$).  Moreover, the algorithm also maintains estimate $\Vtilde{h}{k}(\cdot)$ over $\S(\P_h^k)$.  Clearly we have that the worst-case storage complexity arises from maintaining estimates of the transition kernels over each region in $\P_h^k$.  Thus we have that the total storage requirement of the algorithm is bounded above by \[ \sum_{h=1}^H \sum_{B \in \P_h^K} 2^{d_\S\lev{B}}.\]
Utilizing \cref{lem:LPbound} with $a_\ell = 2^{d_\S \ell}$ we find that the sum is bounded above by
\begin{align*}
     \sum_{h=1}^H \sum_{B \in \P_h^K} 2^{d_\S\lev{B}} & \leq \sum_{h=1}^H 2^{d \ell^\star} a_{\ell^\star} \\
     & \lesssim HK^{\frac{d+d_\S}{d+\gam}}.
\end{align*}
Plugging in the definition of $\gam$ from the splitting rule yields the results in \cref{tab:comparison_of_bounds}.

\medskip

\noindent \textbf{Run-Time \AdaMB}: We assume that the oracle access discussed occurs in constant time.  The inner loop of \cref{alg:brief} has four main steps.  Finding the set of relevant balls for a given state can be implemented in $\log_d(|\P_h^k|)$ time by traversing through the tree structure.  Updating the estimates and refining the partition occur in constant time by assumption on the oracle.  Lastly we need to update the estimates for $\Qhat{h}{k}$ and $\Vhat{h}{k}$.  Since the update only needs to happen for a constant number of regions (as only one ball is selected per step episode pair) the dominating term arises from computing the expectation over $\Tbar{h}{k}(\cdot \mid B_h^k)$.  Noting that the support of the distribution is $|\dyad{\lev{B_h^k}}| = 2^{d_\S \lev{B_h^k}}$ the total run-time of the algorithm is upper bounded by \[\sum_{h=1}^H \sum_{k=1}^K 2^{d_\S \lev{B_h^k}}.\]  Rewriting the sum we have
\begin{align*}
    \sum_{h=1}^H \sum_{k=1}^K 2^{d_\S \lev{B_h^k}} & \leq \sum_{h=1}^H \sum_{\ell \in \mathbb{N}} \sum_{B \in \P_h^K : \lev{B} = \ell} 2^{d_\S \ell} \sum_{k \in [K] : B_h^k = B} 1 \\
    & \lesssim \sum_{h=1}^H \sum_{\ell \in \mathbb{N}} \sum_{B \in \P_h^K : \lev{B} = \ell} 2^{d_\S \ell} n_+(B) \\
    & \lesssim \sum_{h=1}^H \sum_{\ell \in \mathbb{N}} \sum_{B \in \P_h^K : \lev{B} = \ell} 2^{d_\S \ell} \phi 2^{\gam \ell}.
\end{align*}
Utilizing \cref{lem:LPbound} with $a_\ell = 2^{(d_\S + \gam) \ell}$ we find that the sum is bounded above by
$H \phi 2^{d \ell^\star} a_{\ell^\star} \lesssim HK^{1 + \frac{d_\S}{d+\gam}}.$  Plugging in $\gam$ from the splitting rule yields the result in \cref{tab:comparison_of_bounds}.

\medskip

\noindent \textbf{Storage Requirements \AdaQL}:  In contrast to \AdaMB, \AdaQL only maintains a partition $\P_h^k$ over $\S \times \A$ for every $h$.  Each element $B \in \P_h^k$ maintains two estimates, $n_h^k(B)$ and $\Qhat{h}{k}(B)$.  Thus we can upper bound the worst-case storage complexity via the cumulative worst-case size of the partition.  Using \cref{lem:size_partition} with $\gamma = 2$ we get the storage complexity is upper bound by $HK^{d / (d+2)}$.

\medskip

\noindent \textbf{Run-Time \AdaQL}: Similar to \AdaMB we assume the oracle access occurs in constant time.  The inner loop of the algorithm has three main steps.  First, finding the set of relevant balls for a given state is implemented in $\log_d(|\P_h^k|)$ by traversing through the tree structure.  Updating the estimates and refining the partition occur in constant time by the oracle assumption.  Thus we get that the total run-time of the algorithm is dominated by finding the set of relevant balls at every iteration.  This is upper bound by $HK \log_d(K)$ where we use the fact that at most one region is activated at every iteration so the partition size can be upper bound by $K$. \Halmos
\end{rproof}

\section{Parametric Wasserstein Concentration Bounds for \AdaMB}
\label{app:parametric}

We can show improved performance of \AdaMB under additional restrictions on the transition kernel.  The results only allow \AdaMB to match the regret and space requirements of \AdaQL.  This highlights a downfall in current analysis of model-based algorithms more broadly where we require uniform concentration of the transition distribution with respect to the Wasserstein metric.  In essence, this is providing concentration inequalities for the transitions evaluated over \emph{arbitrary} Lipschitz functions, where in reality we only require concentration when taking expectation of the Lipschitz Value function for the optimal policy.

\subsection{Deterministic Transitions}

Suppose that $T_h(\cdot \mid x,a) = \delta_{f_h(x,a)}$ for some deterministic function $f_h(x,a) : \S \times \A \rightarrow \S$.  Note that in order for $T_h(\cdot \mid x,a)$ to be Wasserstein Lipschitz continuous we need that
\[
    d_w(T_h(\cdot \mid x,a), T_h(\cdot \mid x', a')) = \D_\S(f_h(x,a), f_h(x',a')) \leq L_T \D((x,a), (x',a')),
\]
i.e. that $f_h(x,a)$ is Lipschitz continuous when viewed as a function from $\S \times \A$ to $\S$.  

In this scenario, \AdaMB can forego maintaining estimates of the full transition kernel.  In its place, for each region $B \in \P_h^k$ we maintain the previously observed transition from that region, i.e. we set $\Tbar{h}{k}(\cdot \mid B) = \delta_{X_{h+1}^{k_t}}$ where $k_t$ is the previous time that $B$ or its ancestors were selected by the algorithm.  This allows the algorithm to forego maintaining an estimate of $\Tbar{h}{k}(\cdot \mid B)$ over $\S(\P_{\lev{B}})$, and can directly update $\Qhat{h}{k}(B)$ and $\Vhat{h}{k}(B)$ (replacing \cref{eq:q_update}) as follows:
\begin{align*}
    \Qhat{h}{k}(B) & = \rbar{h}{k}(B) + \rbonus{h}{k}(B) + \Vhat{h+1}{k}(X_{h+1}^{k_t}) + \tbonus{h}{k}(B) + \bias(B) \\
    \Vhat{h}{k}(x) & = \max_{B \in \relevant_h^k(x)} \Qhat{h}{k}(B).
\end{align*}
Moreover, the uncertainty in $\Tbar{h}{k}(\cdot \mid B)$ is now zero as the uncertainty can be bound by the diameter of the ball.  We can set $\tbonus{h}{k}(B) = 0$ and adjust $\bias(B)$ to include an additional $L_T \diam{B}$ term to show the following:
\begin{theorem}
\label{thm:ada_mb_determinsitic}
Suppose that $T_h(\cdot \mid x,a) = \delta_{f_h(x,a)}$ for some deterministic and Lipschitz continuous function $f_h(x,a) : \S \times \A \rightarrow \S$.  Then we have that for \AdaMB:
\begin{align*}
    R(K) & \lesssim H^{3/2} \sqrt{K} + L \sum_{k=1}^K K^{\frac{z_h + 1}{z_h+2}} \\
    \textsc{Space}(K) & \lesssim HK^{d/(d+2)} \\
    \textsc{Time}(K) & \lesssim HK \log_d(K)
\end{align*}
where $L = 1 + L_r + L_V + L_V L_T$ and $\lesssim$ omits poly-logairithmic factors of $\frac{1}{\delta}, H, K, d,$ and any universal constants.
\end{theorem}
\begin{rproof}{\cref{thm:ada_mb_determinsitic}}
For the regret bound we note that the only modification required is changing the proof of \cref{lemma:transition_confidence}.  For any $(h,k) \in [H] \times [K]$ and $B \in \P_h^k$ with $(x,a)$ in B we have that:
\begin{align*}
    d_w(\Tbar{h}{k}(\cdot \mid B), T_h(\cdot \mid x,a)) & = d_w(\delta_{f_h(X_h^{k_t}, A_h^{k_t})}, \delta_{f_h(x,a)}) \\
    & = \D_\S(f_h(X_h^{k_t}, A_h^{k_t}), f_h(x,a)) \leq L_T \D((X_h^{k_t}, A_h^{k_t}), (x,a)) \leq L_T \diam{B}.
\end{align*}
Following through the rest of the regret bound noticing the dominating bonus term scales as $\conf_h^k(B) \lesssim \frac{1}{\sqrt{n_h^k}}$, taking $\alpha = \frac{1}{2}$, the regret bound follows.

For the bound on the space requirements, we note that the algorithm only maintains a partition $\P_h^k$ of $\S \times \A$ for every $h$ and no induced partition.  Each element $B \in \P_h^k$ has four estimates, $(n_h^k(B), \Qhat{h}{k}(b), \rbar{h}{k}(B), \Tbar{h}{k}(\cdot \mid B))$.  All of these are linear with respect to the size of the partition.  Thus we can upper bound the space requirement by the total size of the partitions, which using \cref{lem:size_partition} with $\alpha = \frac{1}{2}$ (or $\gamma = 2$) we get $|\P_h^k| \lesssim K^{d/(d+2)}$.

For the bound on the time requirement, finding the set of relevant balls for a given state is implemented in $\log_d(|\P_h^k|)$ time by traversing through the tree structure.  Updating estimates and refining the partition occur in constant time by the oracle assumptions.  Thus the total run time is dominated by finding set of relevant balls.  Thus is upper bound by $HK \log_d(K)$ where we use that at most one region is activated at every iteration so the partition size is bounded by $K$.  \Halmos
\end{rproof}

\subsection{Normally Distributed Transitions}

Now suppose that $T_h(\cdot \mid x,a) = N(\mu(x,a), I_{d_\S})$ for some unknown mean $\mu(x,a)$ and $\S \subset \mathbb{R}^{d_\S}, \A \subset \mathbb{R}^{d_\A}$ under the $\ell_2$ metric.  From properties on the Wasserstein metric for normal distributions, \cref{assumption:Lipschitz_mb} implies that $\mu(x,a)$ is Lipschitz continuous for every $(x,a) \in \S \times \A$.  In this case, \AdaMB can again forego maintaining explicit estimates of the full transition kernel.  Instead, for each region $B \in \P_h^k$ we maintain an estimate $\bar{\mu}(B)$ for $\mu(x,a)$ by $\frac{1}{t} \sum_{i=1}^t X_{h+1}^{k_i}$ where $k_1 < k_2 < \ldots < k_t$ are the times $B$ or its ancestors have been selected by the algorithm.  This can be used in the Bellman equations for updating estimates (replacing \cref{eq:q_update}) by:
\begin{align*}
    \Qhat{h}{k}(B) & = \rbar{h}{k}(B) + \rbonus{h}{k}(B) + \E_{Y \sim N(\bar{\mu}(B), I_{d_\S})}[\Vhat{h+1}{k}(Y)] + \tbonus{h}{k}(B) + \bias(B) \\
    \Vhat{h}{k}(x) & = \max_{B \in \relevant_h^k(x)} \Qhat{h}{k}(B).
\end{align*}
Moreover, we can set $\tbonus{h}{k}(B) = \frac{\sqrt{d_\S(1 + \sqrt{6\log(1 / \pfail) / d_\S})}}{\sqrt{n_h^k(B)}}$ and adjust $\bias(B)$ to include an additional $4 L_T \diam{B}$ term.  Unfortunately the time complexity of this algorithm is more difficult to analyze as evaluating  $\E_{Y \sim N(\bar{\mu}(B), I_{d_\S})}[\Vhat{h+1}{k}(Y)]$ requires taking expectations of $d_\S$-dimensional Gaussian distribution.  This can be doing with Monte-Carlo simulation, or using the fact that $\Vhat{}{}$ is piecewise constant so the expectation can be evaluated by getting mass of regions where $\Vhat{}{}$ is constant.  However, we can show the following:
\begin{theorem}
\label{thm:ada_mb_normal}
Suppose that $T_h(\cdot \mid x,a) = N(\mu(x,a), I_{d_\S})$ for some known covariance $I_{d_\S}$ and unknown mean $\mu(x,a)$.  Then we have that for \AdaMB:
\begin{align*}
    R(K) & \lesssim H^{3/2} \sqrt{K} + L \sum_{k=1}^K K^{\frac{z_h + 1}{z_h+2}} \\
    \textsc{Space}(K) & \lesssim HK^{d/(d+2)}.
\end{align*}
where $L = 1 + L_r + L_V + L_V L_T$, and $\lesssim$ omits polylogairithmic factors of $\frac{1}{\delta}, H, K, d,$ linear factors of $d$, and any universal constants.
\end{theorem}
\begin{rproof}{\cref{thm:ada_mb_normal}}
For the regret bound we note that the only modification required is changing the proof of \cref{lemma:transition_confidence}.  However, by construction for any $(h,k) \in [H] \times [K]$ and $B \in \P_h^k$ with $(x,a)$ in B we have that:
\begin{align*}
    d_w(\Tbar{h}{k}(\cdot \mid B), T_h(\cdot \mid x,a)) & = d_w(N(\bar{\mu}(B), I_{d_\S}), N(\mu(x,a), I_{d_\S})) \\
    & = \norm{\bar{\mu}(B) - \mu(x,a)}_2 \\
    & \leq \norm{\bar{\mu}(B) - \frac{1}{t} \sum_{i=1}^t \mu(X_h^{k_i}, A_h^{k_i})}_2 + \norm{\frac{1}{t}\sum_{i=1}^t \mu(X_h^{k_i}, A_h^{k_i}) - \mu(x,a)}_2.
\end{align*}
For the first term we use standard concentration inequalities for random vectors in high dimensions which says for $Y \sim N(0, I_{d_\S}^2 I_n)$ that (see \cite{vershynin2018high}):
\[
    \Pr(\norm{Y}_2^2 \geq (1 + \epsilon) \Exp{\norm{Y}_2^2}) \leq e^{-\epsilon^2 n / 6}.
\]
However, note that $\bar{\mu}(B) - \frac{1}{t} \sum_{i=1}^t \mu(X_h^{k_i}, A_h^{k_i}) = \frac{1}{t} \sum_{i=1}^t X_{h+1}^{k_i} - \mu(X_h^{k_i}, A_h^{k_i})$.  Using that $X_{h+1}^{k_i} \sim T_h(\cdot \mid X_h^{k_i}, A_h^{k_i}) = N(\mu(X_h^{k_i}, A_h^{k_i}), I_{d_\S})$ this vector is $N(0, \frac{1}{t}I_{d_\S})$.  Hence, taking $\epsilon = \sqrt{6 \log(1 / \pfail) / d_\S}$ and using that $\Exp{\norm{\bar{\mu}(B) - \frac{1}{t} \sum_{i=1}^t \mu(X_h^{k_i}, A_h^{k_i})}_2^2} = d_\S / t$ we get that with probability $1 - \pfail$:
\begin{align*}
    \norm{\bar{\mu}(B) - \frac{1}{t} \sum_{i=1}^t \mu(X_h^{k_i}, A_h^{k_i})}_2 \leq \sqrt{\frac{d_\S(1 + \sqrt{6 \log(1 / \pfail) / d_\S})}{n_h^k(B)}} = \tbonus{h}{k}(B).
\end{align*}

For the second term we use the assumption that the mean function is Lipschitz to get
\begin{align*}
    \norm{\frac{1}{t}\sum_{i=1}^t \mu(X_h^{k_i}, A_h^{k_i}) - \mu(x,a)} & \leq \frac{1}{t} \sum_{i=1}^t \norm{\mu(X_h^{k_i}, A_h^{k_i}) - \mu(x,a)} \\
    & \leq \frac{1}{t} \sum_{i=1}^t L_T \diam{B_h^{k_i}} \\
    & \leq 4L_T \diam{B} \text{ by \cref{lem:sum_bias_term_ball}.}
\end{align*}
Following through the rest of the regret bound noticing the dominating bonus term scales as $\conf_h^k(B) \lesssim \frac{1}{\sqrt{n_h^k}}$ and taking $\alpha = \frac{1}{2}$ the regret bound follows.

For the bound on the space requirements, we note that now the algorithm only maintains a partition $\P_h^k$ of $\S \times \A$ for every $h$ and no induced partition.  Each element $B \in \P_h^k$ has four estimates, $(n_h^k(B), \Qhat{h}{k}(b), \rhat{h}{k}(B), \bar{\mu}(B)$.  All of these are linear with respect to the size of the partition.  Thus we can upper bound the space requirement by the total size of the partitions, which using \cref{lem:size_partition} with $\alpha = \frac{1}{2}$ (or $\gamma = 2$) we get $|\P_h^k| \lesssim K^{d/(d+2)}$.  \Halmos
\end{rproof}

\section{Algorithm Implementation and Experiment Details}
\label{app:full_algo}

In this section we outline the oracle assumptions for implementing the algorithm, heuristic approaches to discretization, and the computing infrastructure used in the experiments.

\subsection{Oracle Assumptions}

There are three main oracle assumptions needed to execute the algorithm.    
\begin{assumption}
The algorithm is given as input a hierarchical partition of $\S \times \A$ (see \cref{def:partition}).
\end{assumption}

This routine is easy in many metrics of interest (e.g. the Euclidean norm or any equivalent norms in $\mathbb{R}^d$) by splitting each of the principle dimensions in half.
\begin{assumption}
Computing $\S(B)$ for any $B \in \S \times \A$ takes $O(1)$ time.
\end{assumption}

As our algorithm is maintaining a dyadic partition of the space, this subroutine is simple to implement as each ball $B$ is of the form $\S(B) \times \S(A)$ and so the algorithm can store the two components separately.  Lastly, we require computing $\relevant_h^k(x)$.

\begin{assumption}
Computing $\relevant_h^k(x)$ for any $x \in \S$ takes $O(\log_d(|\P_h^k|))$ time.
\end{assumption}

By storing the partition as a tree, this subroutine can be implementing by traversing down the tree and checking membership at each step.  See the Github repositories at \cite{orsuite,rlberry} for implementing these methods.  Storing the discretization as a hash function would allow some of these access steps to be implemented in $O(1)$ time with the downside being that splitting a region has a larger computational requirement.

% \subsection{Algorithm Implementation, Storage Complexity, and Run-Time}
% \label{app:implementation_run_time}

\subsection{Computing Infrastructure}
\label{app:experiments}

\noindent \textbf{Experiment Setup}: Each experiment was run with $50$ iterations where the relevant plots are taking the mean over those iterations.  We picked a fixed horizon of $H = 5$ and ran each algorithm for $K = 2000$ episodes.  As each algorithm uses bonus terms of the form $c / \sqrt{t}$ where $t$ is the number of times a related region has been visited, we tuned the constant $c$ separately for each algorithm (for $c \in [.001, 10]$) and plot the results on the performance of the algorithm for the best constant $c$.  For the uniform discretization algorithms we also tuned the discretization parameter $\epsilon$ (controlling the size of the fixed partition) for values of $\epsilon \in [0.001, .5]$ and report the performance for the best constant $\epsilon$.  Lastly, the hyperparameters for the \textsc{SB PPO} algorithm were tuned using recommended hyperparameters from \cite{schulman2017proximal}.

\medskip

\noindent \textbf{Fixed Discretization UCBVI}: We benchmarked our adaptive algorithm against a fixed-discretization model-based algorithm with full and one-step planning.  We implemented UCBVI from \cite{azar2017minimax} using a fixed discretization of the state-action space.  The algorithm takes as input a parameter $\epsilon$ and constructs an $\epsilon$-covering of $\S$ and $\A$ respectively.  It then runs the original UCBVI algorithm over this discrete set of states and actions.  The only difference is that when visiting a state $x$, as feedback to the algorithm, the agent snaps the point to its closest neighbour in the covering.

UCBVI has a regret bound of $H^{3/2}\sqrt{SAK} + H^4 S^2 A$ where $S$ and $A$ are the size of the state and action spaces.  Replacing these quantities with the size of the covering, we obtain $$H^{3/2} \sqrt{\epsilon^{-d_\S} \epsilon^{-d_\A} K} + H^4 \epsilon^{-2d_\S} \epsilon^{-d_\A}.$$
A rough calculation also shows that the discretization error is proportional to $H L K \epsilon$.  Tuning $\epsilon$ so as to balance these terms, we find that the regret of the algorithm can be upper bounded by

\[ L H^2 K^{2d/(2d + 1)}.\]

\medskip

\noindent \textbf{Computing Infrastructure}: The experiments were conducted on a personal computer with an AMD Ryzen 5 3600 6-Core 3.60 GHz processor and 16.0GB of RAM. No GPUs were harmed in these experiments.
\section{Technical Details}
\label{app:technical_details}

\subsection{Concentration Inequalities}

\begin{lemma}[Lemma 4.1 in \cite{jin_2018}]
\label{lemma:lr}
Recall $\alpha_t^i = a_i \prod_{j=i+1}^t (1 - \alpha_j)$. The learning rates $\alpha_t$ satisfy:
\begin{enumerate}
	\item $\sum_{i=1}^t \alpha_t^i = 1$, $\max_{i \in [t]} \alpha_t^i \leq \frac{2H}{t}$ and $\sum_{i=1}^t (\alpha_t^i)^2 \leq \frac{2H}{t}$ for every $t \geq 1$
	%\item For $t = 0$ we have $\Ind{t = 0} = 1$ and $\sum_{i=1}^t \alpha_t^i = 0$
	\item $\frac{1}{\sqrt{t}} \leq \sum_{i=1}^t \frac{\alpha_t^i}{\sqrt{t}} \leq \frac{2}{\sqrt{t}}$ for every $t \geq 1$
	\item $\sum_{t=i}^\infty \alpha_t^i = 1 + \frac{1}{H}$ for every $i \geq 1$.
\end{enumerate}
\end{lemma}

\begin{lemma}
\label{lemma:martingale_difference_sequence}
Suppose that $\xi_{h+1}^k$ is a martingale difference sequence with $|\xi_{h+1}^k| \leq 2H$ almost surely.  Then for any $\pfail \in (0, 1)$ with probability at least $1 - \pfail$ we have that $$\left| \sum_{h=1}^H \sum_{k=1}^K \xi_{h+1}^k \right|  \leq 2\sqrt{2H^3 K \log(2HK^2/\pfail)}.$$
\end{lemma}

\begin{rproof}{\cref{lemma:martingale_difference_sequence}}
Using that $|\xi_h^k| \leq 2H$ we have via Aszuma-Hoeffdings,

\begin{align*}
& \Pr\left(\left| \sum_{h=1}^H \sum_{k=1}^K \xi_{h+1}^k \right| > \sqrt{8 H^3 K \log(2HK^2/\pfail)}\right) \leq 2\exp\left( -\frac{8 H^3 K \log(2HK^2/\pfail)}{2HK(2H)^2}\right) \\
& = 2 \exp  \left( -\frac{8 H^3 K \log(2HK^2/\pfail)}{8H^3 K} \right)  = 2 \frac{\pfail}{2KH^2} \leq \pfail. \Halmos
\end{align*}
\end{rproof}

\begin{lemma}[Concentration of Transition for \AdaMB, \cref{lemma:transition_confidence}]
\label{app:transition_confidence}
With probability at least $1 - \pfail$ we have that for any $h,k \in [H] \times [K]$ and ball $B \in \P_h^k$ with $(x,a) \in B$ that 
\begin{align*}
    d_W(\Tbar{h}{k}(\cdot \mid B), T_h(\cdot \mid x,a)) \leq \frac{1}{L_V}(\tbonus{h}{k}(B) + (5 L_T + 4) \diam{B}).
\end{align*}
\end{lemma}
\begin{rproof}{\cref{app:transition_confidence}}

Let $h,k \in [H] \times [K]$ and $B \in \P_h^k$ be fixed and $(x,a) \in B$ be arbitrary.  
We use a combination of Proposition 10 and 20 from \cite{weed2019sharp} when $d_\S > 2$, and Proposition 20 from \cite{weed2019sharp} and Proposition 1.1 from \cite{boissard2014mean} when $d_\S \leq 2$.  Let $P_0 = T_h(\cdot \mid x_0,a_0)$ where $(x_0, a_0) = (\tilde{x}(B), \tilde{a}(B))$ is the center of the ball $B$.  Our goal is to come up with concentration between the one-Wasserstein metric of $\Tbar{h}{k}(\cdot \mid B)$ and $T_h(\cdot \mid x,a)$.  We break the proof down into four stages, where we show concentration between the one-Wasserstein distance of various measures. As defined, $\Tbar{h}{k}(\cdot \mid B)$ is a distribution over $\dyad{\ell(B)}$, the uniform discretization of $\S$ over balls with diameter $2^{-\ell(B)}$.  However, we will view $\Tbar{h}{k}(\cdot \mid B)$ as a discrete distribution over $\S$, where 
\begin{align*}
    \Tbar{h}{k}(x \mid B) = \Tbar{h}{k}(A \mid B) \quad \text{ if } x = \tilde{x}(A).
\end{align*}

\noindent \textbf{Step One}:  Let $\tilde{T}_h^k(\cdot \mid B)$ be the true empirical distribution of all samples collected from $B'$ for any $B'$ which is an ancestor of $B$.  Denoting $t = n_h^k(B)$ and letting $k_1 < \ldots < k_t \leq k$ be the episodes in which $B$ or its ancestors were selected by the algorithm we have that
\begin{align}
    \tilde{T}_h^k(\cdot \mid B) = \frac{1}{t} \sum_{i=1}^t \pfail_{X_{h+1}^{k_i}}. \label{eq:Wass_step1}
\end{align} 

Let $A_{h+1}^{k_i}$ denote the region in $\dyad{\ell(B_h^{k_i})}$ containing the point $X_{h+1}^{k_i}$. Recall $\Tbar{h}{k}(\cdot \mid B)$ is the distribution over $\dyad{\ell(B)}$ defined according to:
\begin{align*}
    \Tbar{h}{k}(\cdot \mid B) = \frac{1}{t} \sum_{i=1}^t \sum_{A \in \dyad{\lev{B}} : A \subseteq A_{h+1}^{k_i}} 2^{-d_\S(\lev{B_h^{k_i}} - \lev{B})} \delta_A.
\end{align*}

We can verify that $\sum_{A \in \dyad{\lev{B}} : A \subseteq A_{h+1}^{k_i}} 2^{-d_\S(\lev{B_h^{k_i}} - \lev{B})} = 1$ as the number of regions in $\dyad{\lev{B}}$ which contain any region in $\dyad{\lev{B'}}$ is exactly $2^{d_\S(\ell(B') - \ell(B))}$. Furthermore $X_{h+1}^{k_i}$ and $\tilde{x}(A)$ are both contained in $A_{h+1}^{k_i}$ so that
$\D_\S(X_{h+1}^{k_i}, \tilde{x}(A)) \leq \D_S(A_{h+1}^{k_i}) \leq \diam{B_h^{k_i}}$, where we use the definition of $\dyad{\ell(B_h^{k_i})}$ for the last inequality.
Using these observations, it follows that
\begin{align*}
    d_W(\Tbar{h}{k}(\cdot \mid B), \tilde{T}_h^k(\cdot \mid B)) & \leq \frac{1}{t} \sum_{i=1}^t \sum_{A \in \dyad{\lev{B}} : A \subseteq A_{h+1}^{k_i}} 2^{-d_\S(\lev{B_h^{k_i}} - \lev{B})} \D_\S(X_{h+1}^{k_i}, \tilde{x}(A)) \\
    & \leq \frac{1}{t} \sum_{i=1}^t \sum_{A \in \dyad{\lev{B}} : A \subseteq A_{h+1}^{k_i}} 2^{-d_\S(\lev{B_h^{k_i}} - \lev{B})} \D_\S(A_{h+1}^{k_i}) \\
    & \leq \frac{1}{t} \sum_{i=1}^t \sum_{A \in \dyad{\lev{B}} : A \subseteq A_{h+1}^{k_i}} 2^{-d_\S(\lev{B_h^{k_i}} - \lev{B})} \diam{B_h^{k_i}} \\
    & \leq \frac{1}{t} \sum_{i=1}^t \diam{B_h^{k_i}} \leq 4 \diam{B} \text{ by  \cref{lem:sum_bias_term_ball}}. 
\end{align*}

\noindent \textbf{Step Two}:
%Next we want to bound 1-Wasserstein distance between $\tilde{T}_h^k(\cdot \mid B)$ and $T_h(\cdot \mid x_0,a_0)$. We do so in two parts, where first we 
Next we bound the difference between $\tilde{T}_h^k(\cdot \mid B)$ and $\tilde{T}_h(\cdot \mid x_0, a_0)$ where $\tilde{T}_h(\cdot \mid x_0, a_0)$ is a ``ghost empirical distribution'' of samples whose marginal distribution is $T_h(\cdot \mid x_0, a_0)$.
%, before then bounding the distance between $\tilde{T}_h(\cdot \mid x_0, a_0)$ and $T_h(\cdot \mid x_0, a_0)$.  
As $T_h(\cdot)$ is Lipschitz, for every $x,a,x_0,a_0$,
\[ d_W(T_h(\cdot \mid x, a), T_h(\cdot \mid x_0, a_0)) \leq L_T \D((x,a), (x_0,a_0)). \]
Using the coupling definition of the Wasserstein metric, there exists a family of distributions $\xi(\cdot, \cdot | x,a,x_0,a_0)$ on $\S^2$ parameterized by $x,a,x_0,a_0$ such that 
\[\E_{(Z,Y) \sim \xi(\cdot, \cdot | x,a,x_0,a_0)}[\D_\S(Z, Y)]= d_W(T_h(\cdot \mid x, a), T_h(\cdot \mid x_0, a_0)) \leq L_T \D((x,a), (x_0,a_0)),\]
whose marginals are $T_h(z \mid x,a)$ and $T_h(y \mid x_0, a_0)$.  For $(Z,Y) \sim \xi(\cdot, \cdot | x,a,x_0,a_0)$, let $\xi'(\cdot | z, x,a,x_0,a_0)$ denote the conditional distribution of $Y$ given $Z$, such that 
\begin{align}
\xi(z, y | x,a,x_0,a_0) = T_h(z \mid x, a) \xi'(y | z, x,a,x_0,a_0). \label{eq:Wass_step2_c}
\end{align}

%\[ d_W(T_h(\cdot \mid X_{h}^{k_i}, A_{h}^{k_i}), T_h(\cdot \mid x_0, a_0)) = \E_{\xi_i}[\D_\S(X_{h+1}^{k_i}, Y_i)]. \]

For the sequence of samples $\{(X_{h}^{k_i}, A_{h}^{k_i}, X_{h+1}^{k_i})\}_{i\in[t]}$ realized by our algorithm, consider a ``ghost sample'' $Y_1, \ldots, Y_t$ such that $Y_i \sim \xi'(\cdot | X_{h+1}^{k_i},X_{h}^{k_i}, A_{h}^{k_i},x_0,a_0)$ for $i \in [t]$. Let $\tilde{T}_h(\cdot \mid x_0, a_0)$ denote the empirical distribution of these samples such that 
\[\tilde{T}_h(\cdot \mid x_0, a_0) = \frac{1}{t} \sum_{i=1}^t \pfail_{Y_i}
~~\text{ and recall by definition }~~ \tilde{T}_h^k(\cdot \mid B) = \frac{1}{t} \sum_{i=1}^t \pfail_{X_{h+1}^{k_i}}.\]

Using the definition of the Wasserstein distance we have that
\begin{align}
    d_W(\tilde{T}_h^k(\cdot \mid B), \tilde{T}_h(\cdot \mid x_0,a_0)) & \leq \frac{1}{t} \sum_{i=1}^t \D_\S( X_{h+1}^{k_i}, Y_i). \label{eq:Wass_step2_a}
\end{align}
We will use Azuma-Hoeffding's to provide a high probability bound on this term by its expectation.  For any $\tau \leq K$ define the quantity 
\begin{align*}
    Z_\tau = \sum_{i=1}^\tau \D_\S(X_{h+1}^{k_i}, Y_i) - \Exp{\D_\S(X_{h+1}^{k_i}, Y_i)}.
\end{align*}
Let $\F_i$ be the filtration containing $\F_{k_i + 1} \cup \{Y_j\}_{j \leq i}$.  It follows that $Z_\tau$ is a martingale with respect to $\F_\tau$.  The process is adapted to the filtration by construction, has finite first moment, and we have that
\begin{align*}
    \Exp{Z_\tau \mid \F_{\tau - 1}} & = Z_{\tau - 1} + \Exp{\D_\S(X_{h+1}^{\tau}, Y_\tau)} - \Exp{\D_\S(X_{h+1}^{\tau}, Y_\tau)} = Z_{\tau - 1}.
\end{align*}
Moreover, we also have the differences are bounded by
\begin{align*}
    \left|Z_\tau - Z_{\tau - 1}\right| & = \left|\D_\S(X_{h+1}^{k_\tau}, Y_\tau) - \Exp{\D_\S(X_{h+1}^{k_\tau}, Y_\tau)}\right|
    \leq 2
\end{align*}
since by assumption $\D_\S(\S) \leq 1$. By Azuma-Hoeffding's inequality, with probability at least $1 - \frac{\pfail}{HK^2}$,
\begin{align}
    \frac{1}{\tau} \sum_{i=1}^\tau \D_\S(Y_i, X_{h+1}^{k_i}) & \leq \Exp{\frac{1}{\tau} \sum_{i=1}^\tau \D_\S(Y_i, X_{h+1}^{k_i})} + \sqrt{\frac{8\log(2HK^2/\pfail)}{\tau}}. \label{eq:Wass_step2_b}
\end{align}
Moreover, by construction of the ghost samples we have that
\begin{align*}
    \frac{1}{\tau} \sum_{i=1}^\tau \Exp{\D_\S(Y_i, X_{h+1}^{k_i})} & = \frac{1}{\tau} \sum_{i=1}^\tau \Exp{d_W(T_h(\cdot \mid X_{h}^{k_i}, A_h^{k_i}), T_h(\cdot \mid x_0, a_0))} \\
    & \leq \frac{1}{\tau} \sum_{i=1}^\tau L_T \D(B_{h+1}^{k_i}) \leq 4 L_T \diam{B} \text{ for } \tau = n_h^k(B).
\end{align*}
since $x_0, a_0$ is in the ball $B$ which is contained in the ball $B_{h+1}^{k_i}$. By plugging this into \cref{eq:Wass_step2_b}, taking a union bound over the number of steps $H$, the number of episodes $K$, the number of potential stopping times $K$, and combining it with \cref{eq:Wass_step2_a} and using the construction of $t$, it follows that with probability at least $1 - \pfail$, for all $h,k,B$
%\begin{align*}
%    d_W(\tilde{T}_h^k(\cdot \mid B), \tilde{T}_h(\cdot \mid x_0, a_0)) & \leq \frac{1}{t} \sum_{i=1}^t \D_\S(Y_i, X_{h+1}^{k_i}) \\
    %& \leq \Exp{\frac{1}{t} \sum_{i=1}^t \D_\S(Y_i, X_{h+1}^{k_i})} + \sqrt{\frac{8\log(HK^2/\pfail)}{t}} \\
   % & \leq \frac{1}{t} \sum_{i=1}^t L_T \D(B_h^{k_i}) + \sqrt{\frac{8\log(HK^2/\pfail)}{t}}.
%\end{align*}
%Rewriting the previous quantity by definition of $t$ we have that
\begin{align*}
    d_W(\tilde{T}_h^k(\cdot \mid B), \tilde{T}_h(\cdot \mid x_0, a_0)) & \leq 4 L_T \diam{B} + \sqrt{\frac{8\log(2HK^2/\pfail)}{n_h^k(B)}}.
\end{align*}
Note that we do not need to union bound over all balls $B \in \P_h^k$ as the estimate of only one ball is changed per (step, episode) pair, i.e. $\Tbar{h}{k}(B)$ is changed for only a single ball $B = B_h^k$ per episode. For all balls not selected, it inherits the concentration of the good event from the previous episode because its estimate does not change. Furthermore, even if ball $B$ is ``split'' in episode $k$, all of its children inherit the value of the parent ball, and thus also inherits the good event, so we still only need to consider the update for $B_h^k$ itself.

\noindent \textbf{Step Three}: Next we bound $d_W(\tilde{T}_h(\cdot \mid x_0, a_0), T_h(\cdot \mid x_0, a_0))$. Recall $\F_i$ is the filtration containing $\F_{k_i + 1} \cup \{Y_j\}_{j \leq i}$. Note that the joint distribution over $\{(X_{h}^{k_i}, A_{h}^{k_i}, X_{h+1}^{k_i},Y_i)\}_{i\in[t]}$ is given by
\[G_t(\{(X_{h}^{k_i}, A_{h}^{k_i}, X_{h+1}^{k_i},Y_i)\}_{i\in[t]}) = \prod_{i=1}^t (P(X_{h}^{k_i}, A_h^{k_i} ~|~ \F_{i-1}) T_h(X_{h+1}^{k_i} | X_{h}^{k_i}, A_h^{k_i}) \xi'(Y_i | X_{h+1}^{k_i}, X_{h}^{k_i}, A_h^{k_i}, x_0, a_0),\]
where $P(X_{h}^{k_i}, A_h^{k_i} ~|~ \F_{i-1})$ is given by the dynamics of the MDP along with the policy that the algorithm plays. Then we have 
\begin{align*}
&\int_{\S \times \A \times \S} G_t(\{(X_{h}^{k_i}, X_{h}^{k_i}, X_{h+1}^{k_i},Y_i)\}_{i\in[t]}) dX_{h}^{k_t} dA_{h}^{k_t} dX_{h+1}^{k_t} \\
&= G_{t-1}(\{(X_{h}^{k_i}, X_{h}^{k_i}, X_{h+1}^{k_i},Y_i)\}_{i\in[t-1]}) \\
&\qquad \cdot \int_{\S \times \A} P(X_{h}^{k_t}, A_h^{k_t} ~|~ \F_{k_{t-1}}) \left(\int_{\S} \xi(X_{h+1}^{k_i}, Y_i | X_{h}^{k_i}, A_h^{k_i}, x_0, a_0) dX_{h+1}^{k_t} \right) dX_{h}^{k_t} dA_{h}^{k_t} \\
&= G_{t-1}(\{(X_{h}^{k_i}, X_{h}^{k_i}, X_{h+1}^{k_i},Y_i)\}_{i\in[t-1]}) T_h(Y_i | x_0, a_0) \int_{\S \times \A} P(X_{h}^{k_t}, A_h^{k_t} ~|~ \F_{k_{t-1}}) dX_{h}^{k_t} dA_{h}^{k_t} \\
&= G_{t-1}(\{(X_{h}^{k_i}, X_{h}^{k_i}, X_{h+1}^{k_i},Y_i)\}_{i\in[t-1]}) T_h(Y_i | x_0, a_0).
\end{align*}
By repeating this calculation, we can verify that the marginal distribution of $Y_1 \dots Y_t$ is $\prod_{i \in [t]} T_h(Y_i | x_0, a_0)$. 

First consider when $d_\S > 2$.  Following Proposition 10 and 20 from \cite{weed2019sharp} we have that with probability at least $1 - \pfail / HK^2$ for some universal constant $c$,
\begin{align*}
    d_W(\tilde{T}_h(\cdot \mid x_0, a_0), T_h(\cdot \mid x_0, a_0)) & \leq \Exp{d_W(\tilde{T}_h(\cdot \mid x_0, a_0), T_h(\cdot \mid x_0, a_0)} + \sqrt{\frac{\log(2 HK^2 / \pfail)}{n_h^k(B)}} \\
    & \leq c \left(n_h^k(B)\right)^{-1/d_\S}+ \sqrt{\frac{\log(2 HK^2 / \pfail)}{n_h^k(B)}}
\end{align*}

For the case when $d_\S \leq 2$ we instead bound $\Exp{d_W(\tilde{T}_h(\cdot \mid x_0, a_0), T_h(\cdot \mid x_0, a_0)}$ using Proposition 1.1 from \cite{boissard2014mean}:
\begin{align*}
    \Exp{d_W(\tilde{T}_h(\cdot \mid x_0, a_0), T_h(\cdot \mid x_0, a_0)} & \leq \inf_{t > 0} ct + \frac{c}{\sqrt{n_h^k(B)}} \int_{t}^{1 / 4} N_r(\S)^{1/2} dr
\end{align*}
where $c$ is an absolute constant.  A straightforward calculation using that $N_r(\S) \lesssim r^{-d_\S}$ and picking $t \sim \log(n_h^k(B)) / \sqrt{n_h^k(B)}$ shows that for some absolute constant $c$,
\begin{align*}
    \Exp{d_W(\tilde{T}_h(\cdot \mid x_0, a_0), T_h(\cdot \mid x_0, a_0)} \leq c \frac{\log(n_h^k(B))}{\sqrt{n_h^k(B)}} \leq c \frac{\log(K)}{n_h^k(B)}
\end{align*}
and so:
\begin{align*}
    d_W(\tilde{T}_h(\cdot \mid x_0, a_0), T_h(\cdot \mid x_0, a_0)) & \leq \Exp{d_W(\tilde{T}_h(\cdot \mid x_0, a_0), T_h(\cdot \mid x_0, a_0)} + \sqrt{\frac{\log(2 HK^2 / \pfail)}{n_h^k(B)}} \\
    & \leq c \frac{\log(K)}{\sqrt{n_h^k(B)}} + \sqrt{\frac{\log(2 HK^2 / \pfail)}{n_h^k(B)}}
\end{align*}

\noindent \textbf{Step Four}: Using the assumption that $T_h$ is Lipschitz and $(x_0, a_0)$ and $(x,a) \in B$ we have that

$$d_W(T_h(\cdot \mid x,a), T_h(\cdot \mid x_0, a_0)) \leq L_T \D((x,a),(x_0, a_0)) \leq L_T \diam{B}.$$

Putting all of the pieces together we get that
\begin{align*}
&d_W(\Tbar{h}{k}(\cdot \mid B), T_h(\cdot \mid x,a)) \\
&\qquad \leq d_W(\Tbar{h}{k}(\cdot \mid B), \tilde{T}_h^k(\cdot \mid B)) + d_W(\tilde{T}_h^k(\cdot \mid B), \tilde{T}_h(\cdot \mid x_0, a_0)) \\
&\qquad\qquad + d_W(\tilde{T}_h(\cdot \mid x_0, a_0), T_h(\cdot \mid x_0, a_0)) + d_W(T_h(\cdot \mid x_0, a_0), T_h(\cdot \mid x, a)) \\
&\qquad \leq 4 \diam{B} + 4 L_T \diam{B} + \sqrt{\frac{8 \log(2HK^2 / \pfail)}{n_h^k(B)}} + d_W(\tilde{T}_h(\cdot \mid x_0, a_0), T_h(\cdot \mid x_0, a_0)) + L_T \diam{B} \\
& = (4 + 5 L_T) \diam{B} + d_W(\tilde{T}_h(\cdot \mid x_0, a_0), T_h(\cdot \mid x_0, a_0)) \\
& = \frac{1}{L_V}\left((4 + 5 L_T) \diam{B} + \tbonus{h}{k}(B)\right)
\end{align*}
after plugging in the bound from Step Four for the two cases depending on $d_\S$.  The result then follows via a union bound over $H$, $K$, the $K$ possible values of the random variable $n_h^k(B)$.  Per usual we do not need to union bound over the number of balls as the estimate of only one ball is updated per iteration. \Halmos
\end{rproof}

\subsection{Lipschitz Properties}

We first show that \cref{assumption:Lipschitz_mb} implies \cref{assumption:Lipschitz_mf}.  We present a slightly stronger result, showing that the Lipschitz constant scales with respect to $H$ across all steps.
\begin{lemma}[Relationship between Lipschitz Assumptions, \cref{lem:relation_lipschitz}]
Suppose that \cref{assumption:Lipschitz_mb} holds.  Then \cref{assumption:Lipschitz_mf} holds with Lipschitz constant at step $h$ as $L_V = \sum_{i=0}^{H-h} L_r L_T^{i}$.
\end{lemma}

\noindent We begin with the following lemma.
\begin{lemma}
\label{lemma:Lipschitz}
Suppose that $f : \S \times \A \rightarrow \mathbb{R}$ is $L$ Lipschitz and uniformly bounded.  Then $g(x) = \sup_{a \in \A} f(x,a)$ is also $L$ Lipschitz.
\end{lemma}
\begin{proof}
Fix any $x_1$ and $x_2 \in \S$.  First notice that $|f(x_1, a) - f(x_2, a)| \leq L \D((x_1, a), (x_2,a)) \leq L \D_\S(x_1, x_2)$ by choice of product metric.

Thus, for any $a \in \A$ we have that $$f(x_1, a) \leq f(x_2, a) + L \D_\S(x_1, x_2) \leq g(x_2) + L \D_\S(x_1, x_2).$$  However, as this is true for any $a \in \A$ we see that $g(x_1) \leq g(x_2) + L \D_\S(x_1, x_2)$.  Swapping the role of $x_1$ and $x_2$ in the inequality shows $|g(x_1) - g(x_2)| \leq L \D_\S(x_1, x_2)$ as needed.
\end{proof}

\begin{rproof}{\cref{lem:relation_lipschitz}}
We show this by induction on $h$.  For the base case when $h = H+1$ then $V_{H+1}^\star = Q_{H+1}^\star = 0$ and so the result trivially follows.

Similarly when $h = H$ then by Equation~\ref{eqn:bellman_equation} we have that $Q_H^\star(x,a) = r_H(x,a)$.  Thus, $|Q_H^\star(x,a) - Q_H^\star(x',a')| = |r_H(x,a) - r_H(x',a')| \leq L_r \D((x,a),(x',a'))$ by assumption.  Moreover, $V_H^\star(x) = \max_{a \in \A} Q_H^\star(x,a)$ is $L_r$ Lipschitz by Equation~\ref{eqn:bellman_equation} and Lemma~\ref{lemma:Lipschitz}.

For the step case we assume that $Q_{h+1}^\star$ and $V_{h+1}^\star$ are both $\sum_{i=0}^{H-h-1} L_r L_T^{i}$ Lipschitz and show the result for $Q_h^\star$ and $V_h^\star$.  Indeed,
\begin{align*}
    & |Q_h^\star(x,a) - Q_h^\star(x', a')| = |r_h(x,a) - r_h(x', a') + \Exp{V_{h+1}^\star(Y) \mid x,a} - \Exp{V_{h+1}^\star(Y) \mid x' a'}| \\
	& \leq |r_h(x,a) - r_h(x', a')| + |\Exp{V_{h+1}^\star(Y) \mid x,a} - \Exp{V_{h+1}^\star(Y) \mid x' a'}| \\
	& \leq L_r \D((x,a),(x',a')) + \left|\int V_{h+1}^\star(\cdot) \, dT_h(\cdot \mid x,a) - \int V_{h+1}^\star(\cdot) \, dTr_h(\cdot \mid x', a')\right|
\end{align*}
Now denoting by $K = \sum_{i=0}^{H-h-1} L_1 L_2^{i}$ by the induction hypothesis and the properties of Wasserstein metric (Equation 3 from \cite{gibbs2002choosing}) we have that
\begin{align*}
|Q_h^\star(x,a) - Q_h^\star(x', a')| & \leq L_1 \D((x,a),(x',a')) \\
& \qquad + K \left|\int \frac{1}{K} V_{h+1}^\star(\cdot) \, d T_h(\cdot \mid x,a) - \int \frac{1}{K} V_{h+1}^\star(\cdot) \, dT_h(\cdot \mid x', a')\right| \\
	& \leq L_r \D((x,a),(x',a')) + K d_W(T_h(\cdot \mid x, a), T_h(\cdot \mid x', a')) \\
	& \leq L_r \D((x,a),(x',a')) + K L_T \D((x,a), (x',a')) \\
	& = (L_r + K L_T)\D((x,a),(x',a')).
\end{align*}
Noting that by definition of $K$ we have $L_r + K L_T = \sum_{i=0}^{H-h-1} L_r L_T^{i}$.  To show that $V_h^\star$ is also Lipschitz we simply use Lemma~\ref{lemma:Lipschitz}. 
\end{rproof}

\subsection{Zooming Dimension Examples}

\begin{lemma}[Example Bounds on Zooming Dimension, \cref{lem:zooming_dim_value}]
The following examples show improved scaling of the zooming dimension over the ambient dimension.
\begin{itemize}
    \item {Linear $Q_h^\star$:} Suppose that $Q_h^\star(x,a) = \theta^\top (x,a)$ for some vector $\theta \in \mathbb{R}^{d_\S + d_\A}$ with $\S \subset \mathbb{R}^{d_\S}$ and $\A \subset \mathbb{R}^{d_\A}$ under any $\ell_p$ norm.  Then $z_h \leq d_\S + d_\A - \norm{\theta_{\A}}_0$.  %In this setting the optimal action corresponds to the all $1$'s vector, independent of the state.  Using that the condition defining $Z_h^r$ is independent of the states and coordinates of $\A$ such that $\theta_a = 0$ the zooming dimension is $d_\S + d_\A - \norm{\theta_{\A}}_{0}$ while the ambient dimension is $d_\S + d_\A$.
    \item {Low-Dimensional Optimality:} Suppose that there exists a set $Y \subset \A$ which contains all optimal or near-optimal actions for every state.  Then $z_h \leq d_\S + d_Y$.  
    \item {Strongly Concave:} Suppose that the metric space is $\S = [0,1]^{d_\S}$ and $\A = [0,1]^{d_\A}$ under any $\ell_p$ metric.  If $Q_h^\star(x,a)$ is $C^2$ smooth, and for all $x \in \S$ we have that $Q_h^\star(x, \cdot)$ has a unique maxima and is strongly concave in a neighborhood around the maxima, then $z_h \leq d_\S + \frac{d_\A}{2}.$
\end{itemize}
\end{lemma}
\begin{rproof}{\cref{lem:zooming_dim_value}}
For all of the following derivations we let $c$ be a (possibly changing) constant from line-to-line independent of the value of $r$.  We also use the constant-order equivalence of the packing and covering numbers.

\noindent \textbf{Linear $Q_h^\star$:}  First suppose that $Q_h^\star(x,a) = \theta^\top (x,a)$ for some $\theta \in \mathbb{R}^{d_\S + d_\A}$ with $\S \subset \mathbb{R}^{d_\S}$ and $\A \subset \mathbb{R}^{d_\A}$.  First notice that the optimal action satisfies $\pi_h^\star(x) = \argmax_{a \in \A} Q_h^\star(x,a) = \argmax_{a \in \A} \theta^\top (x,a) = \argmax_{a \in \A} \theta_\A^\top a$ where $\theta_\A$ is the indices of $\theta$ corresponding to the action space.  Denote this action as $a^\star$, and notice it is independent of $x$.  Thus we have that:
\begin{align*}
    Z_h^r & = \{ (x,a) \in \S \times \A : \gap_h(x,a) \leq cr\} \\
    & = \{ (x,a) \in \S \times \A : \theta_\A^\top(a^\star - a) \leq c r\}.
\end{align*}
It is easy to see that the definition of $Z_h^r$ is independent of the state $x$ and the indices of $a$ such that $\theta_\A = 0$.  We construct an $r$-covering of $Z_h^r$ as follows.  First let $S$ be an $r-$covering of $\S$, and $A_1$ an $r-$covering of $\A_{J}$ where $J$ denotes the indices of the action space such that $\theta_\A$ is zero.  Consider the final set $C = \{(x,a) : x \in S, a \in \A_{J}\}$.  We have that $|C| \leq cr^{-(d_\S + d_\A - \norm{\theta_\A}_0)}$.  To see that $z_h \leq d_\S + d_\A - \norm{\theta_\A}_0$ it suffices now to show that $C$ covers $Z_h^r$.  Indeed, given any $(x,a) \in Z_h^r$ let $x_i$ be the point in $S$ closest to $x$ and $a_j$ the point in $\A_{J}$ closest to $a_J$ (i.e. the sub-vector corresponding to indexes where $\theta_\A = 0$).  Consider the final tuple $(x_i, (a_j, a^\star_{\neg J}))$ where we abuse notation and let $(a_j, a^\star_{\neg J})$ be the vector $a$ whose components in $J$ are equal to $a_j$ and components not in $J$ are equal to $a^\star$.  It is easy to see that:
\begin{align*}
    \D((x,a), (x_i, (a_j, a^\star_{\neg J}))) \leq \D_\S(x_i, x) + \D_\A(a, (a_j, a^\star_{\neg J})) \leq 3r.
\end{align*}
Indeed, the first term is bounded by $r$ since $S$ is an $r$-covering of $\S$.  The second term is bounded by $2r$ since the distance from $a_J$ to $a_j$ is bounded by $r$ by construction of $\A_{J}$.  For the last component we use that $(x,a) \in Z_h^r$ so that $r \geq \gap_h(x,a) = \theta_\A^\top(a^\star_{\neg J} - a_{\neg J}) \geq \norm{\theta_{\A}}_{min} \norm{a^\star_{\neg J} - a_{\neg J}}$.

\noindent \textbf{Low-Dimensional Optimality:} Suppose that there exists a set $Y \subset \A$ which contains all of the optimal or near-optimal actions for every state.  Let $d_Y$ be the covering dimension of $Y$.  By assumption we know that there exists an $r_{thresh}$ such that if $(x,a) \in Z_h^r$ for $r \leq r_{thresh}$ then $a \in Y$.  By letting the constant in the zooming dimension calculation be $r_{thresh}^{d_Y}$ we see that $N_r(Z_h^r) \leq r_{thresh}^{d_Y} r^{-(d_\S + d_Y)}$ for all $r > 0$ as required.

\noindent \textbf{Strongly Concave:} Suppose that the metric space is $\S = [0,1]^{d_\S}$ and $\A = [0,1]^{d_\A}$ under any $\ell_p$ metric,  $Q_h^\star(x,a)$ is $C^2$ smooth, $Q_h^\star(x, \cdot)$ has a unique maxima, and is strongly concave in a neighborhood around the maxima.  First notice that for any $(x,a) \in Z_h^r$ for $r$ less than the threshold for $Q_h^\star$ being strongly concave that:
\[
r \geq \gap_h(x,a) = Q_h^\star(x, \pi_h^\star(x)) - Q_h^\star(x,a) \geq \frac{\mu}{2} \norm{\pi^\star(x) - a}_2^2
\]
where $\mu$ is the strong concavitiy parameter.  We construct an $r$ covering for $Z_h^r$ as follows.  Let $S$ be an $r$ covering of $\S$.  For every $x \in S$ denote by $A_x$ as an $r$ covering of the $\ell_2$ ball centered at $\pi_h^\star(x)$ with radius $\sqrt{r}$.  Consider the final set $C = \{(x,a) \mid x \in S, a \in A_x\}$.  It is easy to see that $|C| \leq cr^{-(d_\S + d_\A / 2)}$ since the size of the $r$ covering of the $\ell_2$ ball with radius $\sqrt{r}$ scales as $r^{-d_\A / 2}$.  To show that $z_h \leq d_\S + d_\A / 2$ it suffices to show that $C$ covers the set $Z_h^r$.

Let $(x,a) \in Z_h^r$ be arbitrary.  Set $x_i \in S$ to be the point closest to $x$.  Set $a_j \in A_{x_i}$ the point closest to $\pi_h^\star(x_i)$ in $A_{x_i}$.  Then we have that:
\begin{align*}
    \D((x,a), (x_i, a_j)) & \leq \D_\S(x, x_i) + \D_\A(a, a_j) \\
    & \leq r + \D_\A(a, \pi_h^\star(x)) + \D_\A(\pi_h^\star(x), \pi_h^\star(x_i)) + \D_\A(\pi_h^\star(x_i), a_j) \\
    & \leq 4r.
\end{align*}
The first inequality uses that $(x,a) \in Z_h^r$ so that $a$ is $r$ close to $\pi_h^\star(x)$ in the neighborhood.  The second inequality uses the Lipschitz continuity of the argmax of a $C^2$ strongly concave function.  The last uses that $a_j$ is $r$ close to $\pi_h^\star(x_i)$.

\end{rproof}

\begin{lemma}[Lower Bound on Zooming Dimension, \cref{lem:zoom_d_s}]
For any $h$ we have that $z_h \geq d_\S - 1$.
\end{lemma}
\begin{rproof}{\cref{lem:zoom_d_s}}
Consider the set $\S^{opt} = \{(x, \pi_h^\star(x)) \mid x \in \S\}$.  By definition of the $\gap$ we have that $\S^{opt} \subset Z_h^r$ for any $r > 0$.  Trivially then $N_r(Z_h^r) \geq N_r(\S^{opt})$.  However, any packing of $\S^{opt}$ can be converted into a packing of $\S$ by simply removing the component corresponding to the action.  Thus, $N_r(\S^{opt}) \geq N_r(\S)$.  Since the covering dimension of the state space is $d_\S$ we have $N_r(\S) \geq cr^{-(d_\S - 1)}$ for any $r > 0$.  Combining all of these inequalities we see that $N_r(Z_h^r) \geq cr^{-(d_\S - 1)}$.  However, using the definition of the zooming dimension we also have that $N_r(Z_h^r) \leq cr^{-z_h}$.  These two inequalities implies that $z_h \geq d_\S - 1$ as required. \Halmos
\end{rproof}

\subsection{Lower Bound}

\begin{theorem}[Regret Lower Bound, \cref{thm:lower_bound}]
Let $(\S, \D_\S)$ and $(A, \D_\A)$ be arbitrary metric spaces, and $\D$ be the product metric.  Fix an arbitrary time horizon $K$ and number of steps $H$.  There exists a distribution $\I$ over problem instances on $(\S \times \A, \D)$ such that for any algorithm: \[
\E_{\I}[\regret(K)] \geq \Omega\left(\sum_{h=1}^H K^{\frac{z_h + 1}{z_h + 2}} / \log(K)\right)
\]
\end{theorem}
\begin{rproof}{\cref{thm:lower_bound}}
Theorem 5.1 in \cite{slivkins_2014} shows that there is a collection of problem instances $\I$ when $H = 1$ such that any algorithm has regret $\E_\I[\regret(K)] \geq \Omega(K^{\frac{z_h + 1}{z_h+2}} / \log(K))$.  We can take the set of problem instances $\I$ to create a set of $H$-step MDPs as follows.  First, let $I_1, \ldots, I_H$ be an arbitrary selection of problem instances where each $I = (\mu, T_\S)$ where $\mu$ is the reward distribution and $T_\S$ is the context distribution.  We construct the MDP as follows:
\begin{itemize}
    \item Reward function $R_h = \mu_h$ where $\mu_h$ is the reward distribution for problem instance $I_h$
    \item Transition distribution $T_h(\cdot \mid x,a) = T^h_\S$ where $T^h_\S$ is the context distribution for problem instance $I_h$
\end{itemize}
It is easy to see that this generates an $H$-step MDP where the transition distribution is independent from the state-action pair.  Let $\tilde{I}$ be the set of problem instances constructed in this way.  Note that the regret for any algorithm can be decomposed into the sum of the regret for specific steps $h$ as the transition and rewards are independent of the state-action pairs and across steps $h$.  Moreover, the zooming dimension at step $h$ $z_h$ corresponds to the zooming dimension for the reward function $\mu_h$.  Letting $\regret_h(K)$ be the regret for the step-$h$ process we have
\[
\E_{\tilde{\I}}[\regret(K)] = \sum_{h=1}^H \E_{\I}[\regret_h(K)] \geq \Omega\left(\sum_{h=1}^H K^{\frac{z_h + 1}{z_h + 2}} / \log(K)\right)
\]
as required. \Halmos
\end{rproof}

See \cite{slivkins_2014} for more discussion on the lower bound construction used and a description of the problem instances.  At a high level, each reward function $\mu_h$ can be thought of as a ``needle in a haystack'' where the haystack corresponds to actions with reward $\frac{1}{2}$ and the needle is a specific ball whose reward is slightly higher.

\begin{figure}
\centering     %%% not \center
\subfigure[Laplace $d=1, \alpha=0$]{\label{fig:1a}\includegraphics[width=.4\linewidth]{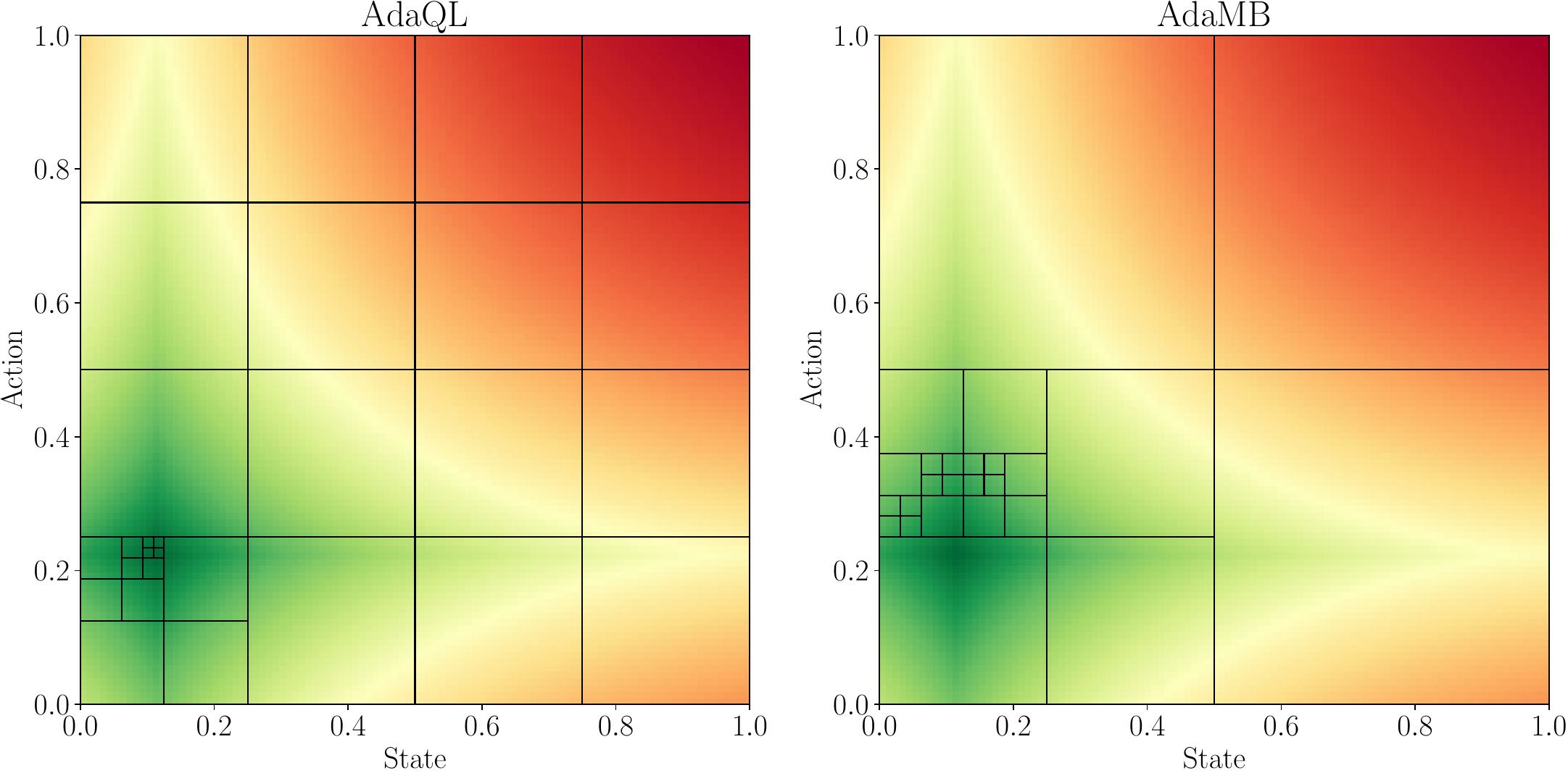}}
\subfigure[Quadratic $d=1, \alpha=0$]{\label{fig:1b}\includegraphics[width=.4\linewidth]{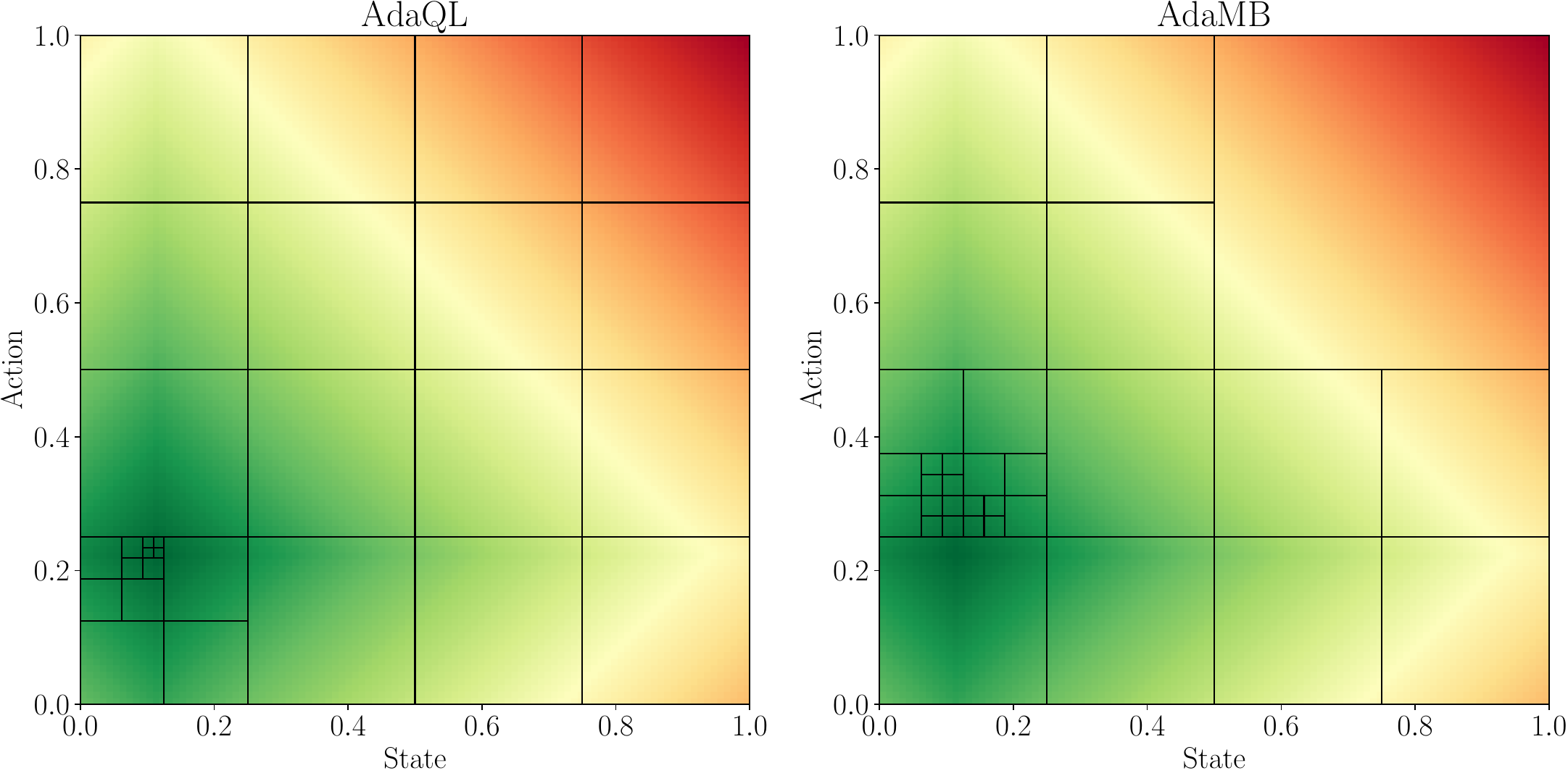}}
\subfigure[Laplace $d=1, \alpha=0.1$]{\label{fig:1c}\includegraphics[width=.4\linewidth]{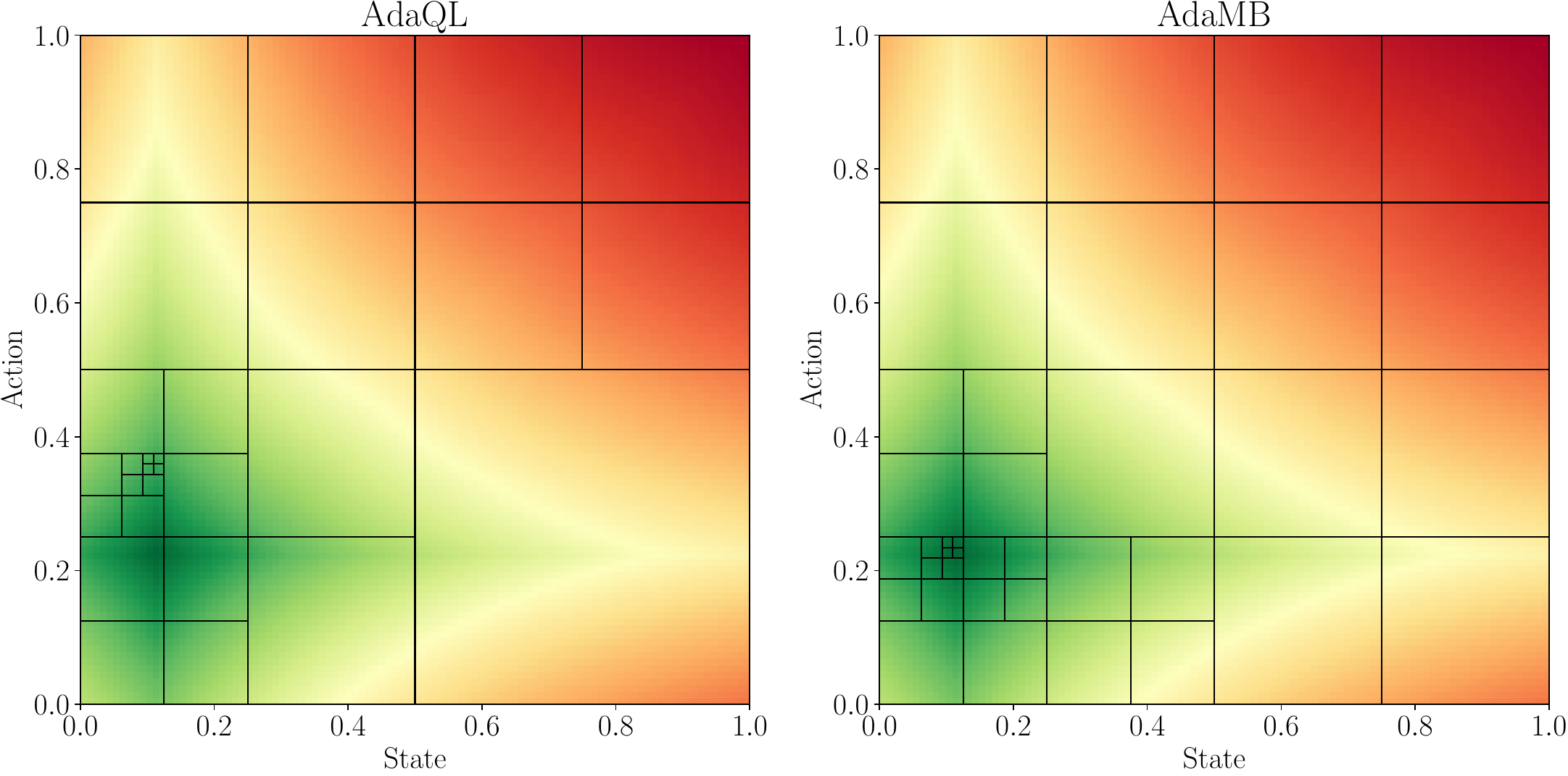}}
\subfigure[Quadratic $d=1, \alpha=0.1$]{\label{fig:1d}\includegraphics[width=.4\linewidth]{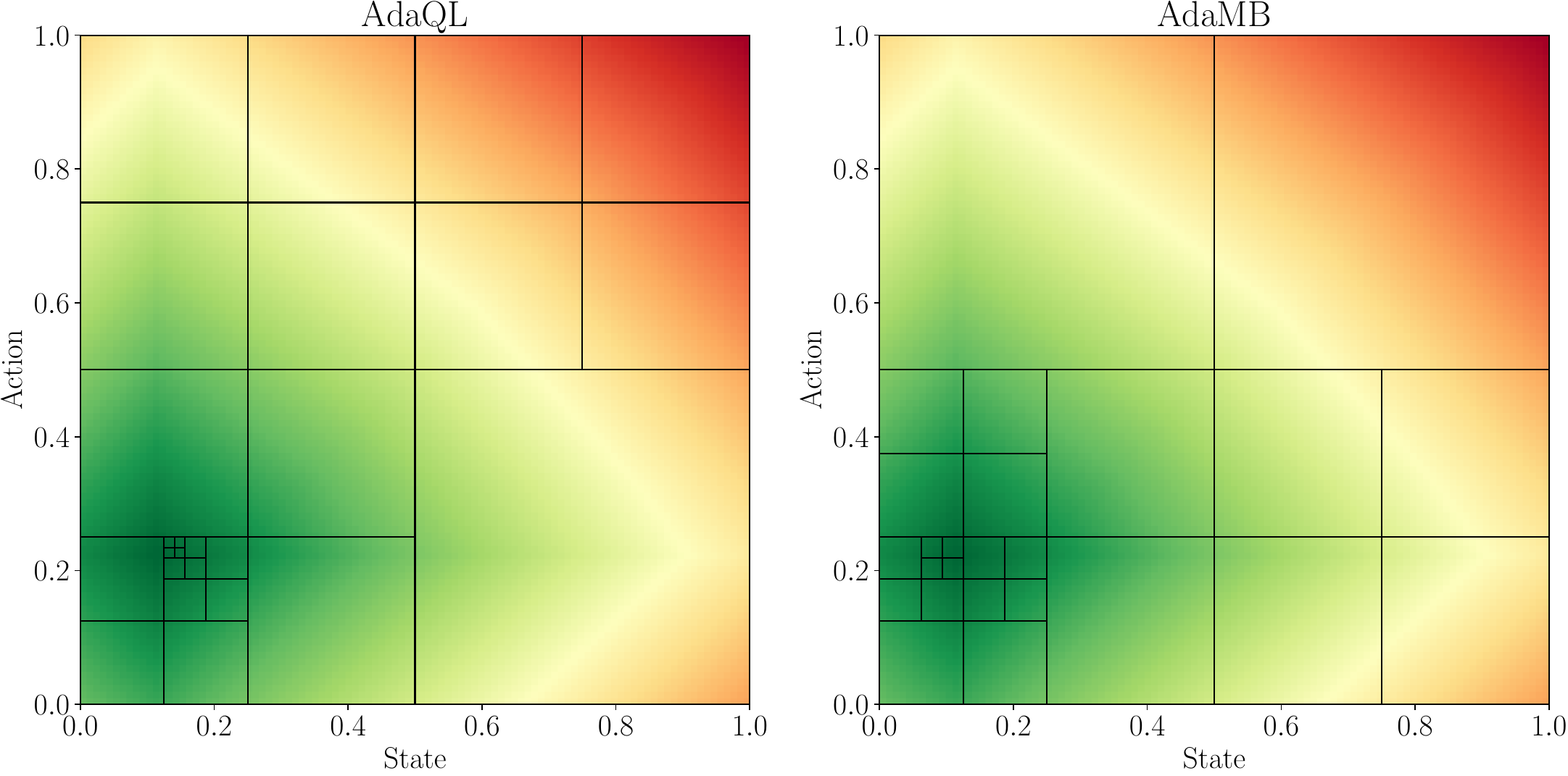}}
\subfigure[Laplace $d=1, \alpha=0.5$]{\label{fig:1e}\includegraphics[width=.4\linewidth]{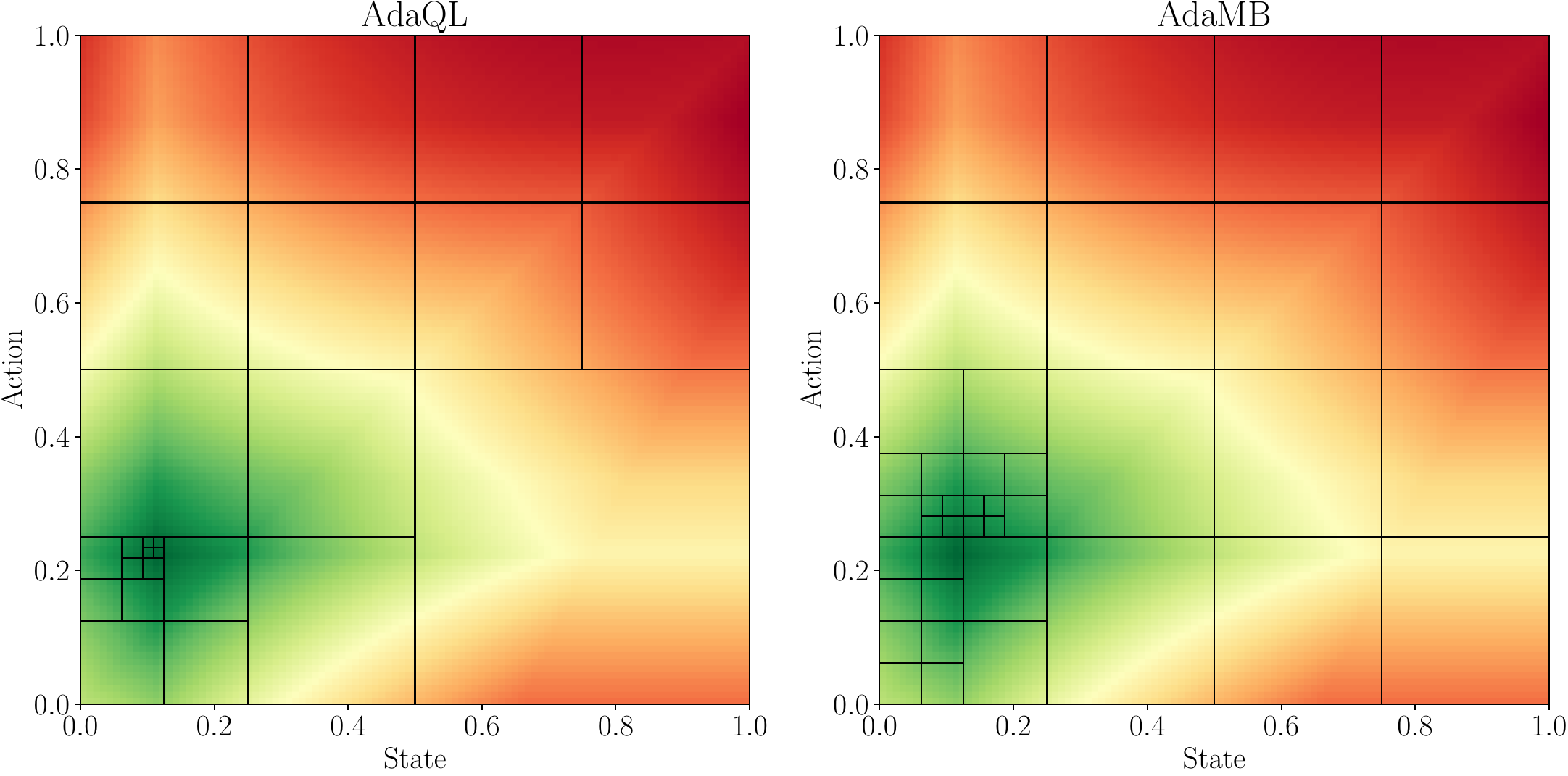}}
\subfigure[Quadratic $d=1, \alpha=0.5$]{\label{fig:1f}\includegraphics[width=.4\linewidth]{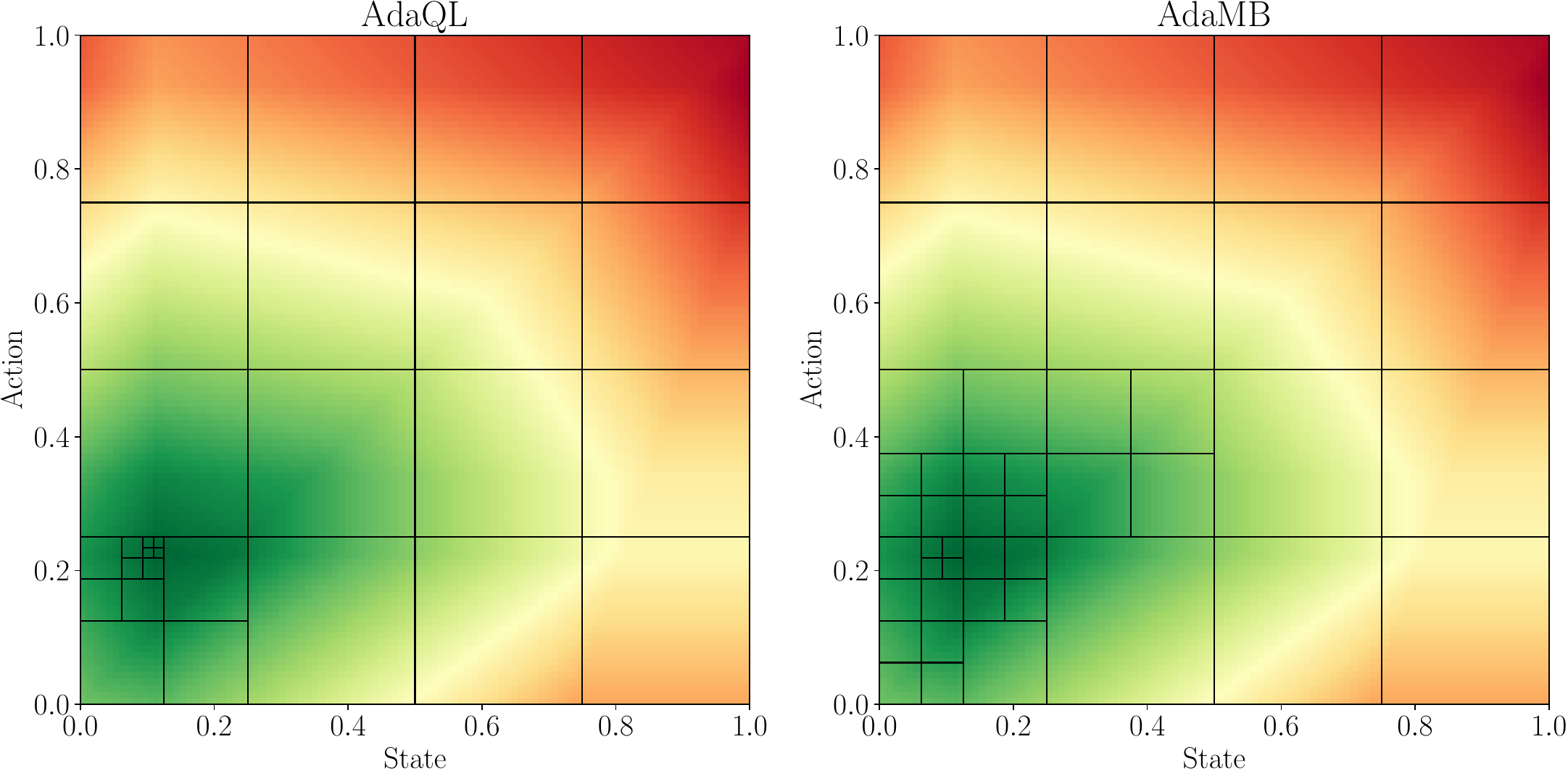}}
\caption{Comparison of the discretization observed between \AdaMB and \AdaQL for the oil environment with $d = 1$ at step $h = 2$.  The underlying colors correspond to the true $Q_2^\star$ function where green corresponds to a higher value.  In all of the results we see that the adaptive discretization algorithms maintain a level of discretization proportional to the underlying $Q_h^\star$ value.  This leads to sample complexity gains as the algorithm quickly learns where the set of near-optimal state action pairs are, and space and time complexity improvements by not maintaining a fine discretization across the entire space.  For the settings when $\alpha = 0$ where we have concrete zooming dimension improvements (as in \cref{fig:1a}) these zooming dimension guarantees justifies the improvements on the adaptive algorithms.}
\label{fig:disc_oil}
\end{figure}

\begin{table}[!tb]
\setlength\tabcolsep{0pt} % let LaTeX compute intercolumn whitespace
\centering
\begin{tabular*}{\columnwidth}{@{\extracolsep{\fill}}lcccc}
\toprule
  Algorithm Type  & Oil ($d=1$) & Oil ($d=2$) & Ambulance ($k=1$) & Ambulance ($k=2$) \\
\midrule
  Model-Based & 42\% & 36\% & 51\% & 16\% \\
  Model-Free & 35\% & 37\% & 52\% & 41\% \\
\bottomrule
\end{tabular*}
\caption{\em Comparison of the size of the uniform discretization to the adaptive discretization algorithms.  Each value corresponds to the average size of the partition for the adaptive algorithm divided by the size of the partition for the uniform algorithm.  Lower values corresponds to larger improvements and a smaller partition.} 
\label{tab:size_partition}
\end{table}

\begin{figure}
\centering     %%% not \center
\subfigure[Beta $k=1, \alpha=0$]{\label{fig:2a}\includegraphics[width=.4\linewidth]{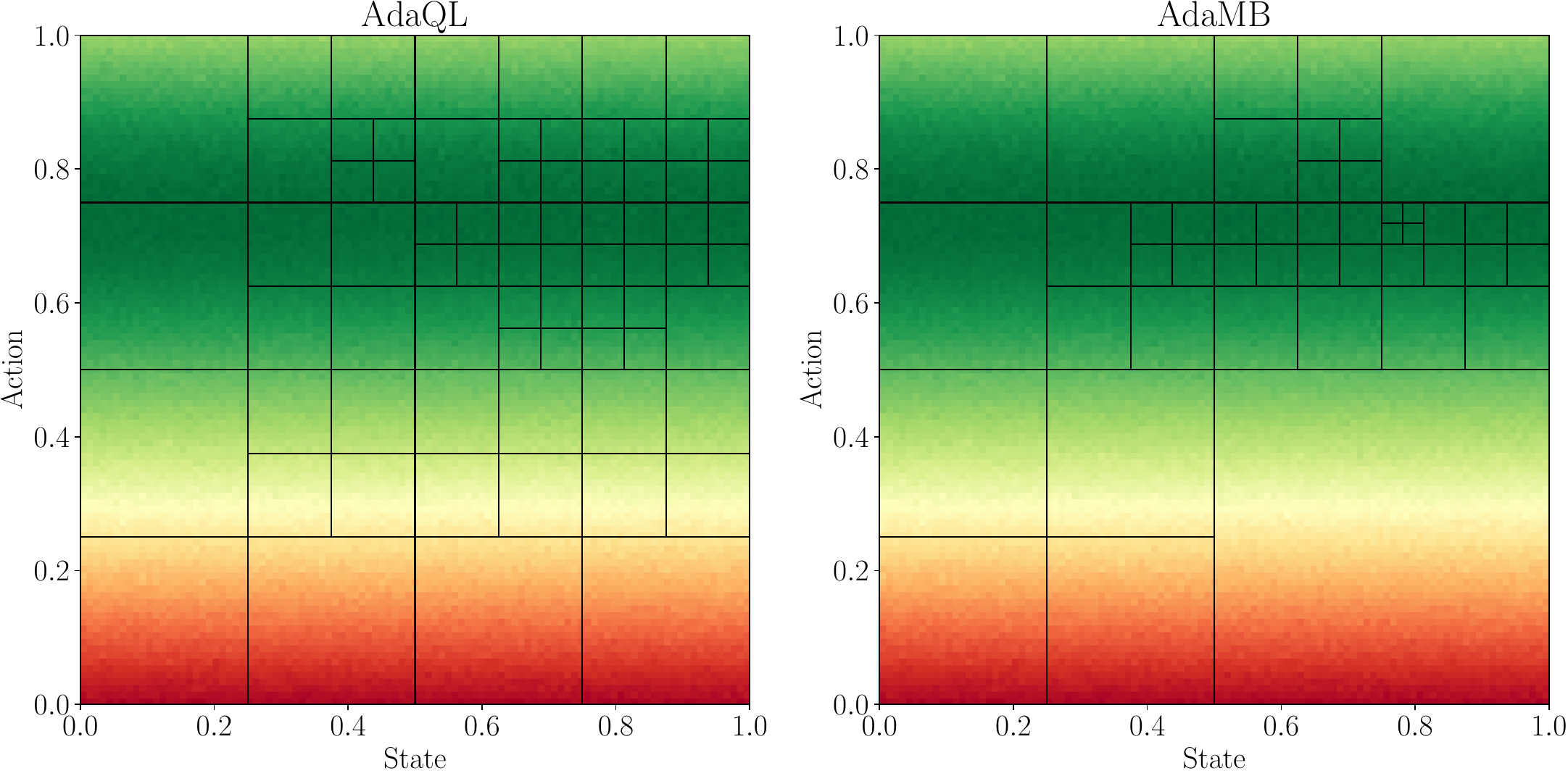}}
\subfigure[Shifting $k=1, \alpha=0$]{\label{fig:2b}\includegraphics[width=.4\linewidth]{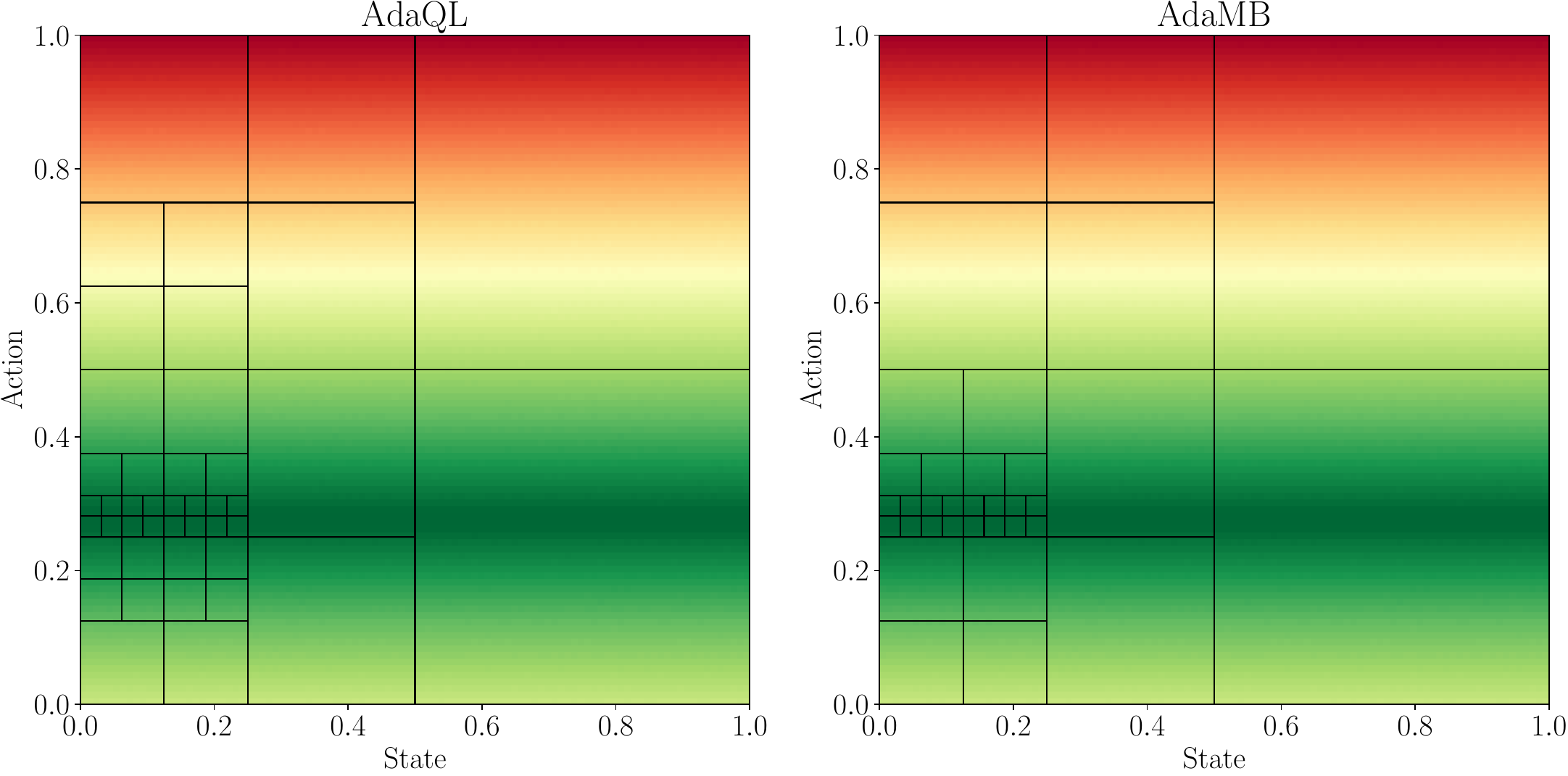}}
\subfigure[Beta $k=1, \alpha=0.25$]{\label{fig:2c}\includegraphics[width=.4\linewidth]{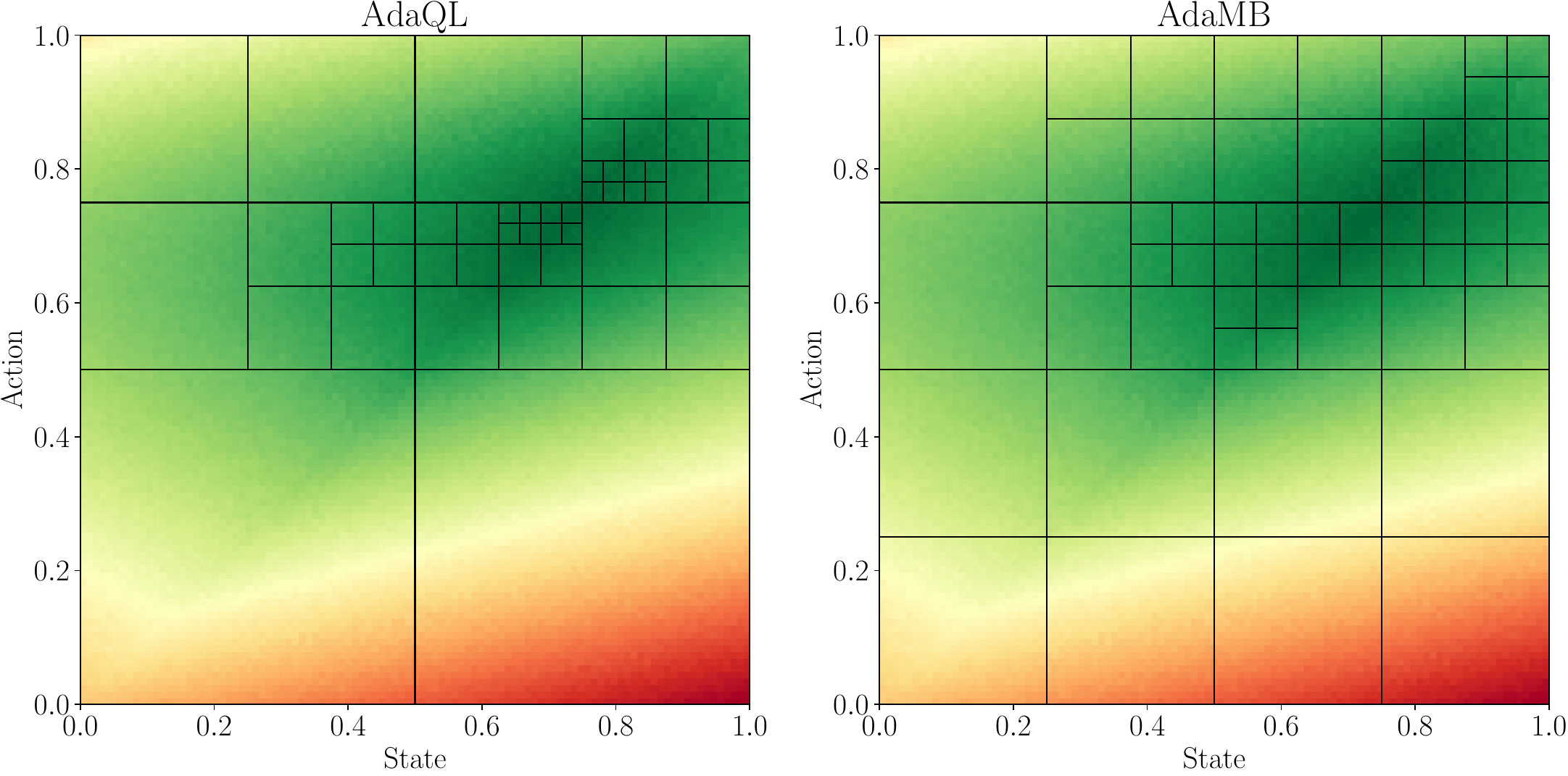}}
\subfigure[Shifting $k=1, \alpha=0.25$]{\label{fig:2d}\includegraphics[width=.4\linewidth]{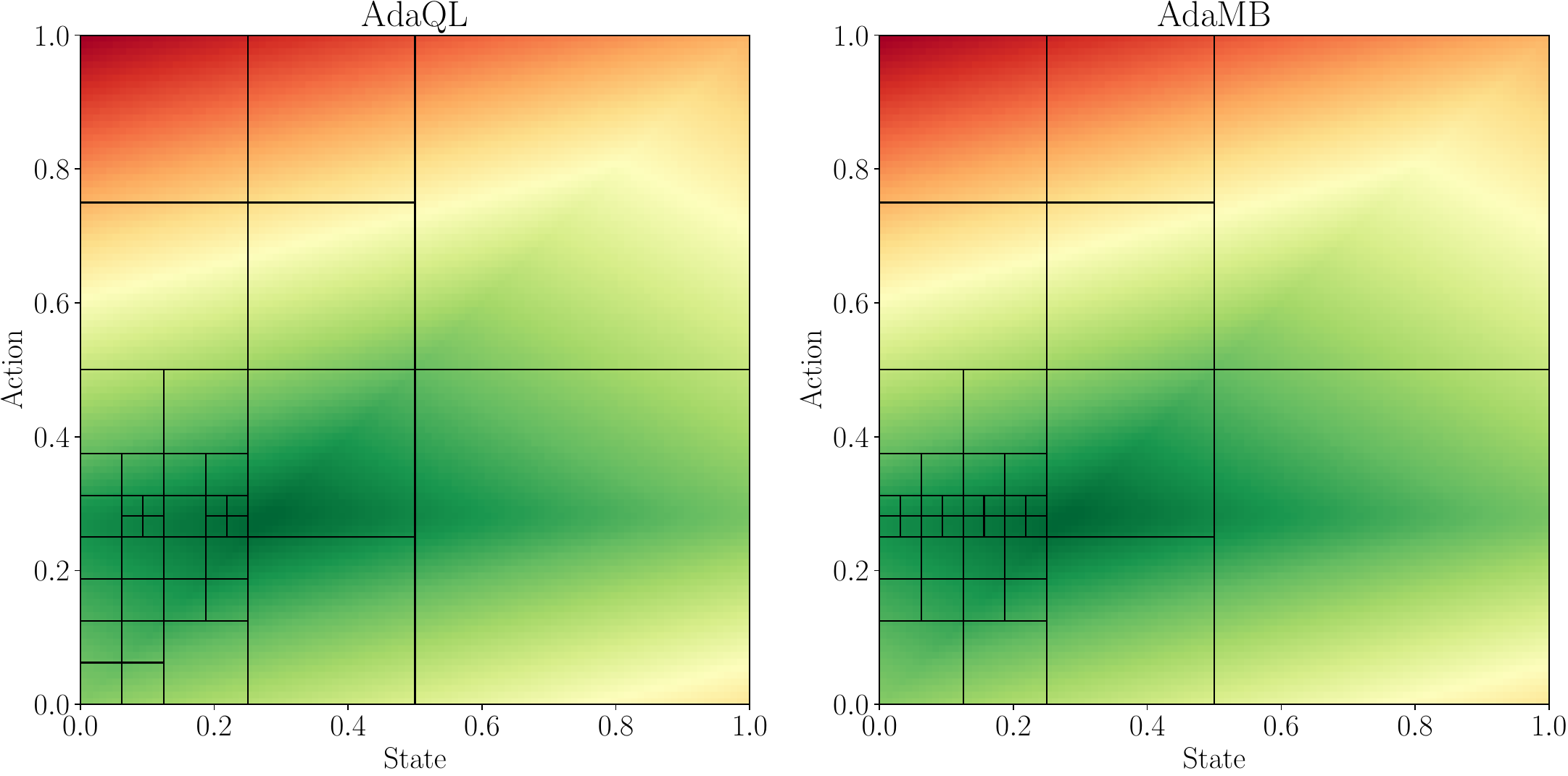}}
\subfigure[Beta $k=1, \alpha=1$]{\label{fig:2e}\includegraphics[width=.4\linewidth]{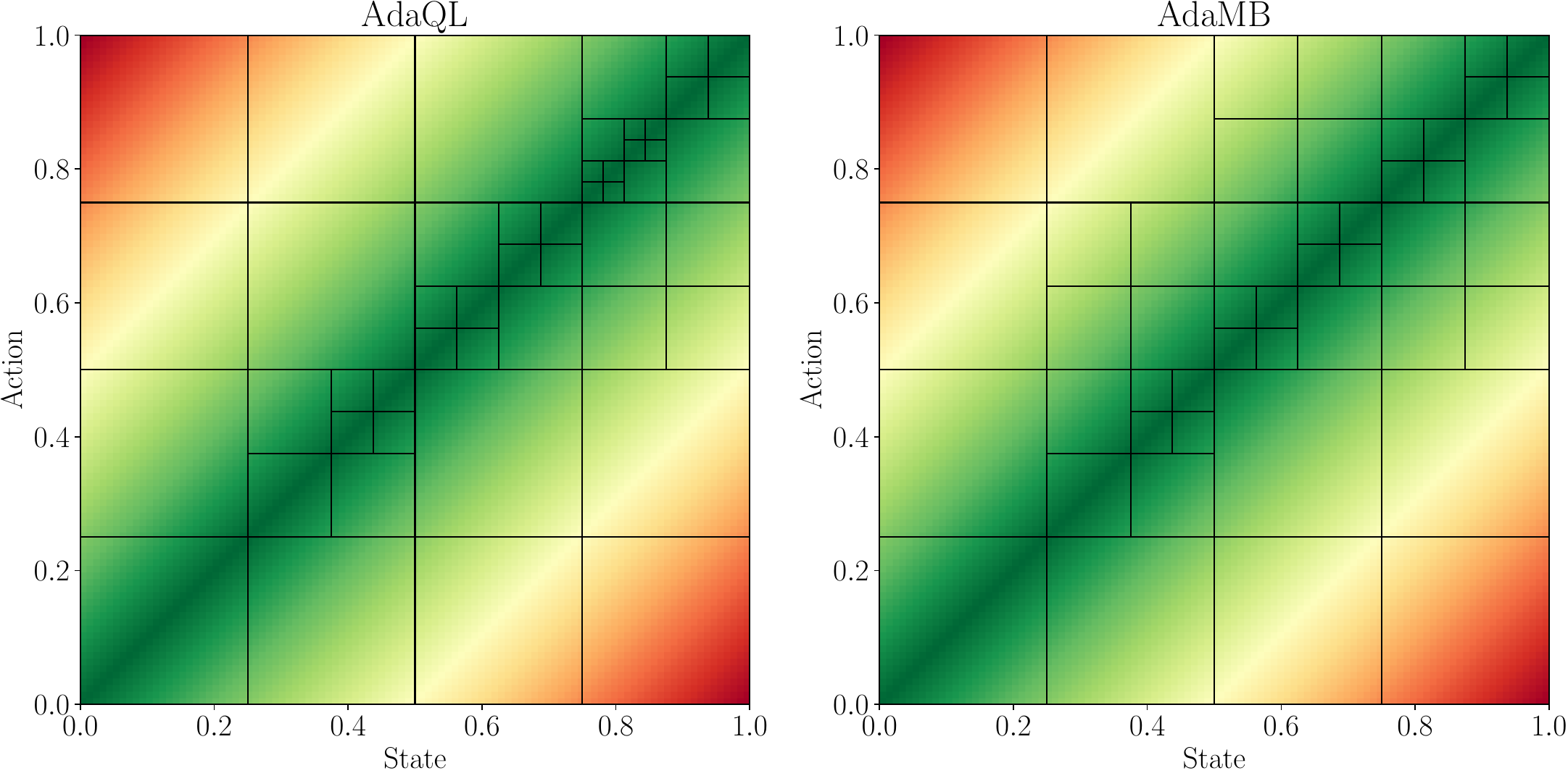}}
\subfigure[Shifting $k=1, \alpha=1$]{\label{fig:2f}\includegraphics[width=.4\linewidth]{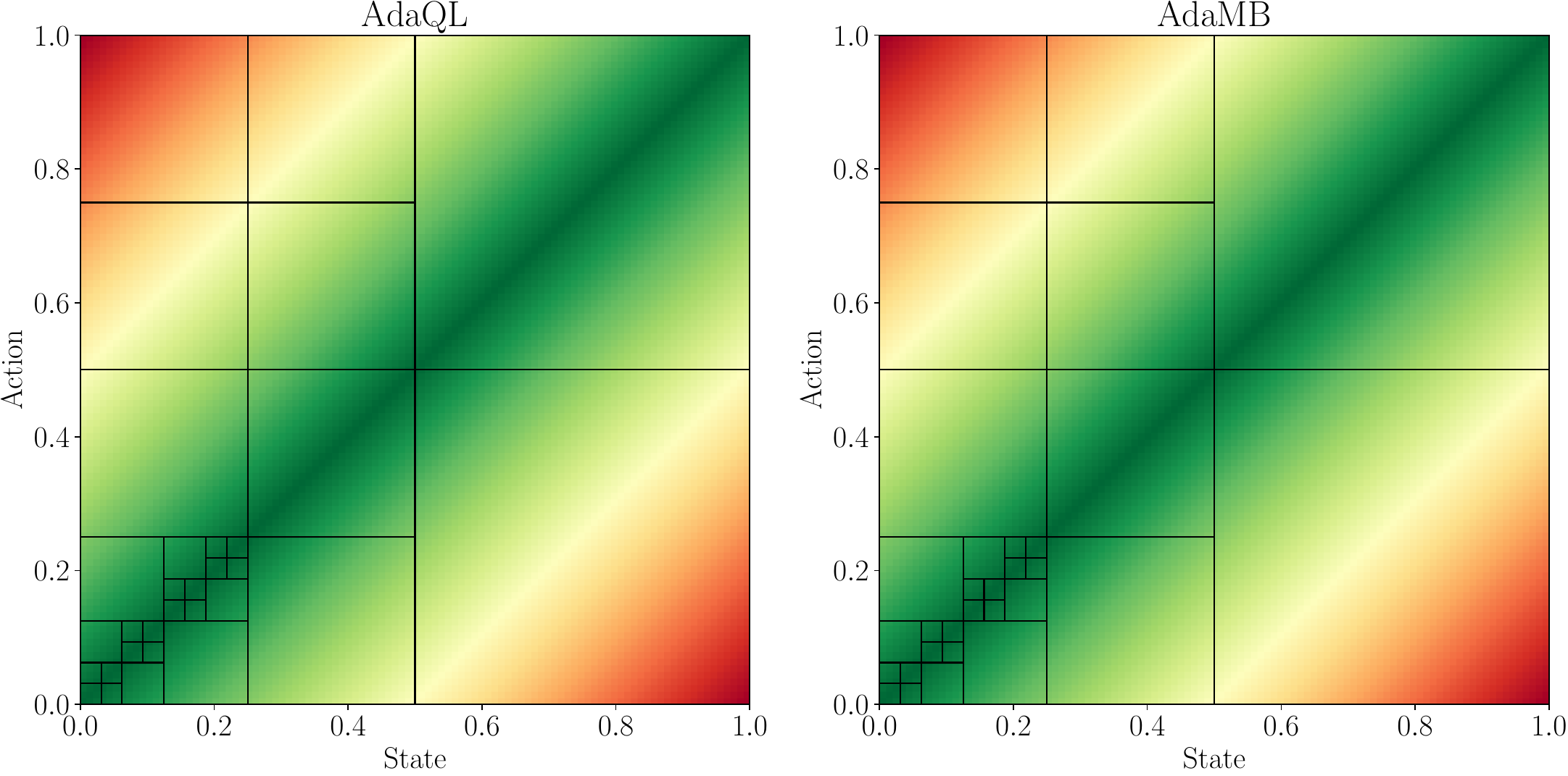}}
\caption{Comparison of the discretization observed between \AdaMB and \AdaQL for the ambulance environment with $k = 1$ at step $h = 2$.  While the zooming dimension gives no improvements on the state-space dependence, empirically we see the algorithm only maintaining a discretization on states induced by the visitation distribution of the optimal policy.}
\label{fig:disc_amb}
\end{figure}

\begin{figure}
\centering     %%% not \center
\subfigure[Laplace $d=1, \alpha=0$]{\label{fig:3a}\includegraphics[width=.49\linewidth]{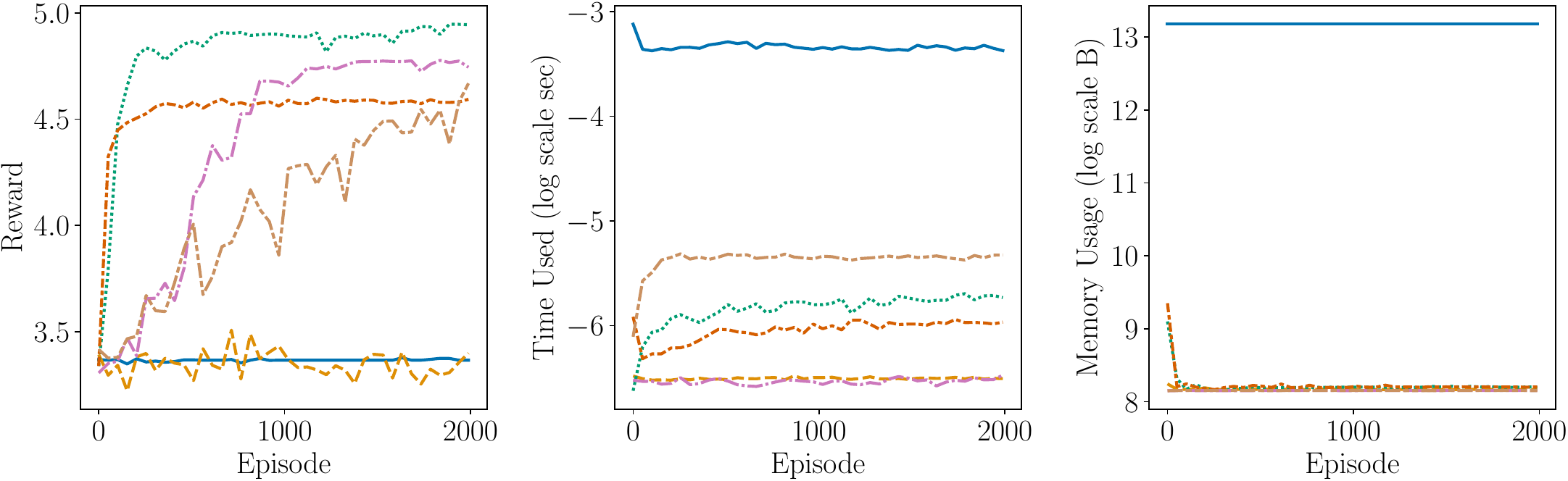}}
\subfigure[Laplace $d=1, \alpha = 0.1$]{\label{fig:3b}\includegraphics[width=.49\linewidth]{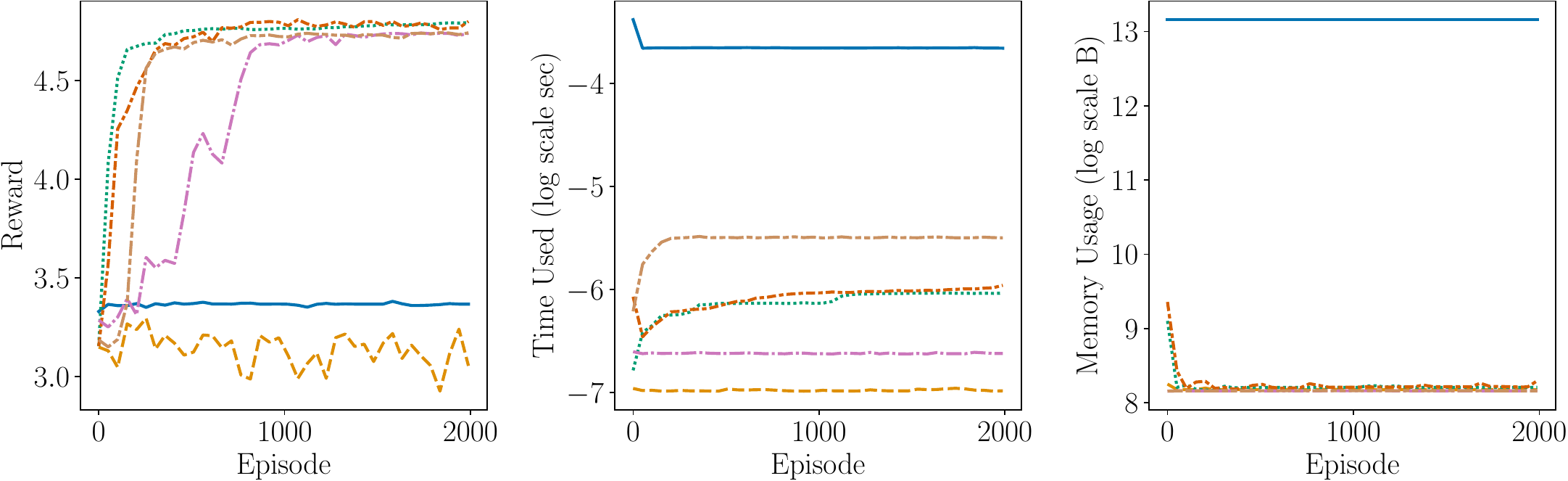}}
\subfigure[Laplace $d=1, \alpha=0.5$]{\label{fig:3c}\includegraphics[width=.49\linewidth]{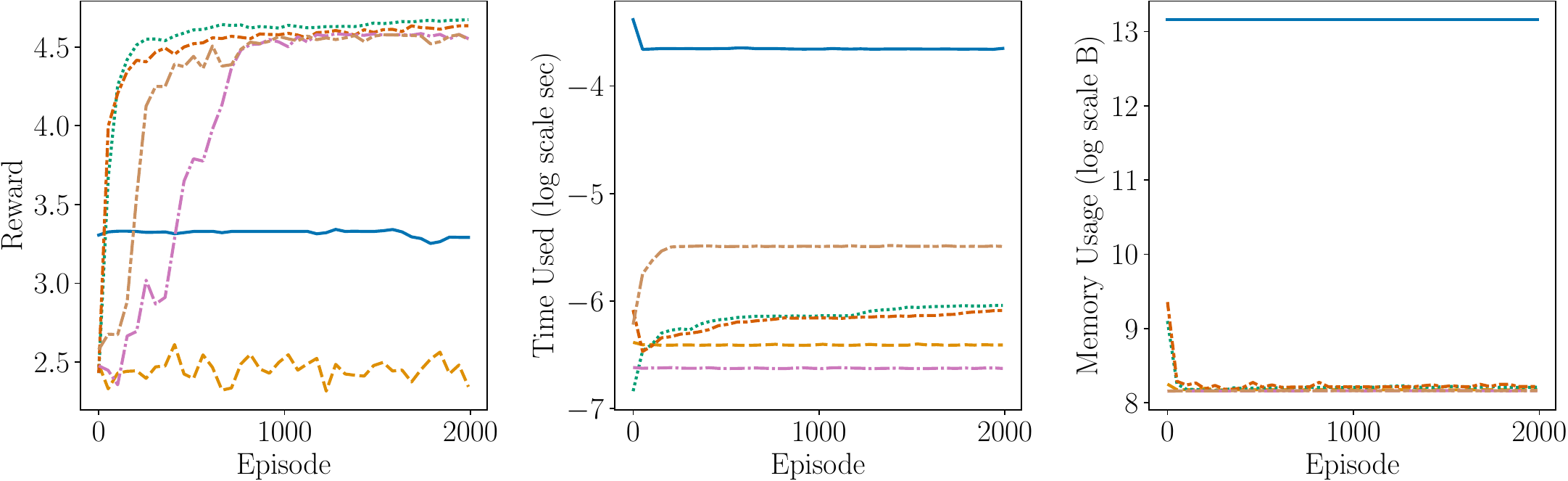}}
\subfigure[Laplace $d=2, \alpha=0$]{\label{fig:3d}\includegraphics[width=.49\linewidth]{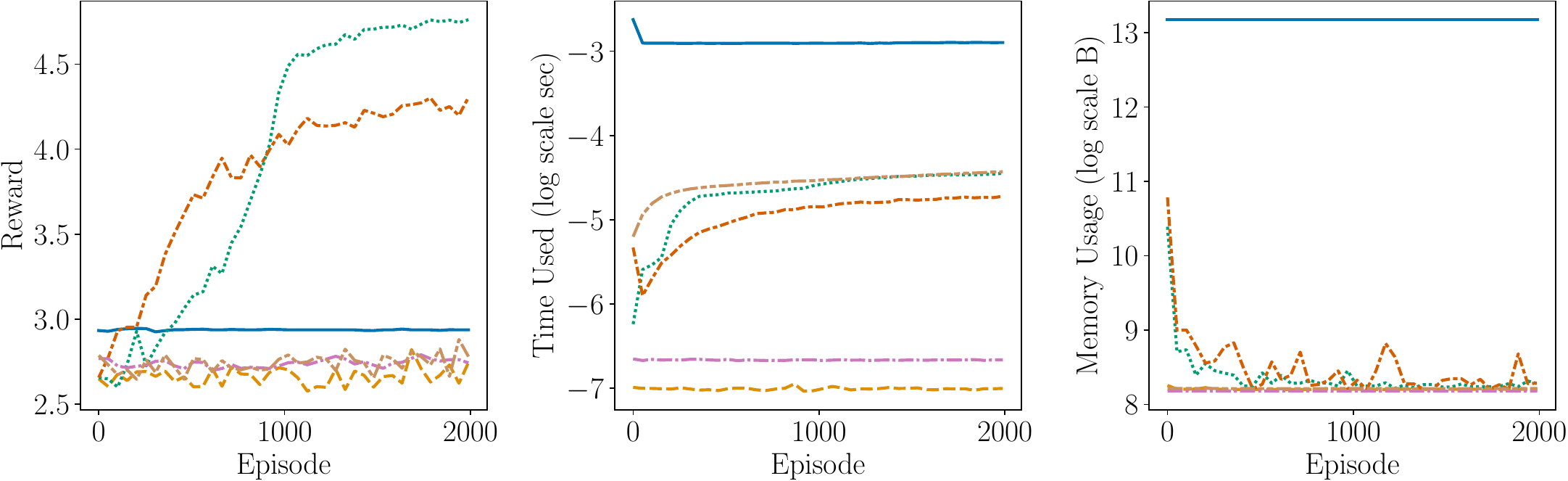}}
\subfigure[Laplace $d=2, \alpha=0.1$]{\label{fig:3e}\includegraphics[width=.49\linewidth]{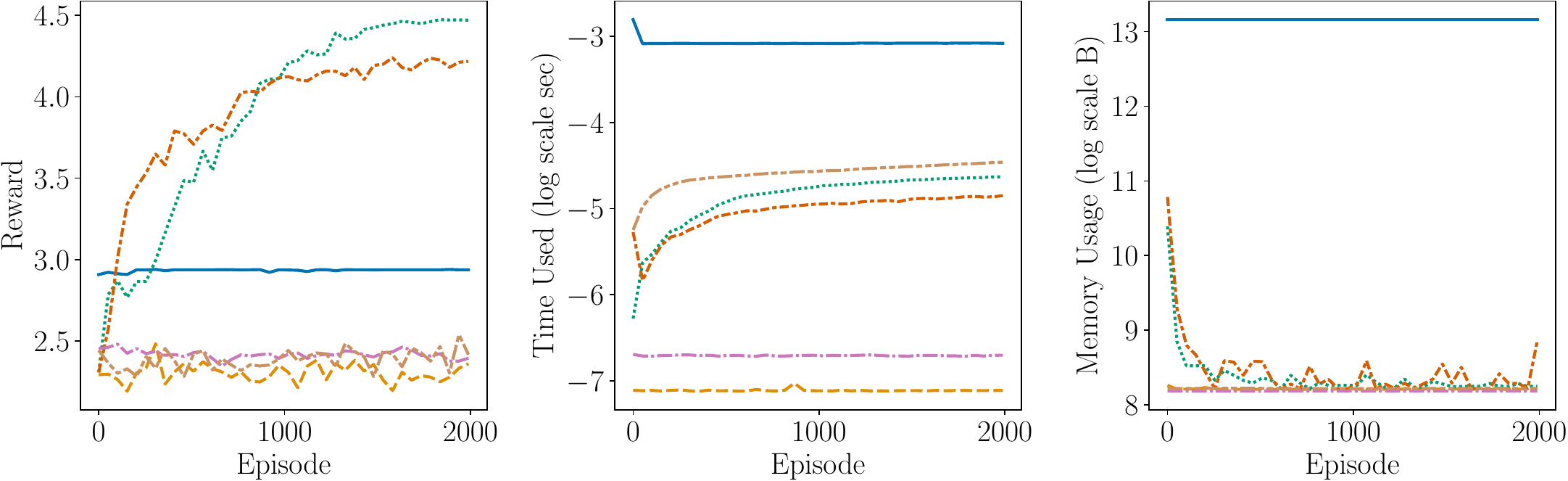}}
\subfigure[Laplace $d=2, \alpha=0.5$]{\label{fig:3f}\includegraphics[width=.49\linewidth]{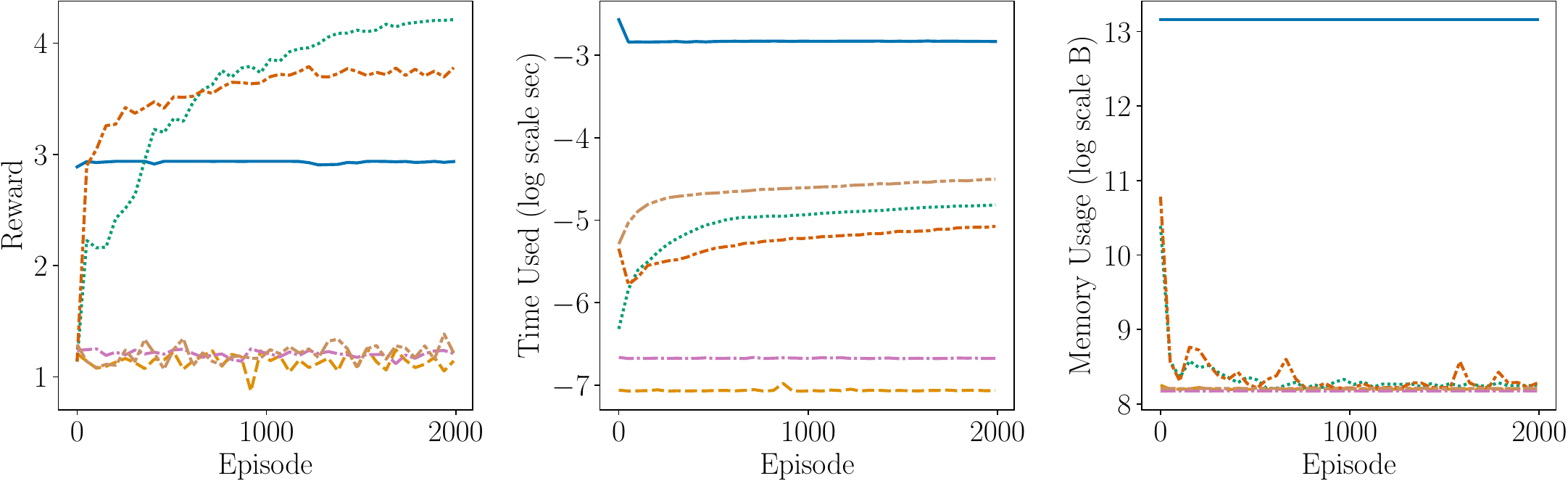}}
\subfigure[{Legend}]{\label{fig:3g}\includegraphics[width=.7\linewidth]{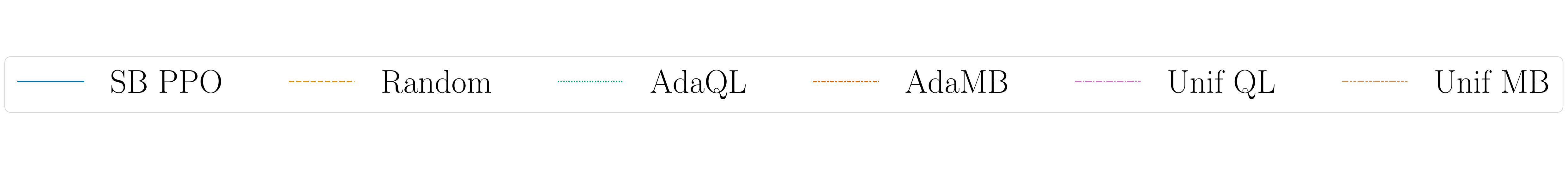}}
\caption{Comparison of the performance (including average reward, time complexity, space complexity) between \AdaMB, \AdaQL, \EpsMB, \EpsQL, \textsc{Random}, and \textsc{SB PPO} for the oil environment with Laplace rewards.  We see that the adaptive discretization algorithms outperform their uniform discretization counterparts, with \AdaMB and \AdaQL achieving similar levels of performance. \textsc{SB PPO} does not have enough episodes in order to learn any signal, so its performance is essentially that of a randomized algorithm.  When $\alpha = 0$ (as in \cref{fig:3a}) we note that the zooming dimension gives improved guarantees for the sample complexity, providing potential justification of the improved performance of the adaptive algorithms.}
\label{fig:oil_perf}
\end{figure}

\begin{figure}
\centering     %%% not \center
\subfigure[Quadratic $d=1, \alpha=0.1$]{\label{fig:4a}\includegraphics[width=.49\linewidth]{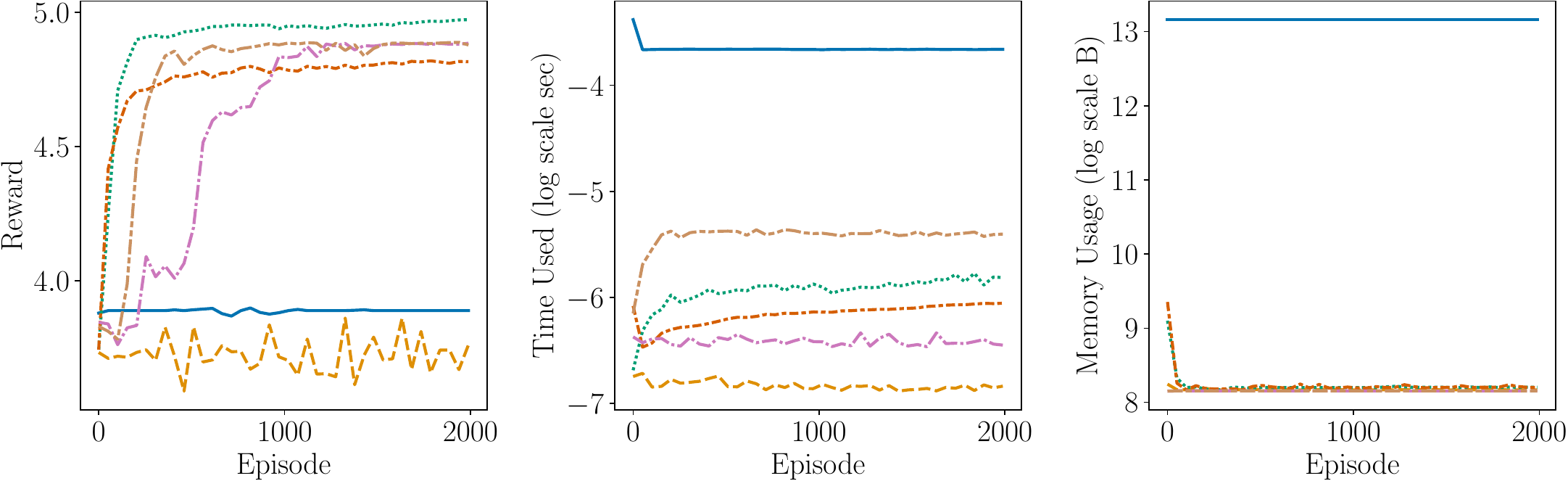}}
\subfigure[Quadratic $d=1, \alpha=0.1$]{\label{fig:4b}\includegraphics[width=.49\linewidth]{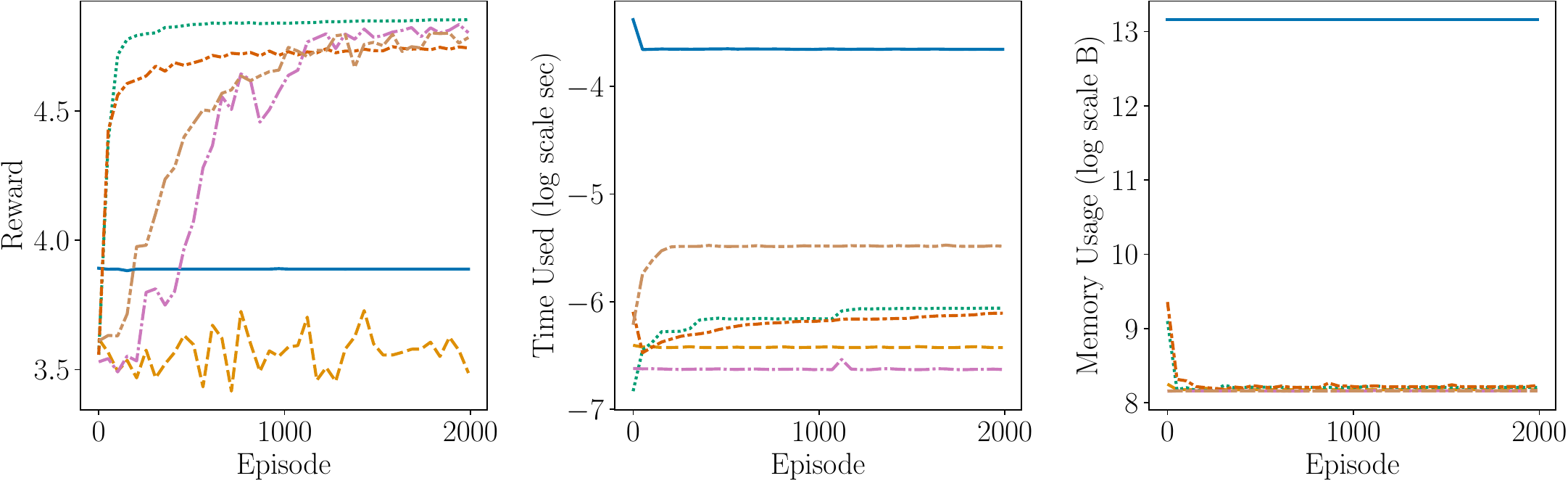}}
\subfigure[Quadratic $d=1, \alpha=0.5$]{\label{fig:4c}\includegraphics[width=.49\linewidth]{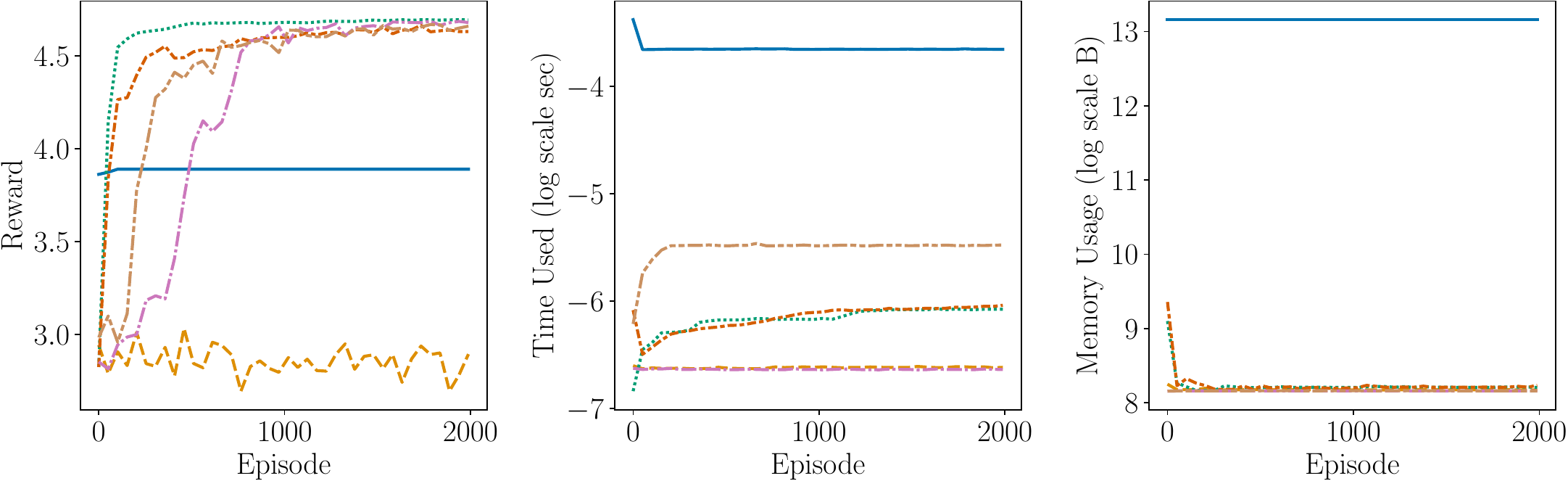}}
\subfigure[Quadratic $d=2, \alpha=0$]{\label{fig:4d}\includegraphics[width=.49\linewidth]{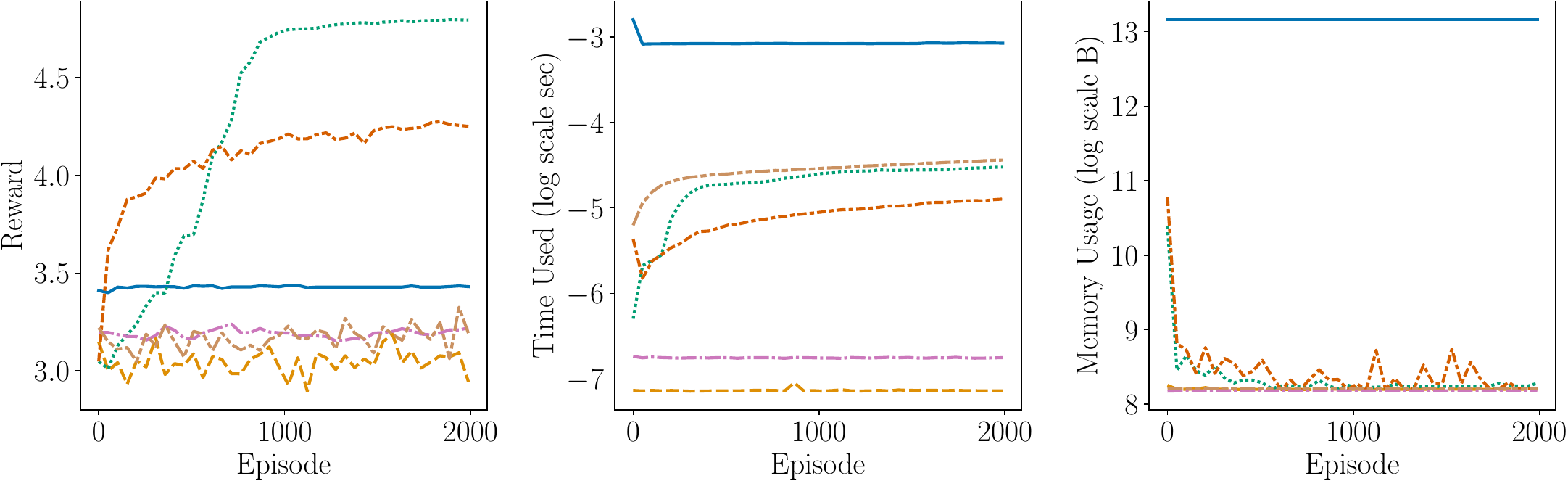}}
\subfigure[Quadratic $d=2, \alpha=0.1$]{\label{fig:4e}\includegraphics[width=.49\linewidth]{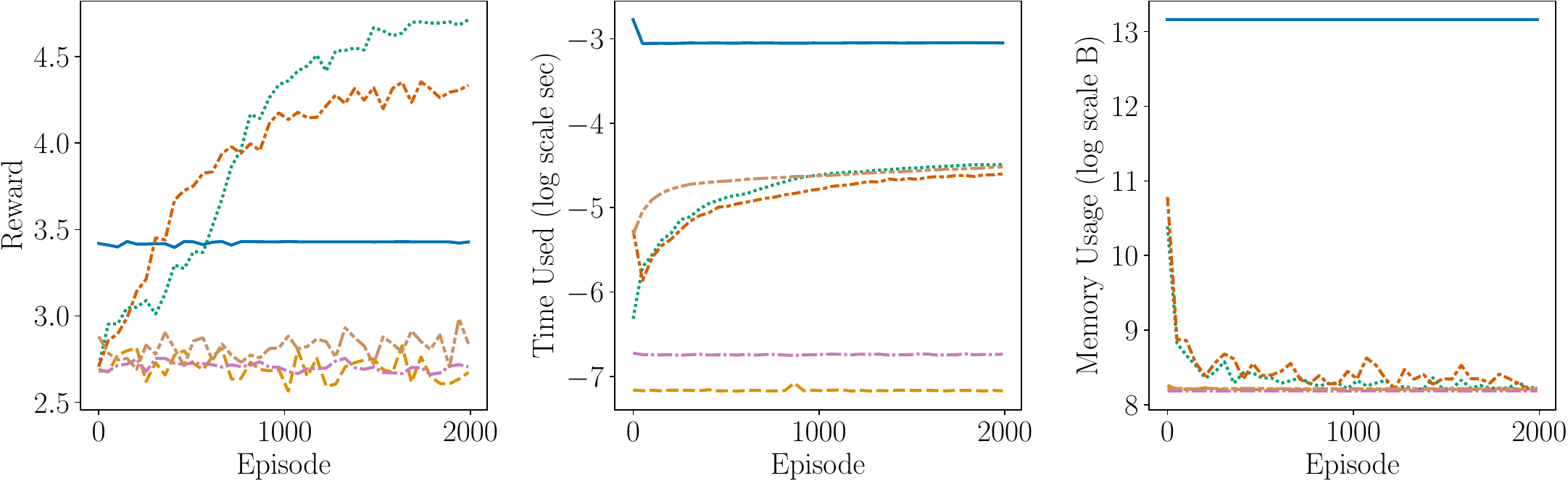}}
\subfigure[Quadratic $d=2, \alpha=0.5$]{\label{fig:4f}\includegraphics[width=.49\linewidth]{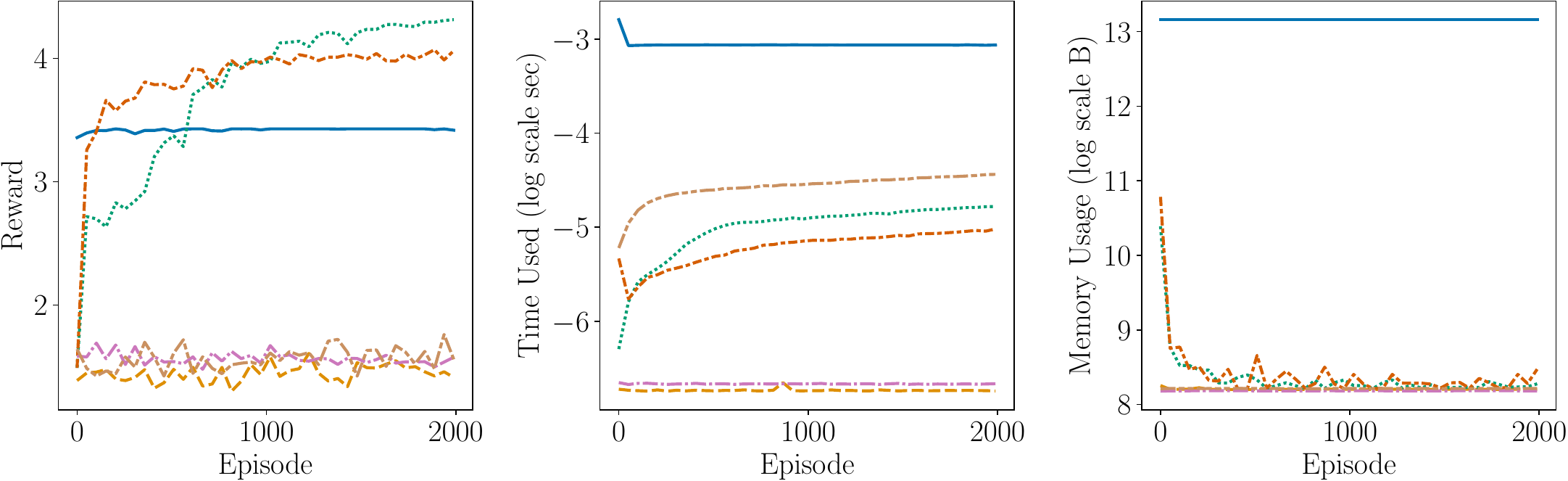}}
\subfigure[{Legend}]{\label{fig:4g}\includegraphics[width=.7\linewidth]{figures/OIL_LABEL_ONLY.pdf}}
\caption{Comparison of the performance (including average reward, time complexity, space complexity) between \AdaMB, \AdaQL, \EpsMB, \EpsQL, \textsc{Random}, and \textsc{SB PPO} for the oil environment with quadratic.  When $d = 2$ (as in \cref{fig:4e}) we see that the adaptive algorithms drastically outperform all other algorithms.  This can be attributed to the adaptive algorithms maintaining a smaller partition of the space, hence requiring exponentially less samples used for exploration.}
\label{fig:oil_quadratic_perf}
\end{figure}

\begin{figure}
\centering     %%% not \center
\subfigure[Beta $k=1, \alpha=0$]{\label{fig:5a}\includegraphics[width=.49\linewidth]{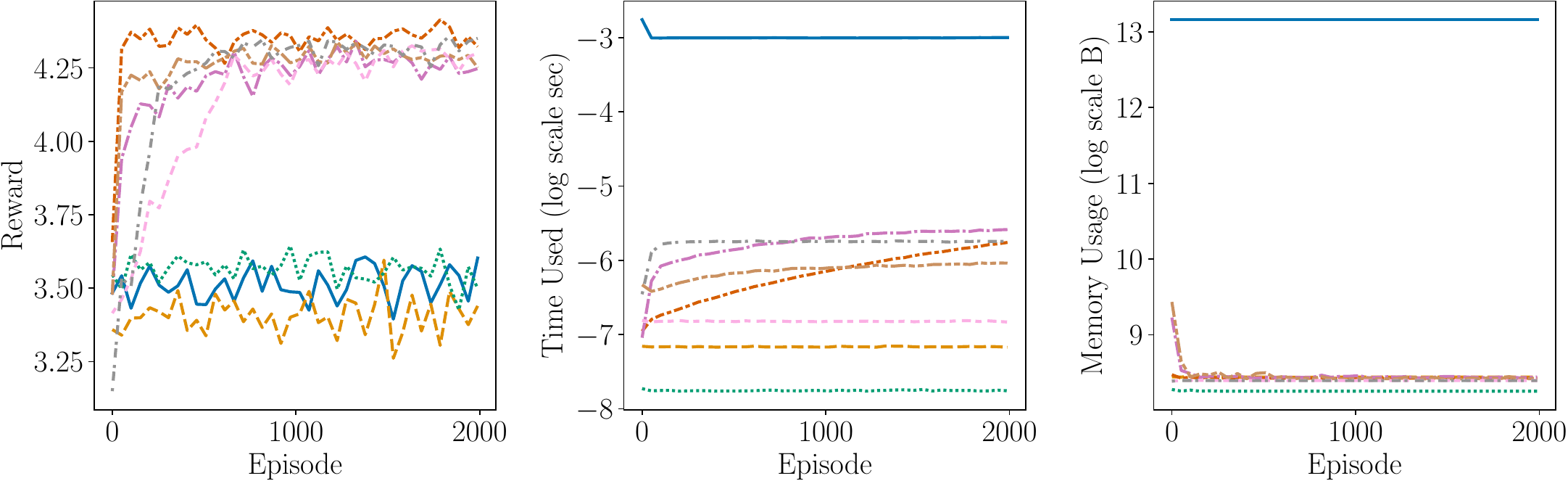}}
\subfigure[Beta $k=1, \alpha = 0.25$]{\label{fig:5b}\includegraphics[width=.49\linewidth]{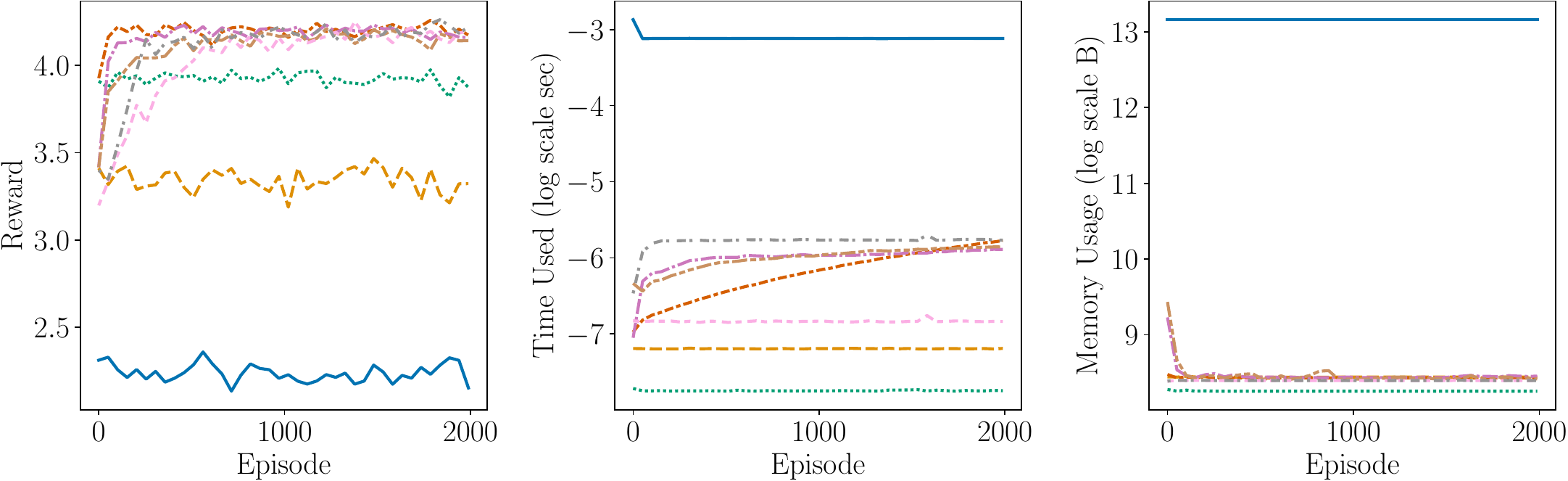}}
\subfigure[Beta $k=1, \alpha=1$]{\label{fig:5c}\includegraphics[width=.49\linewidth]{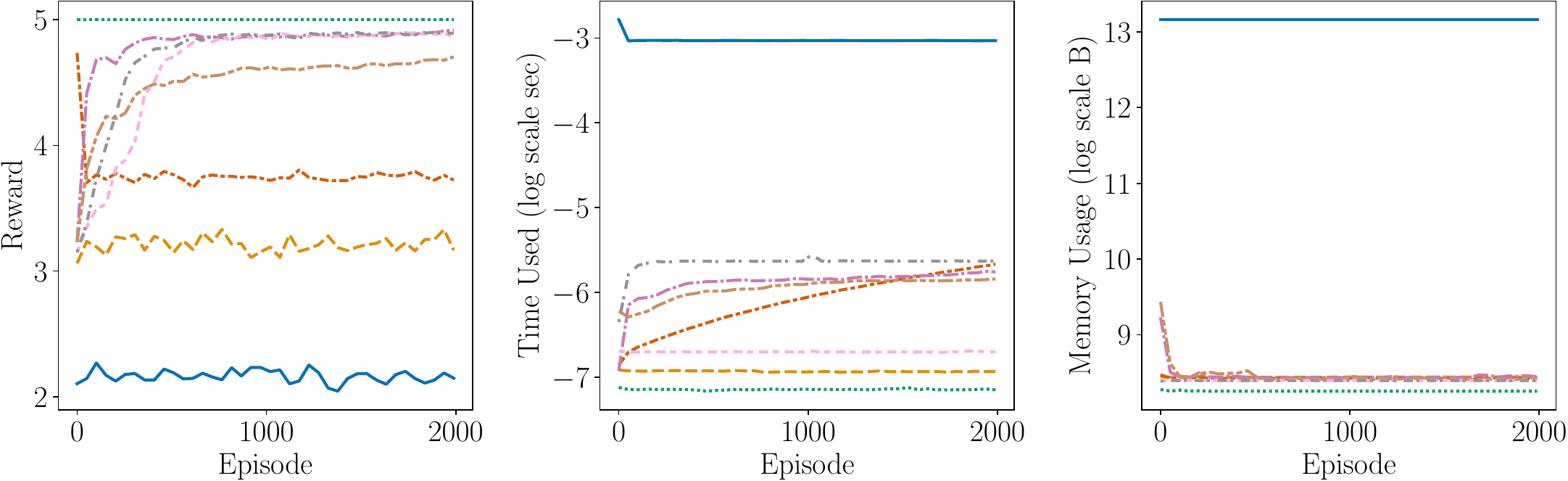}}
\subfigure[Beta $k=2, \alpha=0$]{\label{fig:5d}\includegraphics[width=.49\linewidth]{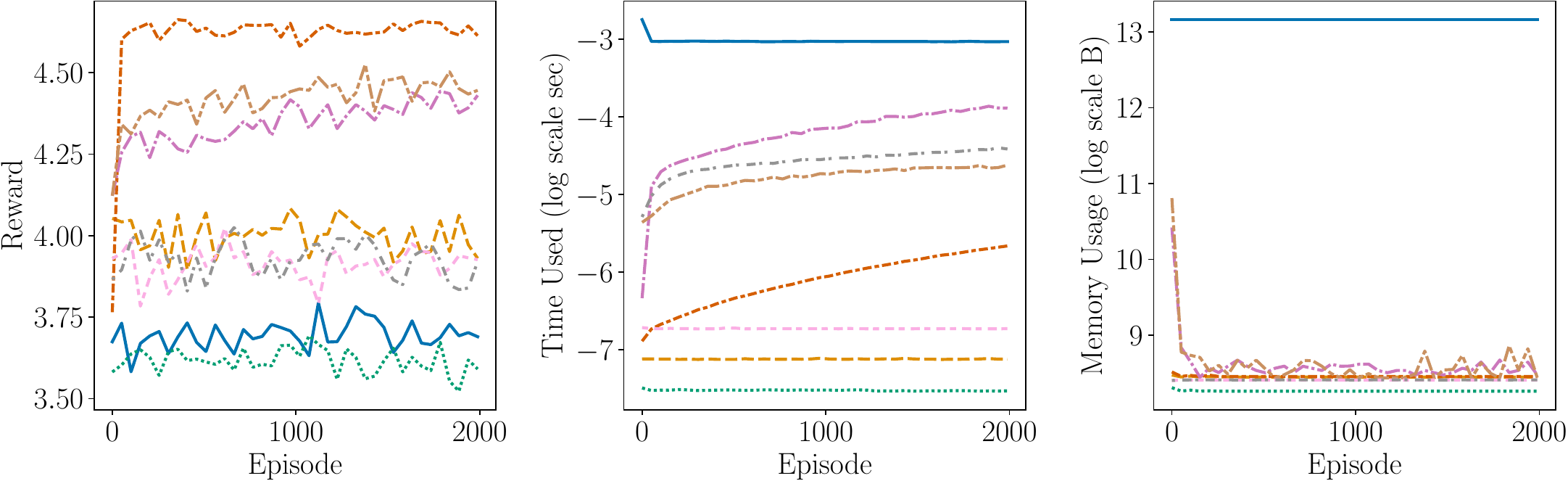}}
\subfigure[Beta $k=2, \alpha=0.25$]{\label{fig:5e}\includegraphics[width=.49\linewidth]{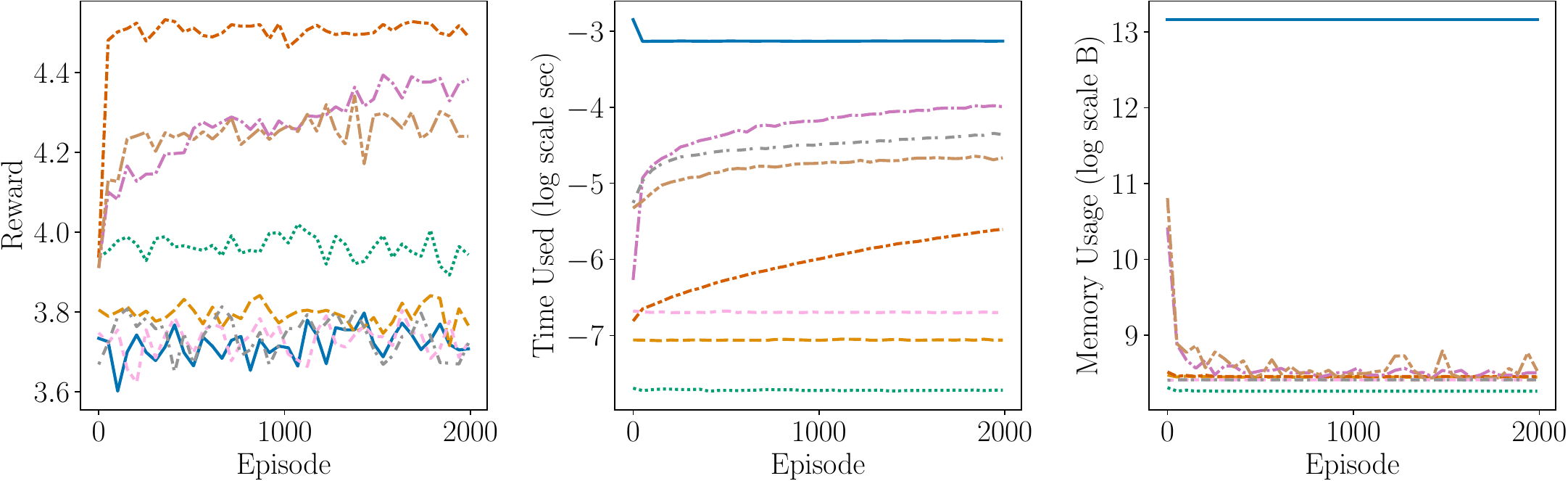}}
\subfigure[Beta $k=2, \alpha=1$]{\label{fig:5f}\includegraphics[width=.49\linewidth]{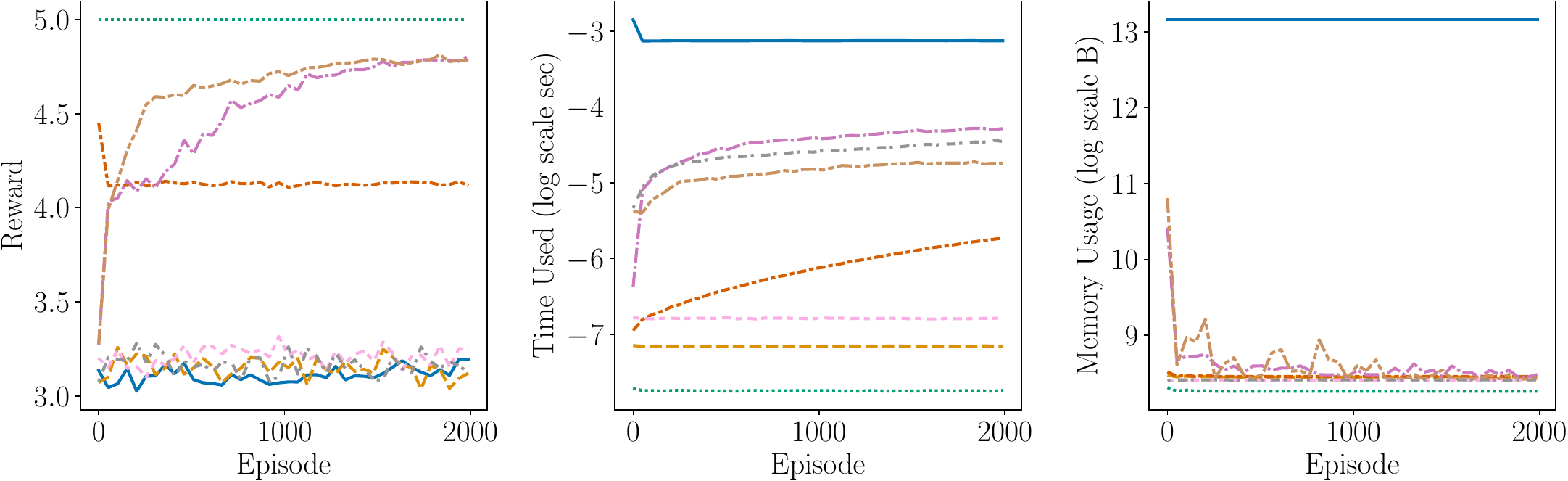}}
\subfigure[Legend]{\label{fig:5g}\includegraphics[width=.7\linewidth]{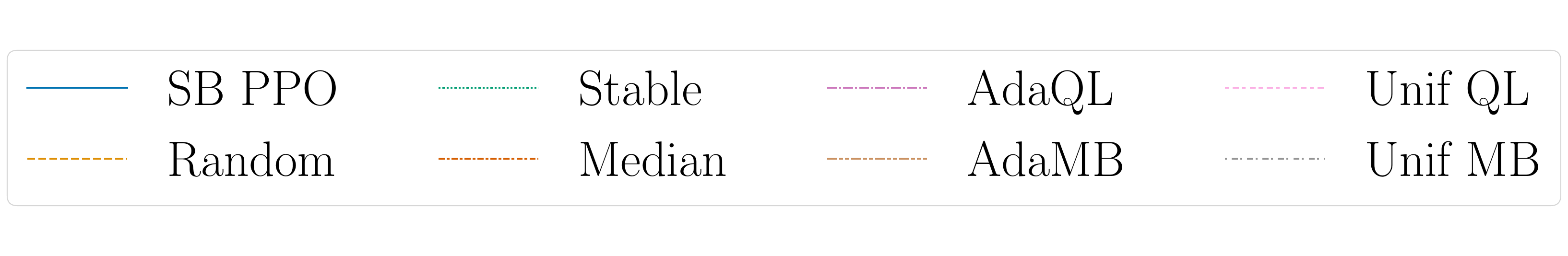}}
\caption{Comparison of the performance (including average reward, time complexity, space complexity) between \AdaMB, \AdaQL, \EpsMB, \EpsQL, \textsc{Random}, \textsc{Median}, \textsc{Stable}, and \textsc{SB PPO} for the ambulance environment with Beta arrivals.  With more ambulances (as in \cref{fig:5e}) we see that the adaptive algorithms outperform all other benchmarks except the \textsc{Median} algorithm (designed to work well for this particular problem instance).}
\label{fig:amb_perf}
\end{figure}

\begin{figure}
\centering     %%% not \center
\subfigure[Shifting $k=1, \alpha=0$]{\label{fig:6a}\includegraphics[width=.49\linewidth]{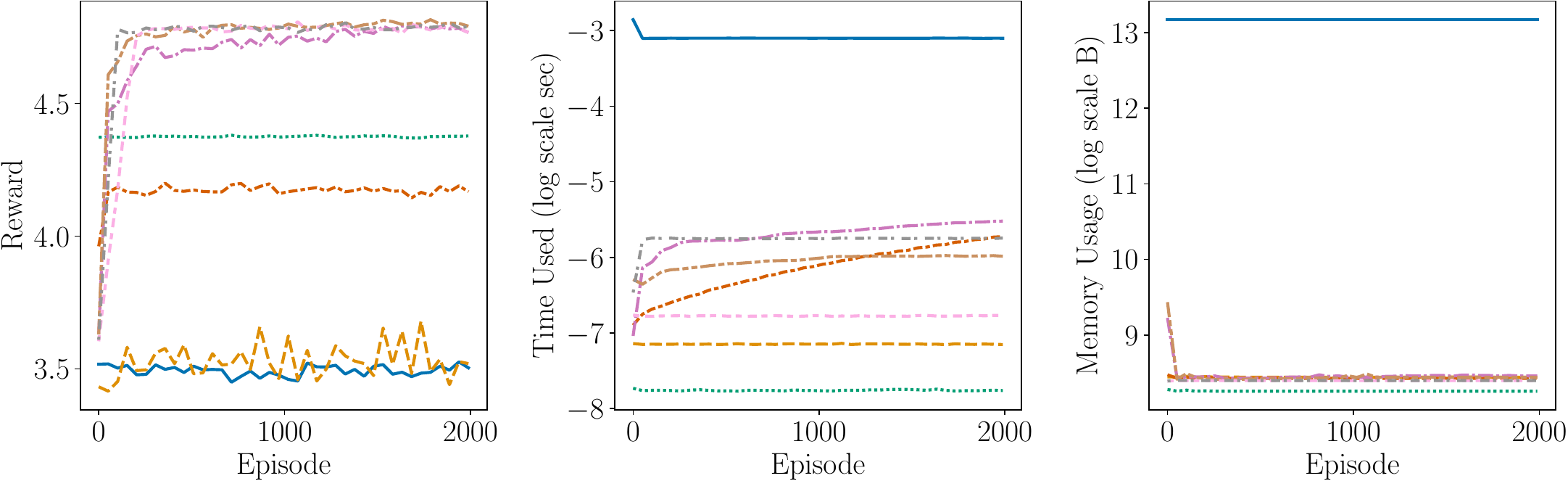}}
\subfigure[Shifting $k=1, \alpha=0.25$]{\label{fig:6b}\includegraphics[width=.49\linewidth]{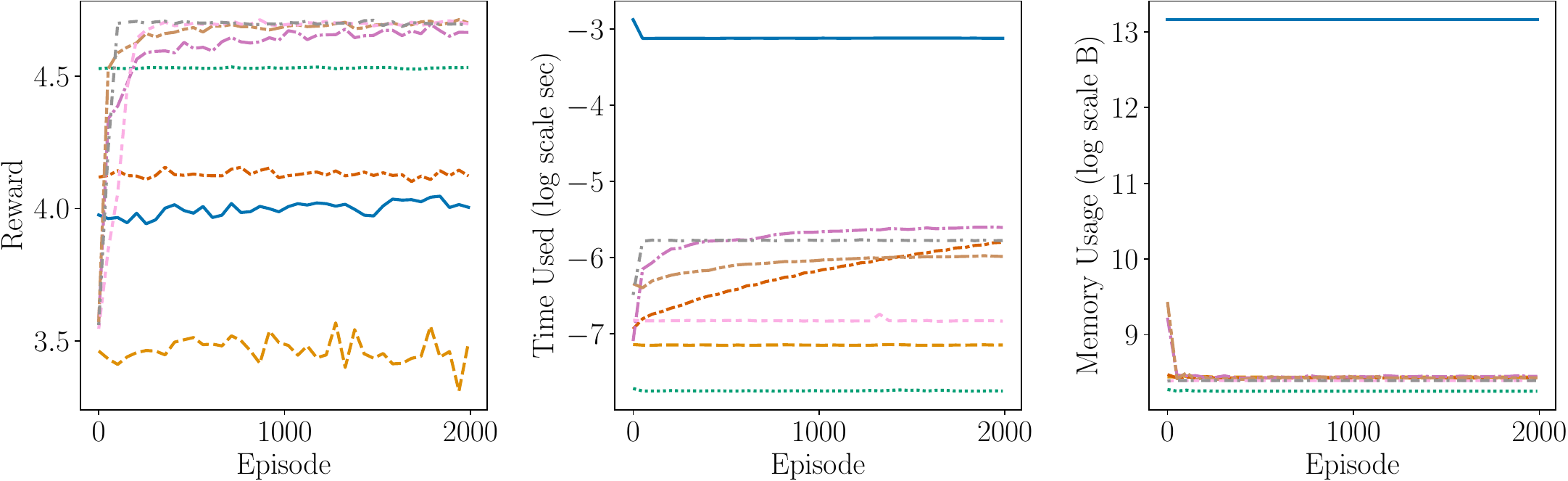}}
\subfigure[Shifting $k=1, \alpha=1$]{\label{fig:6c}\includegraphics[width=.49\linewidth]{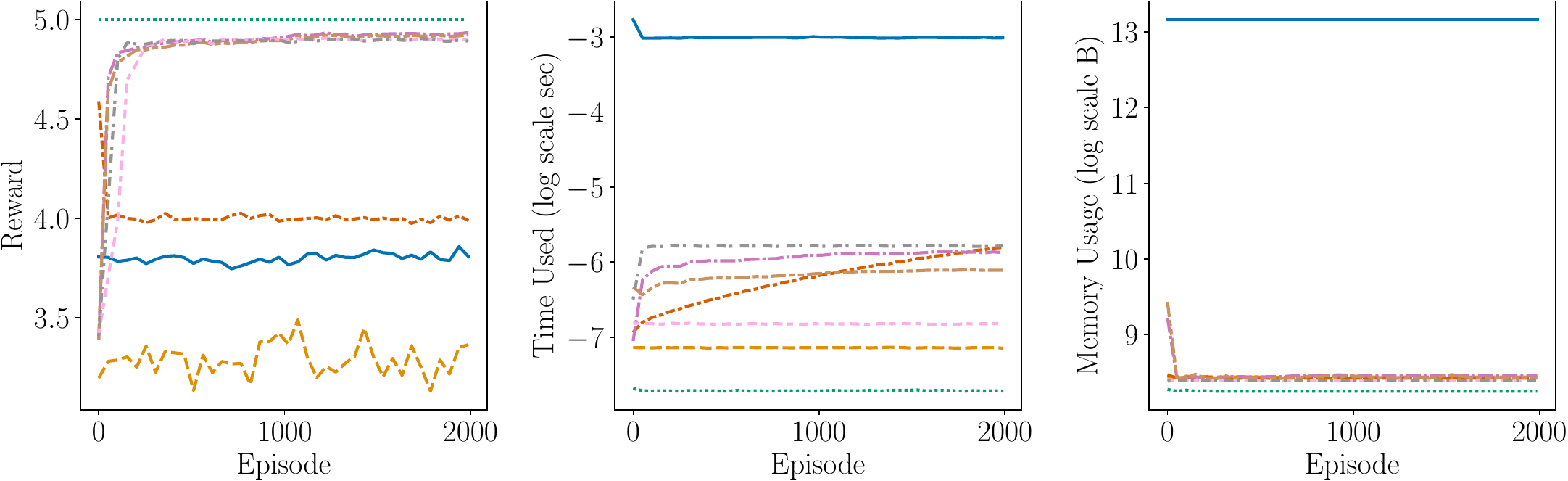}}
\subfigure[Shifting $k=2, \alpha=0$]{\label{fig:6d}\includegraphics[width=.49\linewidth]{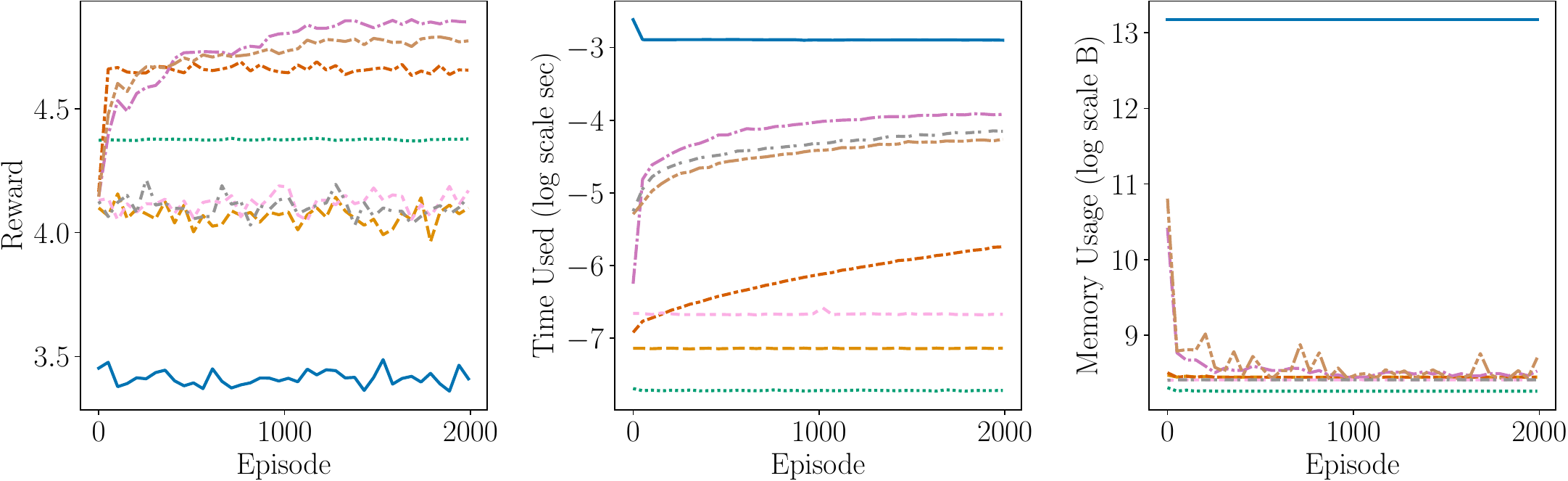}}
\subfigure[Shifting $k=2, \alpha=0.25$]{\label{fig:6e}\includegraphics[width=.49\linewidth]{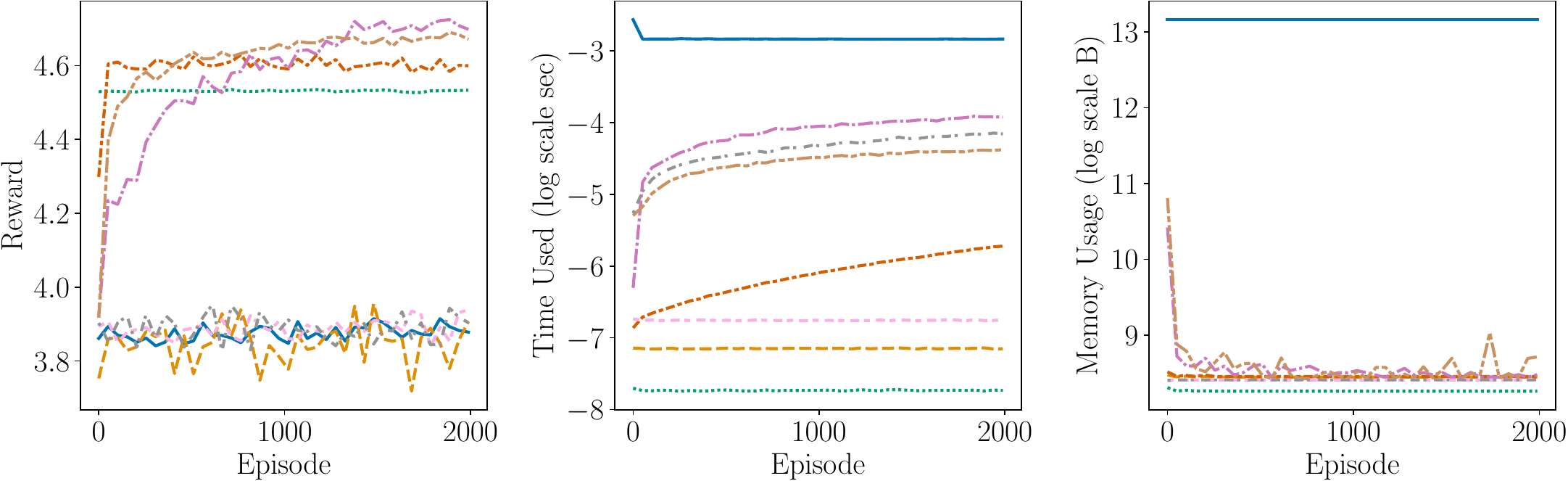}}
\subfigure[Shifting $k=2, \alpha=1$]{\label{fig:6f}\includegraphics[width=.49\linewidth]{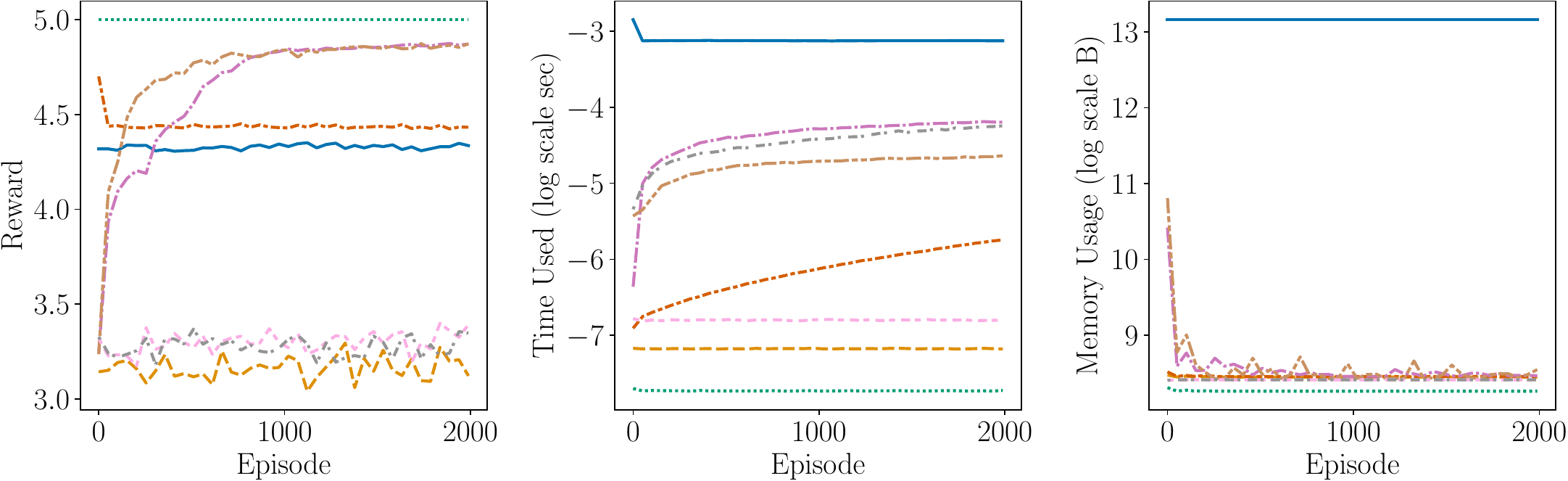}}
\subfigure[Legend]{\label{fig:6g}\includegraphics[width=.7\linewidth]{figures/AMB_LABEL_ONLY.pdf}}
\caption{Comparison of the performance (including average reward, time complexity, space complexity) between \AdaMB, \AdaQL, \EpsMB, \EpsQL, \textsc{Random}, \textsc{Median}, \textsc{Stable}, and \textsc{SB PPO} for the ambulance environment with shifting arrivals.  The experiment results highlight that \AdaMB and \AdaQL improve upon all of the other baseline algorithms.  However, we note that in settings when $\alpha = 0$ and we have explicit bounds on the zooming dimension (as in \cref{fig:6a}) that there is improved gains on the adaptive algorithms over the uniform counterparts.}
\label{fig:amb_perf_shifting}
\end{figure}

\clearpage

% Acknowledgments here

\end{document}